%% file: iclr2027_conference.tex
\documentclass{article}
\usepackage[T1]{fontenc}
\usepackage{iclr2027_conference,times}
\iclrfinalcopy                                    % <-- ADDED: removes running head + line numbers, reveals authors
\usepackage{microtype}
\usepackage{amsmath,amssymb,amsthm,mathtools}
\usepackage{graphicx,booktabs,longtable,array,tabularx}
\usepackage{float}
\usepackage{hyperref,url,listings,seqsplit}
\hypersetup{pdftitle={Representable but Unlearned: Encoding Rank and the Interaction-Prediction Floor},pdfauthor={Zahra Khodagholi, Niloofar Yousefi},colorlinks=true,linkcolor=blue,citecolor=blue,urlcolor=blue}   % <-- CHANGED: pdfauthor
\renewcommand{\_}{\textunderscore\allowbreak}
\newcolumntype{L}[1]{>{\raggedright\arraybackslash}p{#1}}
\newtheorem{proposition}{Proposition}[section]
\newtheorem{definition}[proposition]{Definition}
\newtheorem{corollary}[proposition]{Corollary}
\newcommand{\R}{\mathbb R}
\newcommand{\1}{\mathbf 1}
\newcommand{\MSE}{\operatorname{MSE}}
\newcommand{\Var}{\operatorname{Var}}

\title{Representable but Unlearned:\\Encoding Rank and the\\Interaction-Prediction Floor}
\author{Shadi Khodagholi \& Niloofar Yousefi \\
Complex Adaptive Systems Laboratory \\
University of Central Florida \\
Orlando, FL 32816, USA \\
\texttt{\{zahra.khodagholi, niloofar.yousefi\}@ucf.edu}}
\begin{document}
\newcommand{\NRows}{4,766}
\newcommand{\NEnsi}{2,927}
\newcommand{\NApp}{1,839}
\newcommand{\NSseven}{118}
\newcommand{\NAppPools}{17}
\newcommand{\NBpairs}{156}
\newcommand{\NBbackgrounds}{26}
\newcommand{\NBcomponents}{12}
\newcommand{\NBthreeConditions}{165}
\newcommand{\NBthreeRectangles}{140}

\maketitle
\begin{abstract}
Input encodings can prevent a predictor from reproducing measured interactions
jointly, even when each interaction is individually representable. We compute
the attainable contrast space from the encoder and contrast design, without
labels or predictor fitting. Projecting measured contrasts onto this space gives
an empirical squared-error floor for any deterministic decoder using the
encoding. On an siRNA panel, a graph neural network's training-only feature mask
merges 165 measured states into 90 classes, reducing the rank of 140 interaction
contrasts to 72. The floor is 0.009980, or 14.63\% of recorded contrast energy,
the mean squared error of predicting zero for every contrast. Its 14.6\% share
of fitted interaction error adds little: predicted interactions are nearly zero,
and their MSE of 0.068335 slightly exceeds that of the same control.
Antisense-marginal shares are larger: 36.49\% of recorded energy and 38.1\% of
fitted error. Three restored chemistry columns suffice for full rank. Removing
the full mask eliminates the floor, but refitting improves interaction MSE by
only 0.000017; the restored features remain absent from training inputs.
Two further panels clarify the diagnostic's scope. A released splicing predictor
encodes injectively and retains the full design rank of 1{,}986, so its floor
is zero. For an independent count-fingerprint encoder on a lipid panel, 800
pairwise-distinct encoded contrast rows have rank 740, exposing 60 restrictions
missed by duplicate-row checks. The diagnostic identifies encoding-imposed
limits and can trace them to specific features. It separates representability
from learned accuracy, but does not by itself predict refitting gains or explain
the remaining error. Code and data are available at
\url{https://github.com/shadi97kh/REPRESENTABLE-BUT-UNLEARNED}.
\end{abstract}
\section{Introduction}
\looseness=-1 Consider four measurements on the same siRNA sequence: a reference duplex, an antisense modification alone, a sense modification alone, and both modifications together. Write $\mu_{ab}$ for the measured response when the antisense modification is present ($a=1$) or absent ($a=0$) and the sense modification is present ($b=1$) or absent ($b=0$). Their interaction contrast is $I=\mu_{11}-\mu_{10}-\mu_{01}+\mu_{00}$. It measures departure from additivity on the chosen response scale.

Now repeat the comparison with a second antisense modification. We obtain two interaction contrasts, one for each antisense modification paired with the same sense modification. Suppose the encoder merges the two antisense variants when alone, and also when paired with the sense modification, but otherwise distinguishes the six conditions. A deterministic predictor must then assign the same predictions to the corresponding conditions, so its two predicted interaction contrasts must be equal.

Neither contrast is individually forced to be zero. The restriction concerns their relationship: the model cannot predict different interaction effects for the two variants. To illustrate with values chosen only for exposition, if the measured contrasts were $1$ and $2$, the model could reproduce either value separately but could not reproduce both together. This is a joint contrast restriction.

Such a restriction creates an unavoidable prediction error on the measured panel. For these illustrative values, the best common prediction is $1.5$, giving a mean squared error of $0.25$ across the two contrasts. This is the empirical error floor: the smallest mean squared contrast error attainable by any decoder using that encoding. Changing the decoder, or training it longer, cannot remove this floor while the encoding is fixed. Changing the encoding can remove the restriction, but does not guarantee that a fitted model will learn the newly representable effects.

The example turns on an equality between two predicted contrasts. Encodings can also impose relations among several contrasts without forcing any pair to be equal. Our diagnostic identifies these relations and computes their empirical error floor.

On the audited siRNA panel, preprocessing makes chemically distinct conditions identical, removing 68 of 140 independent contrast directions. Activity scores alone do not reveal these restrictions. We examine three interventions: mask removal, added chemistry support in siRNA training, and added feature support in a released splicing predictor. End-to-end evaluation of independent modified-siRNA predictors remains incomplete (Section~\ref{sec:measured}).

\paragraph{What is new, and what is not.} \looseness=-1 Classical rank, column-space and projection calculations identify which contrast vectors a fixed encoding permits and the empirical floor for those it excludes. Our contribution is their construction from an encoder and measurement design, their use to locate restrictions in measured panels, and the resulting representation decisions. Factorial estimability is the closest connection \citep{boxhunter1961,wuchen1992}, stated in Section~\ref{sec:related}. On the siRNA panels, identifying contrasts that the encoding makes identical recovers every restriction. The lipid panel requires joint relations among distinct encoded contrasts (Appendix~\ref{app:lossy}). We claim no consistent overall predictive benefit from the siRNA interventions; the MAVE-NN intervention remains descriptive.
\begin{figure}[ht]
\centering\includegraphics[width=0.85\textwidth]{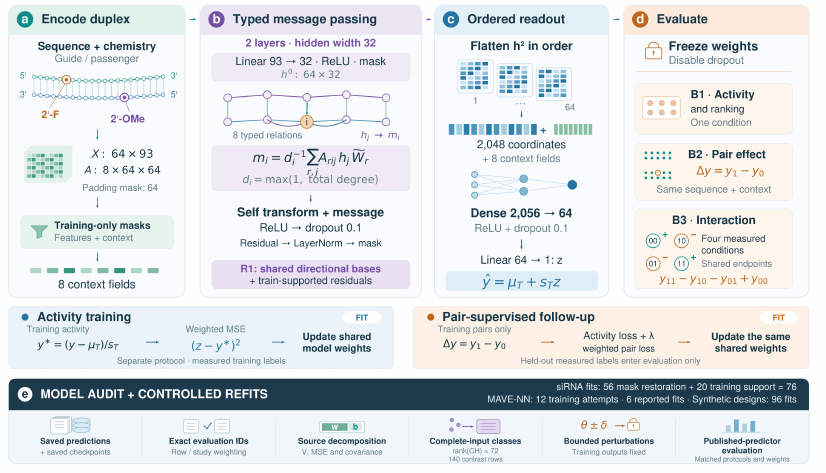}
\caption{Architecture and evaluation (Appendix~\ref{app:architecture}). Panels~(a)--(d) show the
siRNA representation, typed message passing, ordered readout, and evaluation with
frozen weights; held-out labels enter evaluation only. Band~(e) summarizes
saved-model auditing, published-predictor evaluation and controlled refits. All new
fitting in this work: 56 siRNA mask-restoration fits and 20 siRNA training-support fits,
76 together; 12 MAVE-NN attempts of which six supply reported results; and 96 fits on
constructed synthetic designs (Appendix~\ref{app:floorsynthetic}). Appendix~\ref{app:fitledger}
accounts for every fit.}\label{fig:workflow}
\end{figure}
\subsection{Benchmarks and cohorts}\label{sec:setup}
Three cohorts supply measurements. The \emph{grouped ENsiRNA benchmark} is our eligible subset of the activity table released with ENsiRNA \citep{ensi2025}: 2,927 of its 3,527 released rows meet the admission rules of Appendix~\ref{app:data}, in 150 sequence-linked and ten study-linked components never split between training and evaluation. Of the admitted rows, 1,924 come from \citet{bramsen2009} and 594 from one patent family. The \emph{APP cohort} holds 1,839 measurement rows against the amyloid-beta precursor protein gene (APP) from two Alnylam patent families, curated in CMsiRNAdb \citep{cmsirna2026} and measured in 17 overlapping assay pools. Appendix~\ref{app:data} identifies every source family. \emph{Davis S7} is the 118-row Supplementary Table 7 of \citet{davis2025}. Activity models are trained and selected on the grouped benchmark only, so APP and S7 are held out. The saved models we audit come from two historical training campaigns: one trained on the released labels, the other with the \citet{bramsen2009} source excluded or relabelled from its primary assay. They differ in other protocol details as well (Section~\ref{sec:calibration}). Appendix~\ref{app:fitledger} identifies both and lists every fit.

\looseness=-1 Three benchmarks use them. \emph{B1} is activity and ranking: held-out knockdown prediction on all three cohorts, and top-five selection within each APP assay pool. \emph{B2} is a measured chemistry bundle: 156 APP pairs over 26 duplex backgrounds in twelve sequence components, in which no single chemical change is isolated (Appendix~\ref{app:data}). \emph{B3} is a complete 15-by-11 antisense-by-sense sub-rectangle we extract from the 48-by-45 screen of \citet{bramsen2009}; the selection is ours, not a design that screen reports. Its 15 antisense and 11 sense chemistries on one background give 165 measured states, 140 four-condition rectangles, and 14 antisense and 10 sense marginal effects against a shared reference.
\section{An exact representational restriction}\label{sec:structure}

\looseness=-1 Fix a finite panel of $m$ distinct raw endpoint states $s_1,\dots,s_m$, the measured conditions, and $q$ contrasts on them.

Four objects carry the argument of the previous section. First, $C\in\R^{q\times m}$ specifies which measurements are added or subtracted to form each contrast, one contrast per row. Second, $H\in\{0,1\}^{m\times k}$ groups the conditions: it sends each state to one of $k$ complete-input classes, so states in the same class receive identical encoded inputs. Third, $g\in\R^{k}$ assigns a prediction to each group, which is what an unrestricted deterministic decoder on those classes amounts to. Fourth, the product $CHg$ gives the resulting predicted contrasts.

As $g$ ranges over all predictions, $CHg$ traces out the contrast vectors the encoding permits. We call this the attainable space. Its dimension is the number of independent contrast directions the encoding allows. In linear-algebraic terms the attainable space is the column space of $CH$, written $\operatorname{col}(CH)$, and its dimension is $\operatorname{rank}(CH)$.

Representability here means attainability by an unrestricted decoder on these classes, not by a fitted family. The next two statements hold for any finite panel and encoding, with no labels, loss or fitted model. In the opening example, $C$ forms the two four-condition contrasts, while $H$ groups the conditions that the encoder cannot distinguish. The proposition below concerns a weighted sum of the contrasts, with coefficient vector $u$; that sum is forced to zero when it vanishes for every decoder. Here $\1_m$ is the all-ones vector, so $C\1_m=0$ says that the weights in every contrast sum to zero.
\begin{proposition}\label{prop:attain}
The contrast vectors attainable by unrestricted deterministic decoders on the encoding classes are exactly $\operatorname{col}(CH)$. A fixed linear functional $u^{\top}$ of the $q$ contrasts vanishes for every such decoder if and only if $u^{\top}CH=0$, so the forced-zero functionals form a space of dimension $q-\operatorname{rank}(CH)$. If in addition $C\1_m=0$ then $\operatorname{rank}(CH)\le k-1$.
\end{proposition}
\begin{corollary}\label{cor:factor}
If the classes factor as $a$ antisense by $b$ sense classes and the panel is the full antisense-by-sense product with every rectangle anchored at one shared reference state, the interaction rank is exactly $(a-1)(b-1)$.
\end{corollary}
Two kinds of relation must be kept apart. The contrast design may already force relations among the $q$ contrasts, and there are $q - \operatorname{rank}(C)$ of these; the recorded contrasts satisfy these design relations automatically, so these relations do not create an error floor. The encoding then removes a further $\operatorname{rank}(C) - \operatorname{rank}(CH)$ directions. Only this second kind can produce a positive floor. The released splicing predictor of Section~\ref{sec:mavenn} has 97 relations of the first kind and none of the second, and its floor is exactly zero.

Both $\operatorname{rank}(CH)$ and the forced-zero dimension follow by exact rational elimination on integer matrices from the encoder and design alone, so the procedure applies to any categorical panel; proofs are in Appendices~\ref{app:contrastrank} and~\ref{app:chemicalfactor}. The floor is obtained by finding the attainable contrast vector closest to the recorded one and measuring what is left over. For the recorded contrast vector $\widetilde I\in\R^q$, with $\mathrm{Id}$ the identity matrix, the encoding floor is
\[F=\min_{g\in\R^k}\tfrac1q\lVert\widetilde I-CHg\rVert^2=\tfrac1q\lVert(\mathrm{Id}-P)\widetilde I\rVert^2,\qquad P\text{ the orthogonal projector onto }\operatorname{col}(CH).\]
The rank needs no labels, whereas the floor is empirical, since it uses the recorded contrasts; ``achievable'' means attained by an unrestricted class decoder, not necessarily by the fitted family. Grouping coinciding rows of $CH$ is cheaper than this rank computation and is not equivalent to it: a constructed panel hides a forced functional from every row-coincidence check (Appendix~\ref{app:rowgap}), and so does one measured encoder outside this pipeline (Appendix~\ref{app:lossy}).

\begin{figure}[ht]\centering\includegraphics[width=\textwidth]{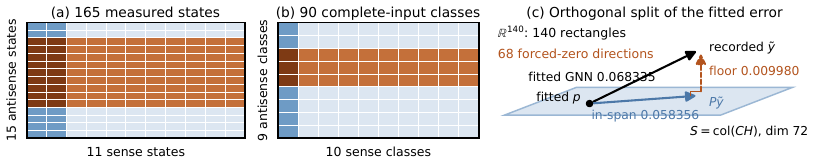}
\caption{How the encoding collapses the panel, and what the collapse costs. (a)~The panel is a 15 by 11 antisense-by-sense product, 165 states. (b)~The mask zeroes 44 node-feature columns, merging three antisense triples and one sense pair (rust rows, blue columns) into 9 by 10 classes. (c)~The 140 contrasts then span only $S=\operatorname{col}(CH)$, of dimension 72, leaving 68 forced-zero directions; the fitted error splits orthogonally into an in-span part and a residual off $S$ whose mean square is the floor. The panel writes the recorded contrast vector $\widetilde y$; the text writes that same vector $\widetilde I$.}\label{fig:mechanism}
\end{figure}
Everything below is specific to the audited panel (Fig.~\ref{fig:mechanism}). With $q=140$, $m=165$ and $k=90$, exact elimination gives $\operatorname{rank}(CH)=72$, so the encoding imposes 68 independent linear constraints on the joint \emph{predicted} contrast vector that the recorded vector need not satisfy. No row of $CH$ is zero: no \emph{single} rectangle is forced to cancel, which is why a per-contrast check sees nothing. The partition factors into 9 antisense and 10 sense classes, matching Corollary~\ref{cor:factor}, and duplicate-row equalities generate all 68 constraints. It is shared by all 40 audited instances and verified on all 27,225 state pairs. All four model families share this preprocessing, so the finding does not extend to unrelated pipelines.

\looseness=-1 We express the same floor relative to two reference quantities, and only one of them is independent of fitted predictions. Against \emph{recorded contrast energy}, the mean squared measured contrast, which is also the error of predicting zero for every contrast, $q^{-1}\lVert\widetilde I\rVert^2=0.068205$, it is $14.63\%$, which is identical for every decoder on these classes and is the share we treat as primary. Against \emph{fitted} squared error it is $14.60\%$; the percentages are close because this model's fitted contrasts are nearly zero, so the second restates the first rather than measuring something new.

\looseness=-1 Every fitted contrast vector lies in $S=\operatorname{col}(CH)$ to within $2\times10^{-16}$, so the fitted error splits orthogonally into a floor part and an in-span part (Table~\ref{tab:floor}).

The restriction binds unequally. As a share of each family's error it falls hardest not on interactions but on antisense marginals. For the audited graph neural network (GNN) the floor is $38.1\%$ of fitted error on \emph{antisense marginals}, $0.4\%$ on sense marginals and $14.6\%$ on interactions. Measured against each family's recorded contrast energy instead, the same three shares are $36.49\%$, $1.73\%$ and $14.63\%$. Counted in directions rather than error, the encoding removes 6 of the 14 antisense directions, 1 of the 10 sense and 68 of the 140 interaction directions.

We foreground interactions for two reasons. The motivating question, whether modifying both strands does more than the sum of the single-strand effects, is an interaction question. And the interaction design has full row rank ($\operatorname{rank}(C)=q=140$), which separates the encoding restriction cleanly from design dependencies. Neither reason makes interactions the largest share.

The rank loss localises to the mask alone: parsed chemistry, the pre-mask tensors and the mask-restored encoding all keep 165 classes and rank 140 (appendix Fig.~\ref{fig:encoding}). Representability does not imply learnability, which Section~\ref{sec:support} examines separately. These numbers are specific to one panel and one background; Appendix~\ref{app:counterexample} gives the complete explainer counterexample.
\begin{table}[ht]
\centering\small
\begin{tabular}{@{}lrrrrrr@{}}\toprule
Family & Method & $q$ & Floor & In-span & $\mathrm{MSE}_{q}$ & Ratio\\\midrule
Antisense & R1 GNN & 14 & 0.018972 & 0.030885 & 0.049856 & 2.627909\\
Antisense & R1 no-message & 14 & 0.018972 & 0.031574 & 0.050546 & 2.664273\\
Antisense & Chemistry tree & 14 & 0.018972 & 0.023931 & 0.042903 & 2.261392\\
Antisense & Token CNN & 14 & 0.018972 & 0.030339 & 0.049311 & 2.599177\\
Sense & R1 GNN & 10 & $1.478\times10^{-5}$ & 0.003320 & 0.003335 & 225.665139\\
Sense & R1 no-message & 10 & $1.478\times10^{-5}$ & 0.003221 & 0.003236 & 218.972986\\
Sense & Chemistry tree & 10 & $1.478\times10^{-5}$ & 0.008408 & 0.008423 & 569.941309\\
Sense & Token CNN & 10 & $1.478\times10^{-5}$ & 0.004025 & 0.004040 & 273.371096\\
Interaction & R1 GNN & 140 & 0.009980 & 0.058356 & 0.068335 & 6.847534\\
Interaction & R1 no-message & 140 & 0.009980 & 0.058301 & 0.068281 & 6.842116\\
Interaction & Chemistry tree & 140 & 0.009980 & 0.061959 & 0.071939 & 7.208643\\
Interaction & Token CNN & 140 & 0.009980 & 0.058429 & 0.068409 & 6.854911\\
\bottomrule\end{tabular}
\caption{Encoding-floor decomposition of the recorded B3 error, equal weights: Floor is $q^{-1}\lVert(\mathrm{Id}-P)\widetilde I\rVert^2$ with $P$ the projector onto $\operatorname{col}(CH)$ for that family, In-span is $q^{-1}\lVert P\widetilde I-p\rVert^2$, and $\mathrm{MSE}_{q}$ their sum, exact before rounding. Ratio is $\mathrm{MSE}_{q}$ divided by Floor. The floor is shared across methods in a family because every fitted contrast vector lies in $\operatorname{col}(CH)$; the contrast counts are the 14 and 10 marginal contrasts and the 140 rectangles. All entries are rounded from stored values; floor, in-span and ratio are each computed before rounding. The tree is extremely randomized trees, the token CNN a convolutional baseline.}\label{tab:floor}
\end{table}

\section{Evaluation framework and identifiability}\label{sec:audit}
\looseness=-1 We distinguish three claims. (C1)~An explanation is \emph{faithful} if it describes the fixed predictor it is computed from; (C2)~a dependence is \emph{training-supported} if the training computation constrains it; (C3)~a prediction \emph{agrees with a measured effect} if it matches a comparable recorded intervention. A faithful input derivative answers C1 alone. A good overall prediction score alone establishes none of these three claims.

\looseness=-1 Appendix~\ref{app:movedconventions} defines four operational terms used later: a \emph{complete-input encoding class}, an \emph{evaluation identity}, a \emph{weighting}, and a \emph{training-absent block}. The last of these is a block that provably cannot change any training prediction at any parameter value, not merely at one gradient snapshot.

\looseness=-1 The audited model is a typed message-passing network over the duplex. R0 uses independent per-relation transforms; R1 shares directional bases and gates residuals by training occurrence. Four relation types are absent from every training graph, covering 8,192 independently parameterised scalars; support changes which routes receive data gradients, not the parameter count. Appendices~\ref{app:architecture} and~\ref{app:training} give the forward map, dimensions and training protocol.
\section{Evidence, provenance and exact evaluation identities}\label{sec:data}

\looseness=-1 One source supplies 65.73\% of the grouped benchmark in a single sequence component. Source identifiers are categories rather than independent studies, and no assay-level partition is claimed. The same 1,003 scored rows support two comparisons, depending on whether that source was excluded at \emph{selection} as well as training, and the two give different numbers. Two evaluations can contain the same number of rows but use different labels or saved models. We therefore record an identifier for each evaluation configuration. Our re-parse of that source\textquotesingle s printed chemistry leaves 299 rows differing from the released labels by more than 0.05. These are unresolved discrepancies, not certified label errors, so we report released-label and primary-assay results as separate targets throughout. Denominators and the provenance rules are in Appendices~\ref{app:metrics}, \ref{app:decomposition} and~\ref{app:movedconventions}.
\section{Activity results and training-fitted controls}\label{sec:activity}
\looseness=-1 No method shows a general advantage over a constant. Two training-fitted constants serve as the comparison: the mean activity of the training rows, and the mean under equal study-component weighting, which gives each study component the same total weight whatever its row count. The tree beats both under row weighting but not under equal-component weighting, and all four paired conditional 95\% intervals include zero (Appendix~\ref{app:constantchecks}). The GNN has higher observed error than both constants. On the primary-assay cohort its $R^2$ moves from $+0.004856$ by rows to $-0.3288$ by equal-component weighting, so the weighting choice, not the model, decides the sign. Negative per-source $R^2$ with positive pooled $R^2$ does not imply a between-source explanation (Appendix~\ref{app:decomposition}). This corrects score interpretation; it is not a contribution.
\begin{table}[ht]
\centering\small
\begin{tabular}{@{}llrrrrrr@{}}\toprule
\multicolumn{8}{@{}l}{Grouped benchmark: 1,924 of 2,927 rows (65.73\%) from one source in one sequence component}\\\midrule
Sensitivity & Method & \shortstack{Grouped\\MSE} & $R^2$ & \shortstack{APP\\MSE} & $R^2$ & \shortstack{S7\\MSE} & $R^2$\\\midrule
Source-excluded & R1 GNN & 0.0641 & -0.212 & 0.1631 & -0.005 & 0.1191 & -0.324\\
 & R1 no-message & 0.0625 & -0.182 & 0.1637 & -0.009 & 0.1221 & -0.358\\
 & Chemistry tree & 0.0529 & -0.001 & 0.1655 & -0.020 & 0.0974 & -0.082\\
 & Token CNN & 0.0628 & -0.187 & 0.1626 & -0.002 & 0.1103 & -0.226\\
 & Training row mean & 0.0571 & -0.079 & 0.1810 & -0.115 & 0.1587 & -0.764\\
 & Oracle test mean & 0.0529 & 0.000 & 0.1623 & 0.000 & 0.0900 & 0.000\\
Primary-assay & R1 GNN & 0.0747 & 0.005 & 0.1628 & -0.003 & 0.1052 & -0.169\\
 & R1 no-message & 0.0746 & 0.007 & 0.1624 & -0.001 & 0.1073 & -0.193\\
 & Chemistry tree & 0.0692 & 0.079 & 0.1899 & -0.171 & 0.0934 & -0.039\\
 & Token CNN & 0.0753 & -0.002 & 0.1641 & -0.011 & 0.1019 & -0.132\\
 & Training row mean & 0.0759 & -0.011 & 0.1742 & -0.074 & 0.1450 & -0.612\\
 & Oracle test mean & 0.0751 & 0.000 & 0.1623 & 0.000 & 0.0900 & 0.000\\
\bottomrule\end{tabular}
\caption{Row-weighted fixed-ensemble prediction on matched rows within each sensitivity; cohorts and models are those of Section~\ref{sec:setup} and Table~\ref{tab:floor}. Grouped coverage is 1,003 and 2,607 rows; APP and S7 retain 1,839 and 118. The training row mean is a deployable constant from each training partition; the oracle test mean reads the evaluation labels and is a diagnostic, not a risk floor.}\label{tab:activity}
\end{table}

\section{Calibration, protocol and ranking sensitivity}\label{sec:calibration}
A relative MSE advantage here does not establish useful discrimination: on the released-label evaluation the GNN's prediction SD is 0.0072 against a label SD of 0.4028. Comparisons are also protocol-specific, and because populations, support rules, optimisation and selection changed together, a reversal identifies no single cause. The squared-bias share, the one protocol reversal and the full tables are in Appendices~\ref{app:results} and~\ref{app:movedconventions}.
\section{Measured chemistry contrasts}\label{sec:measured}
\looseness=-1 For B2 (Section~\ref{sec:setup}), both endpoints of every pair and their sequence groups are held out of training and selection. The pair-GNN equal-component MSE minus the training-fold mean is $-0.0125$ $[-0.0707,0.0464]$ and its within-assay correlation $0.266$ $[-0.045,0.493]$, both intervals include zero, and every pair prediction carries the same sign, matching the majority rule at 111/137. Sequence-specific measured chemistry prediction is not established.
\begin{table}[ht]
\centering\small\setlength{\tabcolsep}{4pt}
\begin{tabular}{@{}lrrrrrrrr@{}}\toprule
Panel/model & Inputs & Forced & $\mathrm{MSE}_{165}$ & $(\Delta\mu)^2$ & AS & SS & NA & Pred.\,NA\\\midrule
Measured panel & 165 & -- & 0.079894 & -- & 0.041341 & 0.020415 & 0.018138 & --\\
R1 GNN & 90 & 0 & 0.077350 & 0.000041 & 0.042930 & 0.016196 & 0.018183 & 0.000002\\
R1 no-message & 90 & 0 & 0.077353 & 0.000005 & 0.042811 & 0.016371 & 0.018166 & 0.000002\\
Chemistry tree & 90 & 0 & 0.078546 & 0.000012 & 0.043859 & 0.016079 & 0.018596 & 0.000146\\
Token CNN & 90 & 0 & 0.077276 & 0.000211 & 0.042819 & 0.016072 & 0.018174 & 0.000005\\
\bottomrule\end{tabular}
\caption{Primary B3 on the complete 15 by 11 panel, equal weights, reusing source-held-out endpoints. Exact complete inputs force no cancellation in any of the 140 rectangles, for all ten preprocessing instances. $\mathrm{MSE}_{165}$ is endpoint MSE over 165 states, not the contrast $\mathrm{MSE}_{q}$ of Table~\ref{tab:floor}. AS, SS and NA are antisense, sense and non-additive components; row one, being centered energy, has no grand-mean term. Pred.\,NA is each model\textquotesingle s own predicted nonadditive energy, and the four components sum to $\mathrm{MSE}_{165}$ exactly (Appendix~\ref{app:b3decomp}).}\label{tab:B3}
\end{table}

\begin{figure}[ht]\centering\includegraphics[width=\textwidth]{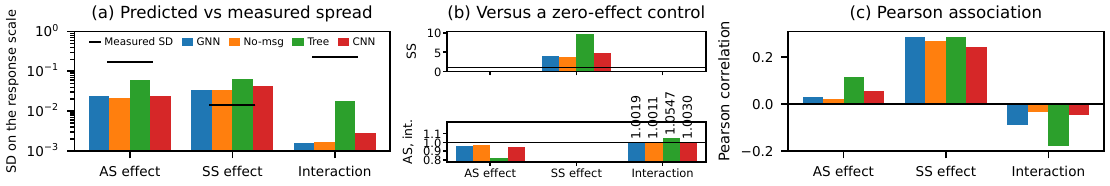}
\caption{Source-held-out predictions on the primary panel, from the ten saved seeds. (a)~Predicted against measured spread per family. (b)~Each family\textquotesingle s error over its zero-effect control, unity being the control, in two sub-panels because sense reaches 9.85; interaction ratios are printed. (c)~Pearson correlation with the measured effect.}\label{fig:b3diag}\end{figure}
Endpoint prediction on the same panel is weak: GNN $R^2=0.0318$, Pearson 0.1848, predicted SD 0.04026 against recorded 0.28266, so this is no interaction-specific failure of an otherwise strong predictor. Nor is every effect attenuated: sense-marginal predicted SD is 0.03447 against recorded 0.01395, while antisense and interaction predictions have too little spread (Fig.~\ref{fig:b3diag}).

\looseness=-1 The 140 contrasts share states and one background, so no independent-rectangle interval is computed. Table~\ref{tab:B3} asks which parts of the endpoint predictions are inaccurate. It separates mean, strand-main-effect and nonadditive errors; Table~\ref{tab:floor} instead separates contrast error according to what the encoding permits. The centered measured energy divides into 51.7\% antisense main effects, 25.6\% sense main effects and 22.7\% nonadditive effects. The GNN's error is 0.04293 antisense and 0.01620 sense against 0.01818 nonadditive, so the deficit lies in the main effects, not specifically in interactions (Appendix~\ref{app:b3decomp}). The predictions are nearly constant: all 140 GNN ensemble contrasts are nonzero but their SD is 0.0015893 against a measured 0.2224421, and the resulting MSE 0.068335 sits only 0.19\% above the 0.068205 zero-interaction control.

Reference shifts preserve this ordering, as do 23 of the 24 single-margin omissions; the one reversal and the full diagnostics are in Appendix~\ref{app:b3sensitivity}.

\paragraph{Reproducing the published predictors.} \looseness=-1 We did not complete native end-to-end evaluation of either published modified-siRNA predictor: MEG-mod \citep{meg2026} builds modification tokens from variables its release never defines, and ENsiRNA-mod \citep{ensi2025} needs licence-gated per-duplex Rosetta geometry that is absent. All 156 B2 pairs, 165 B3 endpoints and 140 contrasts therefore stay unassessed, with unknown supervised overlap. The rank and floor computation, which needs no fitted predictor, did complete on each release\textquotesingle s chemistry branch alone: each branch has full rank and a floor of exactly zero on the rectangles it encodes (Appendix~\ref{app:independentattempt}).
\section{Training-support diagnostics}\label{sec:support}
\looseness=-1 We distinguish relation types absent from every training graph from chemistry features zero on every training node. For absent relations, a layer-by-layer induction proves that training predictions are independent of the corresponding parameters (Appendix~\ref{app:trainingindependence}). This holds at every parameter value, not merely at one zero-gradient point, and motivates a perturbation test requiring no fitting. Validation selection and regularisation remain relevant: AdamW's decoupled weight decay can change zero-gradient parameters, so these are not never-updated weights. Chemistry-feature support is examined separately in the restoration and training-support experiments below.

\looseness=-1 We test this by perturbing those blocks directly (full protocol in Appendix~\ref{app:perturbation}). Every training prediction stays bitwise equal, as the induction requires. Held-out behaviour moves only slightly and only for R0: the largest of 1,839 APP prediction changes is 0.003300, and top-five selection changes in 15 of 102 cases, against 0 of 102 for R1. The B3 contrast changes are exactly zero, so this mechanism does not explain the B3 attenuation.
\begin{figure}[ht]\centering\includegraphics[width=\textwidth]{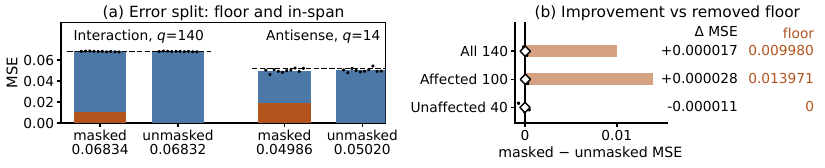}
\caption{Mask-restoration follow-up on the R1 GNN, ten paired seeds; interaction $q=140$,
antisense $q=14$. (a) Ensemble floor (rust) and in-span error (blue), each
seed\textquotesingle s total a dot, control dashed. The restored encoding has zero oracle floor. (b) Improvement
per seed (dots) and ensemble (diamond) against the masked floor per subset (bar);
improvement is masked minus restored, so positive means error fell. The 100 affected
rectangles hold all of the floor.}\label{fig:mask}
\end{figure}

\textbf{Restoring representability.}
\looseness=-1 Refitting the R1 GNN without the training-only node-support mask restores 165 classes and rank 140, reducing the unrestricted oracle floor from $0.009980$ to zero. 
Architecture, optimizer, selection rule, splits and seeds were unchanged, and the masked refit reproduced saved predictions to $6\times10^{-17}$ (Appendix~\ref{app:maskrestore}).
Interaction MSE fell from $0.068335$ to $0.068318$, an improvement of $0.000017$, which is $0.17\%$ of the removed floor, with paired-seed changes from $-0.000176$ to $+0.000262$ (Fig.~\ref{fig:mask}), every one inside the protocol's pre-declared near-zero category of $|\Delta|<0.004990$. The refit changes the evaluation encoding and the selected checkpoint together, so we separate them by inference alone. Crossing the selected checkpoint with the evaluation input gives four ten-seed ensemble values of interaction MSE over the 140 rectangles:

\begin{center}\small
\begin{tabular}{@{}lrr@{}}\toprule
Selected checkpoint & Masked evaluation input & Restored evaluation input\\\midrule
Masked arm & 0.0683350733 & 0.0683368487\\
Restored arm & 0.0683138347 & 0.0683180493\\
\bottomrule\end{tabular}
\end{center}

\looseness=-1 Using masked-minus-restored MSE, changing inputs first gives $-0.0000018$, followed by $+0.0000188$ from switching checkpoints. Reversing the order gives $+0.0000212$ from checkpoints and $-0.0000042$ from inputs. The split is therefore order-dependent. In either order the input change slightly worsens ensemble error, and the selected-checkpoint difference produces the net improvement of $0.000017$. Both components change sign across seeds, and a bootstrap over seeds puts zero inside the 95\% interval of each. We call the second term a selected-checkpoint difference because identical parameters at matching epochs could be confirmed directly only for the five seeds whose selected epoch did not change. The original protocol hypothesized an improvement no larger than the removed floor; that was not a valid general upper bound, since refitting can also change error \emph{within} the original contrast space, so we retain the floor only as a descriptive scale.
The mask zeroes 44 of the 93 node-feature columns, all zero on every training node, so their embeddings receive no training-data gradient. Only the columns that separate otherwise-distinct chemistries collapse states, so the rank loss localises to a few named columns (Appendix Table~\ref{tab:app_collisions}).
The two siRNA interventions are 76 fits and 30,137 optimizer updates (Appendix~\ref{app:fitledger}).

\textbf{Adding chemistry to the training set.} \looseness=-1 Within the audited cohort, every row carrying any of the five chemistry features responsible for the B3 rank deficit belongs to the held-out source family, so a source-held-out experiment that supplies those features in training cannot be built from this cohort. We therefore ran a comparison admitting training examples from the otherwise held-out source. Both arms have full representability and share the same core training rows; they differ in which 200 rows are added, rows carrying those features or rows that do not, with per-background counts matched. Ensemble interaction MSE was $0.06847384$ without the carriers and $0.06845552$ with them; five of ten seeds improved and all ten changes fell in the pre-declared near-zero category. On the 25 rectangles whose corners contain none of the five driver features, MSE increased by $0.000168$, more than the whole-panel change; retraining moves shared parameters, so these rectangles are not insulated from the added rows. This intervention produced no consistent improvement. The arms still differ in training-label distribution. The added rows provide training support for all five rank-relevant features; 26 of the 44 restored columns remain unsupported in both arms (Appendix~\ref{app:supportdecouple}).

\looseness=-1 The floor bounds error under the fixed encoding, not improvement after changing it. In the synthetic sweep, the registered prediction held in 8 of 42 nonzero-deficit fits and failed in 34 (Appendix~\ref{app:floorsynthetic}).
\section{Additional encodings and a released splicing predictor}\label{sec:mavenn}
MAVE-NN \citep{mavenn2022} ships a model of the BRCA2 exon-17 subset of the splicing assay of \citet{Wong2018SpliceSites}, measuring 5-prime splice-site selection as log10 percent spliced in (PSI): not modified siRNA, and nothing here closes that gap.

The two encoders fall on opposite sides of the diagnostic. Unlike the masked siRNA encoding, the splicing encoder distinguishes all 3,696 measured states. It preserves the design rank of $1{,}986$ across 2,083 contrasts and therefore has zero floor. Its 97 relations are design dependencies, not encoding restrictions (Section~\ref{sec:structure}). \looseness=-1 This null second instance shows the diagnostic does not assign a positive floor automatically. A frozen-protocol extension finds the same split (Appendix~\ref{app:breadth}): a wider panel from the same screen loses 730 of 1,845 directions, while a protein-binding panel under the same one-hot encoding loses none; neither of those adds an independent encoder. \looseness=-1 Two caveats on the released MAVE-NN predictor, detailed in Appendix~\ref{app:mavenn}: the released checkpoint's metadata do not list all training-sequence identities, so the recorded count alone does not certify held-out provenance; and our native evaluation reproduces the published information values closely but not bitwise.

\begin{table}[ht]
\centering\small
\begin{tabular}{@{}lrrrrr@{}}\toprule
Quantity & $N$ & Model MSE & Matched control & Difference & Pearson\\\midrule
Endpoints, released test split & 6078 & 0.084224 & 0.259250 & -0.175026 & 0.822247\\
Endpoints, rectangle population & 3696 & 0.089748 & 0.246508 & -0.156759 & 0.797282\\
Single substitutions & 12486 & 0.155262 & 0.302181 & -0.146919 & 0.697187\\
Interactions (primary) & 2083 & 0.361690 & 0.456701 & -0.095011 & 0.462959\\
Interactions, library 2 & 1947 & 0.327946 & 0.455823 & -0.127877 & 0.531564\\
\bottomrule\end{tabular}
\caption{Released MAVE-NN pairwise GE predictor on the released MPSA test split, in
$\log_{10}$ PSI; the checkpoint-provenance limitation is stated in
Appendix~\ref{app:mavenn}. Endpoint controls are the constant fitted to the released training-plus-validation population;
contrast controls are matched zero-effect predictions on identical identities and
weights. The nested rectangle weights are defined in Appendix~\ref{app:mavenn}; for
library~2 they are restricted to surviving identities and renormalized. Negative differences favour the model; differences are computed before rounding.}\label{tab:mavenn}
\end{table}
\begin{figure}[ht]\centering\includegraphics[width=\textwidth]{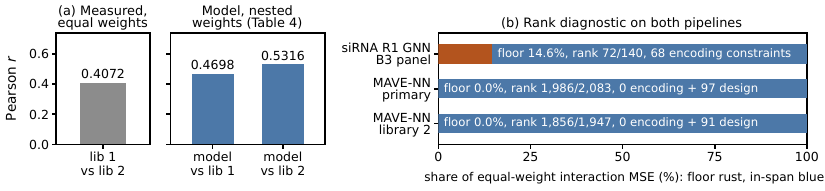}
\caption{MAVE-NN on RNA splicing, not modified-siRNA efficacy. (a) Two separate comparisons on the $1{,}947$ shared rectangles: measured library agreement under equal rectangle weights, model agreement under the nested weights of Table~\ref{tab:mavenn}. Neither is a noise ceiling or a bound on model error. (b) Encoding floor (rust) and in-span error (blue) for both pipelines, as
reported in the text.}\label{fig:mavenn}
\end{figure}
\looseness=-1 Every quantity in Table~\ref{tab:mavenn} beats its matched control, but measured-effect performance is not implied by endpoint performance: interaction Pearson is $0.4630$ against $0.7973$ on the linked endpoint population, and predicted interaction SD $0.3678$ against measured $0.6756$. The replicate comparison shows limited agreement between measured interactions without apportioning the discrepancy between prediction error, measurement error and shared bias.

\paragraph{Fixed-size support intervention.}
\looseness=-1 At training size $14{,}238$ and fixed architecture, replacing $2{,}000$ of each arm's training rows ($14.05\%$) reduced affected-interaction MSE by $0.04415$ and increased unaffected-interaction MSE by $0.01441$; the matched zero-effect MSE on the affected subset was $0.3756$. The affected reduction held in all three paired seeds and the unaffected increase in all three (Appendix Fig.~\ref{fig:support}). The replacement changed feature support, sequence backgrounds and label distributions together, so the comparison does not isolate the contribution of \emph{feature support}: lexicographic selection also shifts background composition, the two blocks differ in label mean and spread, and the shared validation set scoring early stopping sits almost entirely in one background stratum the arms represent very unequally. In a descriptive follow-up under the evaluation weights, aggregate MSE over all $2{,}083$ contrasts increased by $0.00116$, with mixed directions across the three seeds. The protocol declares the affected subset primary and the unaffected a specificity check. A row-for-row background-matched sibling arm would fix the composition and is feasible here; we did not run it, and label this arm descriptive.

The fit and attempt accounting for this arm is in Appendix~\ref{app:support}. The withheld coefficients receive no data gradient but remain under the L2 penalty, so no training-independent block is claimed here.

\paragraph{An independent lossy encoder.} We also audited an encoder we did not build: radius-two count fingerprints, the setting taken by the count-fingerprint configuration in the code released with \citet{moayedpour2024}, on 1,100 cells of the public AGILE ionizable-lipid library \citep{agile2024} in the redistribution of \citet{lantern2025}, carrying AGILE\textquotesingle s HeLa readout. The encoding merges them into 920 classes. For 800 anchored interaction contrasts, $\operatorname{rank}(C)=800$ and $\operatorname{rank}(CH)=740$. Every row of $CH$ is nonzero and no two are equal or proportional, so the 60 restrictions are invisible to duplicate-row and single-contrast checks. The empirical floor is 0.3150033, $2.54\%$ of recorded contrast energy. Under a different contrast design on the same panel and encoder, every restriction does appear as a duplicate row, so whether row checks suffice depends on the design and not only on the encoder. At radius eight the same fingerprint separates all 1,100 lipids, so no restriction remains there. This audits a published input map on a lipid panel; it is not a predictor evaluation and does not bear on modified siRNA (Appendix~\ref{app:lossy}).
\section{Related work}\label{sec:related}
\looseness=-1 We position the work by the assumption each neighbouring line needs, and why it is unavailable here. The statistical ancestor is factorial estimability \citep{boxhunter1961,wuchen1992}: there the \emph{design} limits which contrasts are estimable, whereas here the raw design permits all 140 and the \emph{encoding} adds the restriction.

\emph{Attribution and explanation} methods describe a fixed predictor. Our diagnostic asks which measured contrast vectors its input encoding can represent, before a decoder is selected. \citet{bilodeau2024} show complete linear attributions can fail to beat chance; full rank of our contrast map neither establishes nor refutes that, and we claim no formal separation. \citet{azzolin2025,azzolin2026} study architectural limits and degenerate explanations in GNNs; we instead ask which contrast vectors are attainable at all. Faithfulness to a model and agreement with a measured effect remain distinct targets \citep{chen2020}.

\emph{Functional-decomposition} methods fix a weighting-dependent split of a fitted predictor \citep{lengerich2020pure} rather than asking what its input map can express; ours is reachability under a fixed encoder, with no population-distribution assumption and no fitted predictor.

\emph{Identification} results need conditions we cannot verify on a fixed finite panel. The closest, \citet{kuskova2026real}, assumes continuous support, so its population theorem does not apply here and supplies no encoding floor. \citet{damour2022} show underspecification can leave in-domain performance unchanged while deployment behaviour differs; we prove training-prediction independence for specified blocks, not equal validation performance.

\section{Discussion and limitations}\label{sec:limitations}
\looseness=-1 Rank and training support are distinct checks. On B3, removing the mask restored representability but did not supply training support, and prediction error changed little. In the masked GNN, $85.40\%$ of fitted interaction error lies within the attainable space; the diagnostic does not explain its cause. The decomposition puts the larger error components in the main effects. Two siRNA panels here have a positive floor, B3 and the expanded Bramsen panel, and on both, grouping duplicate rows recovers every restriction. Encoding restrictions that grouping misses appear in one measured case: a lipid library under one featurizer setting and one of two contrast designs, with no fitted model (Appendix~\ref{app:lossy}). Independent validation of modified-siRNA efficacy also remains incomplete, since neither the splicing result nor the splicing support intervention transfers to chemical modification: B3 has one background, B2 jointly changes four positions, and APP and S7 are retrospective, with related sources. A source-exposed arm that gave training support to the restored columns with nonzero marginal rank gain also produced no consistent gain (Appendix~\ref{app:supportdecouple}). Shared-control covariance, the 299 primary/released discrepancies, the orientation of Davis Table S1 and the dose units all stay unresolved (Appendix~\ref{app:robustness}). Seeds are not biological replication; no wet-lab, causal or clinical validation is claimed, and no broadly useful reliability score or independent generalisation is established.
\label{main:end}
\clearpage

\clearpage
\label{references:start}
\bibliographystyle{iclr2027_conference}
\bibliography{references}
\label{references:end}\clearpage
\input{appendix}
\end{document}

%% file: appendix.tex
\appendix\renewcommand{\thesection}{A\arabic{section}}\numberwithin{equation}{section}
\label{appendix:start}
\setlength{\tabcolsep}{4pt}
\section{Data, chemistry, experiment, and grouping}\label{app:data}
\paragraph{Appendix roadmap and evidence map.}
Data, chemistry and partitions: Section~\ref{app:data}; full forward map, losses and derivatives: Sections~\ref{app:architecture}--\ref{app:training}; metrics and uncertainty: Section~\ref{app:metrics}; primary-panel evidence: Section~\ref{app:b3}; adjudication: Section~\ref{app:robustness}. Table~\ref{tab:activity} also uses the complete decomposition proof (\ref{app:decomposition}). Table~\ref{tab:B3} and Figures~\ref{fig:b3diag}, \ref{fig:encoding} resolve to the endpoint/marginal tables, exact rank proof (\ref{app:contrastrank}), chemical factorization (\ref{app:chemicalfactor}) and sensitivities (\ref{app:b3sensitivity}). The perturbation diagnostics quoted in Section~\ref{sec:support} resolve to the induction in Section~\ref{app:trainingindependence}, the numerical utility protocol in Section~\ref{app:utility} and rows of \texttt{perturbation\_summary}. Figure~\ref{fig:workflow} resolves to the architecture; the focused grouped sensitivity plot in the canonical store carries the matched activity and exact metric decomposition. Table~\ref{tab:utility} and Figure~\ref{fig:utility} resolve to the complete utility protocol (\ref{app:utility}); the independent attempt is in Section~\ref{app:independentattempt}. The canonical report supplies each object's store/CSV mapping; every required mathematical argument is written here.
\subsection{Notation and the observed experiment}
\begin{definition}[Observation and endpoint]
An observation $i$ consists of source identifier $u_i$, guide and passenger sequences $s_i=(s_i^g,s_i^p)$ in reported $5'$--$3'$ order, chemistry state $m_i$, assay metadata $a_i$, and a reported activity $y_i$. Chemistry includes positional modification names, available linkage types, terminal groups, and reported stereochemistry. Metadata include target, accession, cell/species, concentration, time, delivery, table, and response normalization where recorded. Unavailable fields remain missing. A parsed observation is not assumed to be an independent biological replicate.
\end{definition}
\begin{longtable}{L{.19\textwidth}L{.72\textwidth}}\toprule Symbol & Meaning\\\midrule\endfirsthead\toprule Symbol & Meaning\\\midrule\endhead
$y_i$ & Inhibition fraction, by cohort: the grouped benchmark uses the released inhibition fraction; APP uses $(100-\text{primary mean percent mRNA remaining})/100$; the primary B3 panel uses $1-\text{relative eGFP}$ (Appendix~\ref{app:b3}); Davis S7 is kept on its original response scale. No clipping.\\
$\mu(s,m,a)$ & Population mean of that assay endpoint; not directly known from a reported sample mean.\\
$\hat f_k(x_i)$ & Committed prediction of model $k$ from the permitted molecule/context features $x_i$.\\
$g(i),q(i)$ & Study/source-linked component and sequence-only component, respectively.\\
$w_i$ & Within-partition observation weight $n/(G n_{g(i)})$, where $G$ is the number of study components and $n_g$ is its size.\\
$X,A,b,C$ & Node features, relation adjacency tensor, node mask, and training-masked numeric context; here $C$ is the context tensor, not the contrast matrix $C$ of Section~\ref{sec:structure}.\\
$L,M,F,R,d$ & Maximum strand length 32; total node slots 64; 93 node features; 8 relation types; hidden width 32.\\
$a,S_a,S_a^*$ & Ranking pool, five model-selected candidates, and five highest measured candidates.\\
$d_j,\hat d_j$ & Measured and predicted differences for matched chemistry pair $j$.\\
$B$ & Number of sequence-component bootstrap draws, fixed at 10,000 in the current revision.\\\bottomrule
\caption{Notation and conventions.}\\
\end{longtable}
This is a retrospective, heterogeneous observation collection. No shared latent Gaussian experiment, independent five-stratum sampling model, private-anchor law, pathwise differentiability, or unknown-profile search class is assumed. Ordinary finite moments are enough to define empirical losses because all admitted numeric values and saved predictions are finite. We do not infer population MSE bounds from those finite empirical sums.

\subsection{Row eligibility and identity reconciliation}
All raw rows are preserved before exclusions. ENsiRNA source identifiers must be seven-to-nine-digit numeric publication identifiers or the explicit US20120088815A1/EP2415869A1 family string. Other release identifiers remain unresolved. Guide and passenger bases must consist of uppercase A, C, G, U, T. Source, both complete parsed chemistry states, reported concentration, and outcome define an exact-release duplicate; repeated observations with distinct outcomes remain measurements but are grouped together for splitting. Two exact duplicate records are removed.

Each concentration must be positive and each outcome finite. Each strand must admit the graph alignment specified below. Chemistry annotations with inconsistent modification/position arity, indices outside the strand, unknown base changes, or unresolved geometry-changing entries are excluded. The strings ``Mutation'', ``Inverted abasic'', and ``Assymetric siRNA (3-End Overhang)'' are not silently translated into canonical bases or guessed duplex structures.

APP rows are first excluded if the cell field is ``Transgenic mice'' or administration time is not exactly 24h. Empty chemistry is not unmodified RNA. Three description/base disagreements are rejected, as are unsupported base strings, nonpositive dose, nonfinite outcome, or failed duplex alignment. Exclusions are assigned the first failing reason in this order, so counts are mutually exclusive within the parser. A row with several deficiencies is not counted several times.

The first reconciliation admits 5,033 rows: 2,927 ENsiRNA and 2,106 APP. The primary-table step subsequently excludes 267 APP rows for unreconciled exact chemical strings, leaving 4,766. Its separate status supersedes the intentionally preserved earlier-stage ``pending'' B2 field. No scientific fit used the earlier population.

\subsection{Complete chemical parser}
For ENsiRNA, \texttt{anti raw seq}/\texttt{sense raw seq} give bases, \texttt{anti mod}/\texttt{sense mod} list modification names separated by an asterisk, and matching position fields list comma-separated one-based indices within each modification group. A position token is converted to an integer and checked against strand length. The literal zero/zero pair denotes ``unmodified as annotated''; an empty field does not. Multiple entries at a position become a sorted set, and an otherwise untagged position is recorded as unmodified within the release's annotation scope. This is not a claim of complete physical chemical characterization.

Only explicit aliases are merged: 2-O-Methyl ribose and 2-Methoxy become 2-O-Methyl; 2-Deoxy-2-Fluoro becomes 2-Fluoro; 2-Deoxyribonucleotide and 2-Deoxythymidine become 2-Deoxy; UnLocked nucleic acid becomes Unlocked nucleic acid; 5-phosphate ribose becomes 5-Phosphate. Other names remain distinct, including source spelling. Indexed phosphorothioate and boranophosphate tags remain node features with an unresolved direction flag, and ENsiRNA backbone edges are encoded as unresolved under their annotation-only PO default, so an indexed tag is never promoted to a resolved directed PS bond. Terminal annotations are retained for 5-Phosphate and Dodecyl derivative; unreported stereochemistry stays unreported.

APP indexed descriptions are separated by \texttt{||}, with each nucleotide item \texttt{index*description}. Base words adenosine, guanosine, cytidine, uridine and thymidine must agree with the raw base string. Sugar chemistry is selected in the order hexadecyl, Glycol Nucleic Acid, Deoxy, Methyl, Fluoro; a description matching none is rejected and every position must be described exactly once. A phosphorothioate descriptor tags its node and the following within-strand bond, becoming terminal PS at the final position, while unmarked bonds are PO under this convention. A vinyl descriptor adds a vinylphosphonate tag and terminal entry. The raw token \texttt{(Tgn)} identifies the reported S-GNA state; other stereochemistry remains unreported.

An unnumbered GalNAc description is localized to the final position only if the raw chemistry ends in \texttt{L96}; indexed GalNAc descriptions retain their stated index, and L96, vinylphosphonate, phosphate and other terminal names are kept in parsed records. Graph terminal coordinates are capped at the last observed nucleotide as an implementation convention, with original coordinates and names retained in provenance. No chemical string is lowercased, since case can encode sugar modifications, and T and U remain distinct model bases whose equivalence is used only in sequence grouping and canonical pairing recognition.

\subsection{Exact primary-table verification}
The APP rows come from two Alnylam patent families, US20220002724A1 (priority 19 December 2018) and US20240117349A1 (priority 29 January 2021); the CMsiRNAdb APP release also carries an Ionis family, US20220380773A1, which is not admitted here. For US20220002724A1, extracted tables 4 and 7 provide means and SDs at 10 and .1nM for two cell types each. Tables 17 and 18 provide agent, mean, SD, and dose for Be(2)C and Neuro2A, respectively. Table 29 is also parsed but contributes no final admitted pool. US20240117349A1 Table 6 provides Be(2)C means and SDs at 10, 1, and .1nM. Lookup keys are publication, agent, exact cell label, and numeric dose. Tables are parsed in the foregoing order; later occurrences replace earlier identical keys. Every admitted row also agrees with its released inhibition within .015 percentage points. There are zero primary-label corrections in the final records.

Chemical strings are reconstructed from wrapped primary rows, joining continuation cells until a blank row or a changed row width. Spaces are removed for string comparison; case and chemical tokens are preserved. An exact agent match must yield the exact guide and passenger chemical strings in the database export. The accepted numeric response is $(100-\text{primary mean remaining})/100$, with the primary SD retained and divided by 100 when used on the activity scale. Negative inhibition is retained. Tables 17/18 identify the AD-1955-relative mRNA endpoint in accompanying source text. No assumption that all normalized assays have identical biological scales is made.

The final 17 assay pools are: Table 4, two cells times two doses, 209 rows each; Table 7, two cells times two doses, 91 each; Tables 17 and 18, one cell times three doses each, 90 each; US20240117349A1 Table 6, one cell times three doses, 33 each. These 17 pools hold $209\cdot4+91\cdot4+90\cdot6+33\cdot3=1{,}839$ measurement rows over a far smaller number of distinct duplexes, so the cohort size is a row count, not a count of siRNAs. All use the corresponding source RNAiMAX screening protocol, and the curated database records 24h incubation for every admitted row. The source protocol disagrees for six of the 17 pools: US20220002724A1 Example 5, which carries Tables 17 and 18, states 48 hours for those 540 rows. The admission rule keyed on the curated field rather than on the primary protocol, and the disagreement is unresolved. Primary tables contain endpoint SDs but do not resolve replicate $n$, raw wells, or covariance from shared normalizers for these admitted records. We do not infer these quantities from the SD.

\subsection{Protected components and current grouped partitions}
Two transitive closures define the grouping units. The study-linked union-find joins all observations sharing a known source family and all observations connected by a shared strand 13-mer. The sequence-only union-find uses the latter criterion alone. Both strands contribute contiguous 13-mers after replacing T by U; a match can connect opposite strands or connect records indirectly. This is a conservative sequence-identity safeguard, not a proof of biological independence. Group identities are deterministic hashes of the retained memberships, and source, record and outcome-bearing identifiers never become model inputs.

The existing five outer ENsiRNA test groups are retained exactly. Every protected study-linked component occurs in one outer test with no sequence or source-family intersection with that test's training or validation partition; historical internal-test observations are retired into this grouped development population, and the APP patent family and all Davis S7 rows remain outside ENsiRNA model selection.

Within each outer training population, sort whole study-linked components by decreasing number of observations, breaking ties by their identifier. Assign the next component to the currently smallest of three inner folds, ties broken by index, and for each inner validation fold train on the other components; no protected component is split to balance counts. The deployment inner design applies the same construction to all ENsiRNA components. The final validation fold is the one with fewest validation observations, with an index tie-break, so the choice uses only membership counts; final-seed fits use its complementary training components and monitor this grouped holdout. The deployment fit therefore updates parameters on a declared subset of ENsiRNA while the remaining rows serve validation, and there is no all-data refit.

Pair outer tests retain the previous assignment of the twelve sequence components to four tests. Three inner folds are greedily balanced by pair count with whole sequence components protected, endpoint training uses only the measured endpoints in the appropriate component set, and the final validation fold follows the same smallest-count rule. Both endpoint identities and all their sequence-linked observations are excluded from the held-out fold's training and selection; membership tables give every role, and row counts and group counts are distinct.

\section{Exact representation, forward map and support claims}\label{app:architecture}
\subsection{Graph construction and all coordinates}
For each duplex, guide position $i$ occupies slot $i$ and passenger position $j$ slot $32+j$, using zero-based indices. Strands longer than 32 are rejected, never truncated. $X\in\R^{64\times93}$ and binary active-node mask $b\in\{0,1\}^{64}$ represent these ordered slots. Coordinates 0--4 are A,C,G,U,T, coordinates 5--6 identify strand, and 7--38 give within-strand position. Coordinates 39--83 are the 45 positional modification names of the fixed schema, enumerated in the saved preprocessing support and the preserved v10 appendix rather than reprinted here. Coordinates 84--88 indicate unreported stereochemistry, S-GNA, unresolved indexed-linkage direction, strand start and strand end; 89--92 indicate L96 GalNAc, vinylphosphonate, 5-phosphate and other terminal groups. Schema exposure included historical label-free test molecules; this is disclosed rather than described as a training-only vocabulary. Effective masks in the current campaign are fitted separately on each training partition.

Adjacency $A\in\{0,1\}^{8\times64\times64}$ uses $A_{r,v,u}=1$ for source $u$ and receiver $v$. For adjacent positions, indices 0,1,2 encode forward PO, PS and unresolved linkages, and 3,4,5 their reverse directions. The annotation-only ENsiRNA PO assumption is unresolved, including indexed PS tags whose bond direction was not established. Relations 6/7 are bidirectional canonical/mismatched pairing. Counts sum directed tensor entries, so each paired strand position contributes two edges and a backbone bond contributes one forward and one reverse edge. Padded nodes contribute no edges.

For guide/passenger lengths $l_g,l_p$, enumerate offsets $o=-l_p+1,\ldots,l_g-1$. Match guide $i$ to passenger $j=l_p-1-(i-o)$ whenever both indices are valid. Require overlap at least 14, then maximize lexicographically the number of Watson--Crick matches, negative mismatch count, overlap, negative absolute offset and negative offset. Complementarity uses A--U and C--G, treating T as U only for matching. Admit a selected pairing fraction of at least .8. The inferred alignment, not an experimental tertiary structure, supplies these edges. Overhangs have no pairing edge, and there is no target-mRNA or inter-observation graph.

\subsection{Training-fitted masks and context}
Let $T$ contain only the current training indices. A node column is enabled if it is nonzero at any active training node; otherwise set that column to zero on every split. Padding is excluded from this support computation. This mask is identical for original/corrected graph and no-message arms. It may discard meaningful chemistry that has no training support, including shifted-cohort terminal flags. No held-out outcome or held-out support count selects the mask.

The eight raw context coordinates are verified $\log(1+c)/\log(101)$ in nM (zero when the unit is unverified), an unverified-unit flag, time divided by 24 (zero when missing), a missing-time flag, a recorded-cell flag, inferred pairing fraction, guide length/32 and passenger length/32. Every coordinate constant across the training partition is set to zero on all splits. No test-fitted centering, variance denominator or clipping is used. The recorded-cell flag is not a cell-identity embedding, and an unverified dose is not converted into nM. This is the conservative historical C1 rule, held fixed across the primary four-arm intervention.

\subsection{Original and corrected relation transforms}
Every graph or no-message model owns two tensors $W^\ell\in\R^{8\times32\times32}$. R0 uses these matrices directly. R1 uses
\begin{align}
\widetilde W_0&=W_2+q_0W_0,&\widetilde W_1&=W_2+q_1W_1,&\widetilde W_2&=W_2,\\
\widetilde W_3&=W_5+q_3W_3,&\widetilde W_4&=W_5+q_4W_4,&\widetilde W_5&=W_5,\\
\widetilde W_6&=q_6W_6,&\widetilde W_7&=q_7W_7,
\end{align}
where $q_r=\1\{\sum_{i\in T,u,v}A_{irvu}>0\}$. This is the complete fallback rule. A new explicit PO or PS edge routes through the common base for its direction; its residual is used only when that type appeared in training. Every training backbone edge uses the corresponding base through the sum above, even if no edge is labeled unresolved. Its realized gradient contribution also depends on the loss and activation state. A positive count is eligibility only: a relation represented by two edges is still rare. Pairing relations have no invented backbone fallback; they are gated separately.

\subsection{Forward pass and exact parameter count}
All hidden states are row vectors. With $D$ denoting dropout at .1 and $\operatorname{LN}_\ell$ affine LayerNorm,
\begin{align}
h_i^0&=b_i\operatorname{ReLU}(X_iE+e),\\
d_i&=\max\left(1,\sum_{r,j}A_{rij}\right),\\
m_i^\ell&=d_i^{-1}\sum_{r,j}A_{rij}h_j^\ell\widetilde W_r^\ell,\\
h_i^{\ell+1}&=b_i\operatorname{LN}_\ell\left[h_i^\ell+D\{\operatorname{ReLU}(h_i^\ell U_\ell+u_\ell+m_i^\ell)\}\right],\quad\ell=0,1.
\end{align}
R0 replaces $\widetilde W$ by $W$. Concatenate the 64 states in slot order and the eight context coordinates to obtain $v\in\R^{2056}$. The scalar normalized output is
\begin{equation}
z=D\{\operatorname{ReLU}(vW_o+b_o)\}\beta+b,\qquad \hat y=\mu_T+s_Tz.
\end{equation}
For a 32-coordinate state $u$, the normalization is
\[
\operatorname{LN}(u)_k=\gamma_k\frac{u_k-\bar u}{\sqrt{32^{-1}\sum_{j=1}^{32}(u_j-\bar u)^2+10^{-5}}}+\beta_k,
\qquad \bar u=32^{-1}\sum_{j=1}^{32}u_j.
\]
The variance divisor is 32, not 31. In training, $D(u)_k=M_k u_k/.9$ with Bernoulli$(.9)$ masks drawn by the recorded PyTorch random stream; in evaluation $D(u)=u$. ReLU is coordinatewise $\max(0,u_k)$. These definitions fix the actual normalization and dropout conventions.

There is no sigmoid, output clipping, global pooling or permutation-invariance claim. LayerNorm uses PyTorch's $10^{-5}$ epsilon over the 32 state coordinates and learned scale/bias. For a linear layer with fan-in $f$, weights and biases initialize uniformly on $[-f^{-1/2},f^{-1/2}]$. Width-three convolution uses fan-in $32\cdot3=96$. LayerNorm scales initialize to one and biases to zero. Each full $(8,32,32)$ relation tensor receives Xavier uniform initialization with fan-in $32\cdot32=1024$ and fan-out $8\cdot32=256$, hence bound $\sqrt{6/1280}$; this is not eight independently initialized $(32,32)$ matrices. Dropout uses inverted scaling in training and is disabled for prediction and the post-fit data-gradient diagnostic.

The embedding has $93\cdot32+32=3,008$ scalars, self transforms $2(32^2+32)=2,112$, normalization $2(2\cdot32)=128$, relation tensors $2\cdot8\cdot32^2=16,384$, and readout $(2056\cdot64+64)+(64+1)=131,713$. Their sum is 153,345. R1 reuses the owned tensors and adds no parameters. Structurally unsupported residuals remain in the optimizer; total parameter count is therefore not a count of effectively data-trained scalars. Counts, data gradients and parameter changes are all reported separately. The final audit counts exactly nonzero scalar gradients of both the full training activity loss and the actual training objective at every selected final neural checkpoint, with dropout disabled. The latter includes the selected pair term for pair-supervised fits. It also compares every scalar to the same seeded initialization. These realized diagnostic counts and the structural upper bounds are distinct: weight decay can change a zero-gradient scalar, and a scalar with zero gradient at selection could have received data gradients earlier. No full-trajectory ever-active count is claimed.

\begin{proposition}[Unsupported route and the no-message intervention]\label{prop:support}
In R0, a relation absent from every training graph has identically zero training-data gradient through its relation matrix at every layer. A parameter can nevertheless change under decoupled weight decay. In the no-message arm, each receiver's hidden state depends only on its own initial features, its relation-degree vector and the common fitted parameters; it has no dependence on other nodes' hidden states through message exchange.
\end{proposition}
\begin{proof}
For a matrix entry $(W_r^\ell)_{ak}$ the message derivative is
$\partial m_{ik}^\ell/\partial(W_r^\ell)_{ak}=d_i^{-1}\sum_jA_{rij}h_{ja}^\ell$.
When $A_r=0$ throughout training, every summand is zero. All downstream derivatives, including weighted residuals, multiply this zero derivative. The conclusion holds for activity loss and for a sum of endpoint difference losses, because both endpoints are training graphs. At ReLU boundaries the implementation uses its fixed subgradient; the zero multiplicative factor remains zero. AdamW additionally multiplies an updated parameter by a weight-decay factor, so a data gradient of zero is not a statement of zero numerical parameter change.

The no-message aggregation is $\sum_jA_{rij}h_i^\ell=(\sum_jA_{rij})h_i^\ell$. Its derivative with respect to $h_j^\ell$ is zero for $j\ne i$. At layer zero the receiver is a function of $X_i$ alone. Substitution into the displayed update proves by induction for both layers that no other node state enters that receiver. The ordered global readout still uses all receiver outputs. Thus this is a precise neighbor-exchange intervention, not a claim that the final predictor ignores the rest of the molecule.

In R1, each directional backbone message contains $W_2$ or $W_5$, so its derivative with respect to the base matrix sums over all backbone types in that direction. An unsupported type's residual is multiplied by $q_r=0$. The fallback therefore supplies a train-supported computational route whenever training contains any backbone edge in that direction. Nonzero support does not guarantee a nonzero realized gradient under every loss/activation state, nor improved test risk.
\end{proof}
The historical four absent transforms per layer account for $4\cdot2\cdot32^2=8,192$ scalars. They are a subset of the 16,384 relation scalars. The count applies to that training partition and encoding; it is not a universal count for every pair-trained or corrected fit.

\subsection{Non-graph baselines and arithmetic}
The token CNN embeds 93 to 32 coordinates, applies two 32-to-32 width-three Conv1d layers with padding one and ReLU, with .1 dropout between them. It adds their output to the embedded states, masks padding and uses the same ordered readout. Its 140,929 trainable scalars include $3,008+2(32\cdot32\cdot3+32)+131,713$. The convolution operates on all 64 slots in fixed order, including the padded gap, and does not read adjacency.

The chemistry tree receives flattened masked node features, each node's eight relation degrees and context. Columns constant in the training partition are removed. Extra-Trees uses 300 estimators, squared-error splitting, no bootstrap, feature fraction .3, no maximum depth and leaf size 2 or 8; minimum split size is two, minimum weighted leaf fraction and impurity decrease are zero, leaf count is unrestricted, pruning alpha is zero, and warm start, out-of-bag scoring and monotonic constraints are disabled. Guide ridge receives 160 guide base/position indicators and eight context fields; pairwise ridge adds all 276 position pairs among the first 24 guide slots, each with 25 ordered base categories, and both omit passenger chemistry. Nonconstant columns are standardized with unweighted training means/SDs, replacing SD below $10^{-6}$ by one. Ridge uses weighted residual sum of squares, $\alpha\in\{.1,1,10\}$, LSQR tolerance $10^{-4}$, no maximum-iteration limit, a copied design matrix, no positivity constraint and an unpenalized intercept. The iterative solver may terminate at its installed automatic iteration limit; no exact-arithmetic minimizer certificate is asserted. These are explicit current settings, not an assertion that the historical ridge settings were identical.

For exact matched pairs with equal guide/context inputs, any deterministic chemistry-invariant predictor returns equal endpoints and hence a zero difference, by equality of inputs and deterministic evaluation rather than a biological no-effect argument. Floating-point differences near machine precision are not chemistry evidence.

Neural arithmetic is float32; exported metrics are float64. Dense graph storage costs $O(RM^2+MF)$ per observation, one layer $O(RM^2d+RMd^2)$ operations and the ordered readout $O(Md\cdot64)$. These follow from the stated matrix products and are not certified bit-complexity or wall-time bounds. No numerical integration, quantile estimator, fitting certificate or unrelated theoretical backend is used.

\section{Complete selection, fitting and checkpoint protocol}\label{app:training}
\subsection{Experiment matrix and input roles}
The historical v3 experiment matrix enumerates every population, method, configuration, development fold and seed before fitting. The five historical ENsiRNA outer tests are unchanged. The sixth activity population supplies APP/S7 follow-up prediction; its development groups are ENsiRNA only. Every outer training population has three inner grouped validation folds as defined in Appendix~\ref{app:data}. We report the final training/validation partition separately from the total development population. Protected groups are never subdivided to achieve a desired row count.

The four controlled arms are original GNN, training-gated GNN, original no-message and corrected no-message. The token CNN, chemistry tree and guide/pairwise ridge are matched on admissible observation roles. All four graph arms use the same input masks, context, losses, initialization convention, parameter count, sampling and optimization budget. Their independently selected configuration is part of the fitted procedure. Thus the loss interaction compares four declared learning procedures, not a post hoc perturbation of one fixed parameter vector.

Development seeds are 701 and 907. The six neural configurations cross learning rates $10^{-4},3\cdot10^{-4},10^{-3}$ with weight decays $10^{-6},10^{-4}$. For each configuration, average the equal-protected-component validation MSE over three folds and both development seeds; select the minimum, breaking ties by the lexicographic configuration name. The tree crosses leaves 2 or 8 with the same development folds/seeds. Ridge selects $\alpha=.1,1,10$ over the three folds and is deterministic. Every selected configuration and competing score is preserved in \protect\path{v3/tables/selection_scores.csv}.

The final seeds, including every finite outcome, are
\begin{equation}
1103,\ 2207,\ 3301,\ 4409,\ 5519,\ 6637,\ 7753,\ 8867,\ 9973,\ 11027.
\end{equation}
They are matched across the controlled arms. Deterministic methods are fitted once per selected partition. These are ten independently initialized training runs on fixed data, not ten independent data samples. The original three-seed campaign is not silently reused: its selection, schedules and input procedures differ from this frozen revision.

\subsection{Targets, weights and activity objective}
For training partition $T$ with $n$ observations, $G$ protected groups, group membership $g(i)$ and counts $n_g$, define
\begin{equation}
w_i=\frac{n}{G n_{g(i)}},\qquad
\mu_T=\frac1n\sum_{i\in T}w_i y_i,\qquad
s_T=\max\left\{.05,\sqrt{\frac1n\sum_{i\in T}w_i(y_i-\mu_T)^2}\right\}.
\end{equation}
Each group has total weight $n/G$, so $n^{-1}\sum_iw_i=1$. Study-linked groups define activity weights, while pair-endpoint training uses sequence components. Targets for the normalized neural output are $y^{\ast}_i=(y_i-\mu_T)/s_T$, written $y^{\ast}$ rather than with a tilde so that the tilde denotes only the recorded contrast vector of Section~\ref{sec:structure}. Both location and scale are computed from training labels alone and saved with the checkpoint. Prediction restores $\hat y_i=\mu_T+s_T z_i$; no test mean, test variance or clipping enters this map. Activity tree and ridge estimators fit original-scale $y_i$ directly; their recorded training means/SDs serve metadata and constant controls, not a hidden target normalization. Direct pair ridge/tree likewise fit original-scale $d_p$.

The full activity objective is $L_{\rm act}=n^{-1}\sum_{i\in T}w_i(z_i-y^{\ast}_i)^2$. Ordinary activity fitting draws a fresh uniform permutation of training rows each epoch and uses batches of at most 128. Each batch minimizes the mean of its weighted squared errors; the last smaller batch uses its own batch denominator. This specifies the actual stochastic objective, including the last-batch convention. No balanced-group sampler or outcome-based resampling is used.

\subsection{Optimization, selection and resume}
AdamW uses the selected learning rate and weight decay, moment coefficients $(.9,.999)$, denominator epsilon $10^{-8}$, no AMSGrad, no maximization, and the installed automatic foreach/fused selection (both arguments left at None). Capturable and differentiable modes are false and there is no learning-rate scheduler. Weight decay applies to every optimizer parameter including biases and normalization scales, global gradient norm is clipped to five before each update, and loss and gradient finiteness are checked. Default linear initialization, whole-tensor Xavier relation initialization, NumPy/Python/PyTorch seeds and a dedicated CPU permutation generator seeded by final/development seed plus 10,000 make the random sources explicit, with deterministic PyTorch algorithms enabled. GPU0 is used when available with a .30 allocator fraction and two CPU threads; no second GPU or paid resource is acquired.

Every epoch evaluates validation with dropout disabled. The selection score is the validation equal-group MSE on the original activity scale for ordinary activity models. The best checkpoint updates only on improvement greater than $10^{-8}$, and training ends at 500 epochs or after at least 50 epochs and 50 epochs since the best qualifying improvement. A best checkpoint from epoch one is retained if it remains best, and later execution is not misrepresented as later selected training. Final-seed models select checkpoints on the predeclared final grouped holdout and are not refitted on it.

The best checkpoint stores weights, optimizer state, epoch/update count, target mean/SD, preprocessing and the full fit signature. The last checkpoint adds epoch history, best score/epoch, permutation-generator state and CPU/CUDA random states for continuation, and atomic temporary-file replacement avoids advertising a partial checkpoint as complete. A fit is reusable only when protocol, code, preservation-input manifest, train/validation/test indices, method, configuration, seed, loss coefficient and selection criterion match and each output hash verifies; a stage likewise verifies its code, dependency-stage hashes and output hashes. The resumable entry point continues missing stages, and a hash check is not a new fit.

A nonfinite training loss, gradient or validation score is recorded as a failed fit with its history and state; no poor finite seed is removed or restarted on a held-out score. The declared failure output is a flagged training-fold mean prediction with failed-fit rate reported separately, a technical setup failure is distinguished from a training failure with its command and cost retained, and no configuration or budget expansion is selected from APP/S7 predictions.

\subsection{Controls and implementation checks}
Zero difference is an input-invariant prediction control. The training-fold mean difference uses equal sequence-component weights on training pairs. Majority direction chooses the largest training weight among negative, operational tie and positive calls, breaking an exact mass tie toward the smaller sign. It is a direction-only rule; its numeric encoding $-1,0,1$ is not treated as a regression effect magnitude in the current results. The fixed band is $|d|\le .02$ for operational ties. Neither the band nor majority direction uses the held-out fold.

Preflight checks verify protected group exclusion, input-hash preservation, training-only feature support, the exact unseen-relation route, masked padding, nonzero neighbor dependence for graph aggregation and zero off-receiver derivative for the no-message aggregation. Pair labels are checked against their endpoint subtraction. A tied constant is checked to yield exact random-tie ranking. Recorded optimizer updates, finite gradients, selected parameter changes and data-only relation gradients supply execution evidence; a global weight change is not used as proof that every transform learned.

\section{Metric definitions, uncertainty and all finite-sample identities}\label{app:metrics}
\subsection{Activity estimands and constants}
Let $w_i\ge0$ sum to one on the declared evaluated observations. Row weighting uses $w_i=1/n$; equal-component weighting assigns each component mass $1/G$ and divides that mass equally among its records. For fixed predictions $\hat y_i$ define
\begin{align}
\MSE_w&=\sum_iw_i(\hat y_i-y_i)^2,&\operatorname{MAE}_w&=\sum_iw_i|\hat y_i-y_i|,\\
\bar y_w&=\sum_iw_iy_i,&\overline{\hat y}_w&=\sum_iw_i\hat y_i,\\
V_y&=\sum_iw_i(y_i-\bar y_w)^2,&V_{\hat y}&=\sum_iw_i(\hat y_i-\overline{\hat y}_w)^2,\\
R_w^2&=1-\MSE_w/V_y,&\operatorname{bias}_w&=\overline{\hat y}_w-\bar y_w.
\end{align}
Prediction/label SDs are $\sqrt{V_{\hat y}}$ and $\sqrt{V_y}$ with population denominators, not uncertainty in a biological mean. Weighted correlation is the weighted centered cross-product over $\sqrt{V_yV_{\hat y}}$; a zero or negligible denominator is reported undefined and no finite correlation is imputed for a constant. Squared error and $R^2$ use the same observations and weighting.

The training-row constant is the arithmetic mean of training activity; the equal-training-group constant uses the training weights from Appendix~\ref{app:training}. Out-of-fold constants are recomputed in each corresponding training partition. The oracle test mean $\bar y_w$ is a separately marked diagnostic. It uses evaluation labels and is not deployed, included in model selection, or described as a statistical lower bound.

\begin{proposition}[Constant and ensemble decompositions]
For every finite weighted dataset, $\sum_iw_i(y_i-c)^2=V_y+(c-\bar y_w)^2$. If $\bar f_i=K^{-1}\sum_{k=1}^K f_{ik}$ averages $K$ fixed fitted predictions, then
\begin{equation}
\frac1K\sum_k\sum_iw_i(f_{ik}-y_i)^2
=\sum_iw_i(\bar f_i-y_i)^2+\sum_iw_i\frac1K\sum_k(f_{ik}-\bar f_i)^2.
\end{equation}
\end{proposition}
\begin{proof}
Write $y_i-c=(y_i-\bar y_w)+(\bar y_w-c)$, expand the square and use $\sum_iw_i(y_i-\bar y_w)=0$. For the second identity write $f_{ik}-y_i=(f_{ik}-\bar f_i)+(\bar f_i-y_i)$. The cross-term vanishes after averaging over $k$ because $\sum_k(f_{ik}-\bar f_i)=0$ for each $i$. Summing weighted residual squares gives the expression. The nonnegative final term explains why mean seed risk can exceed ensemble risk. It does not certify uncertainty coverage or improved future risk.
\end{proof}
The first identity makes the empirical oracle mean the best constant on that weighted test sample, not a lower bound on nonconstant prediction risk. Poor $R^2$ may coexist with a meaningful relative MSE difference but limits practical usefulness, motivating the spread, bias and ranking diagnostics, and no test-fitted centering correction is presented as a model prediction. Calibration plots bin each model/cohort by prediction quantiles into up to five bins, dropping duplicate cut points, then average prediction and measured activity within each nonempty bin. This uses evaluation labels for a descriptive display, not for a fitted correction.

\subsection{Direct matched loss contrasts and conditional bootstrap}
For predictors $a,b$ evaluated on identical record IDs, retain each loss difference $D_i=(\hat y_i^a-y_i)^2-(\hat y_i^b-y_i)^2$ with sequence, study, source and pool identifiers in the activity loss export (and endpoint/background/fold identifiers for B2). Negative mean difference favors $a$. For row-weighted comparisons, resample whole sequence components with replacement, with the number of draws equal to the number of observed components. The model-versus-constant intervals of Table~\ref{tab:app_intervals} are the exception: they resample whole study-linked components under both weightings (10 groups, 9 with the source excluded), so their width is not comparable to sequence-component intervals. If component $g$ is sampled $m_g$ times, recompute
\begin{equation}
\widehat\Delta^*=\frac{\sum_gm_g\sum_{i\in g}D_i}{\sum_gm_g n_g}.
\end{equation}
Thus large sampled components retain their row mass; replacing this ratio by an average of component means would target a different statistic. For equal-component comparisons use $G^{-1}\sum_gm_g\bar D_g$. The grouped equal-study comparison resamples study-linked components, whereas the shifted-cohort equal-component comparison uses sequence components. Tables identify the unit explicitly.

There are 10,000 bootstrap draws with analysis seed 90413. Endpoints are empirical 2.5th/97.5th percentiles of the resampled statistic. Fitted predictions and all selection decisions are held fixed. This is conditional uncertainty within the observed cohort under a component-resampling approximation; it does not include training randomness, unobserved laboratory shifts, source-family variation in a one-family cohort or unresolved shared-control measurement covariance. No nominal unconditional coverage theorem is asserted. In particular, the small number of study-linked components limits the grouped comparison.

The support-by-message interaction is computed from the joint row-level contrast
\begin{equation}
D_i^{\rm int}=(f_{R1,G,i}-y_i)^2-(f_{R1,N,i}-y_i)^2-(f_{R0,G,i}-y_i)^2+(f_{R0,N,i}-y_i)^2.
\end{equation}
It receives one joint component bootstrap; subtracting independently computed interval endpoints is not used. This contrast concerns the four selected procedures on fixed observed outcomes. It is not a causal claim about chemical activity or an isolated single-weight intervention.

\subsection{Training-seed variability}
For every stochastic method, all ten seed risks and matching seed-level loss differences are reported, along with their mean, sample SD, minimum and maximum. Each seed's prediction file is tied to its selected and executed update count and failure status. Seed-average risk and ensemble risk occupy separate fields. No three-seed $t$ test replaces a fixed-ensemble comparison, and no seed-level interval is claimed to represent new-study uncertainty. With deterministic methods there is one fitted object, not ten invented independent replicates. Failed fits and their declared fallbacks remain explicit.

\subsection{Exact candidate-ranking ties and dependence}
An eligible APP pool is an exact assay/table/cell/dose/time group with at least five observations. Larger activity means stronger inhibition. Assign label rank $r_i$ using average ranks for measured ties, and define percentile $q_i=(r_i-1)/(n-1)$ for $n>1$. Set $k=\min(5,n)$. Let $c$ be the $k$th largest prediction, $H=\{i:\hat y_i>c\}$ and $T=\{i:\hat y_i=c\}$. The selection mass is one on $H$, $(k-|H|)/|T|$ on $T$ and zero below. Report $k^{-1}\sum_i s_iq_i$ and $k^{-1}\sum_i s_iy_i$. The equality test is exact equality of saved float predictions, not an outcome-selected tolerance.

\begin{proposition}[Tie-aware constant reference]
For a constant prediction on a pool of size $n>1$, expected top-$k$ mean measured percentile under uniform tie breaking is exactly $1/2$.
\end{proposition}
\begin{proof}
Every observation is tied and has inclusion probability $k/n$. Average ranking preserves the sum of untied ranks, $\sum_i r_i=n(n+1)/2$, even when measured labels tie. Therefore $n^{-1}\sum_iq_i=[(n+1)/2-1]/(n-1)=1/2$, and $k^{-1}\sum_i(k/n)q_i=1/2$. For a partial boundary tie, exchangeability gives each tied candidate inclusion probability $(k-|H|)/|T|$, proving the fractional rule above.
\end{proof}
All pools and their sizes, shared sequence components and direct tree--GNN / GNN--no-message percentile differences are reported. We do not perform an independent-pool significance test: the pools belong to one related patent family, and backgrounds can recur across pools. A ranking advantage is a descriptive candidate-selection property in those pools, not a biological validation or population guarantee.

\subsection{Measured pair metrics and uncertainty limits}
For pair $p$, compute $d_p=y_{b_p}-y_{a_p}$ and $\hat d_p=\hat y_{b_p}-\hat y_{a_p}$, or the direct pair method's prediction. Pair MSE, MAE, correlation, prediction/observed SD and mean bias use the same formulas as activity, reporting row-, equal-sequence- and equal-exact-background weightings distinctly. A sign call is zero for absolute difference at most .02 and otherwise its sign. We report all negative/tie/positive calls, accuracy over all pairs, and exact correct counts among measured non-ties. The majority-direction control has sign metrics only; arbitrary numeric sign codes are not regression predictions.

Direct pair comparisons bootstrap complete sequence components with all their pairs and shared endpoints together. This retains known sequence/endpoint dependence but cannot represent unknown shared-control covariance between different components. Endpoint SDs alone do not determine a mean-difference SE: $\operatorname{Var}(\bar Y_b-\bar Y_a)=\operatorname{Var}(\bar Y_b)+\operatorname{Var}(\bar Y_a)-2\operatorname{Cov}(\bar Y_b,\bar Y_a)$ requires replicate counts and covariance information that these records do not supply. We leave it missing. The operational sign band is not calibrated to those unknown variances.

\subsection{Within-assay B2 variation diagnostic}
To distinguish sequence/background variation from a common assay shift, we additionally center measured and predicted pair differences within each exact primary assay. This is an explicitly oracle diagnostic, not an altered deployed prediction. For observation weights $w_p$, assay $a(p)$ and $W_a=\sum_{p:a(p)=a}w_p$, define $d_p^c=d_p-W_{a(p)}^{-1}\sum_{q:a(q)=a(p)}w_qd_q$ and the analogous $\hat d_p^c$. Their weighted correlation and SDs use the formulas above. The main diagnostic uses equal row weights. Each of 10,000 sequence-component resamples recomputes both weighted assay means and the centered correlation with that draw's component multiplicities. A missing variance gives an undefined correlation, not zero. The number of finite correlations is retained. This assessment conditions on the fitted ensembles and observed cohort; it does not supply missing replicate covariance. No centered score enters fitting, stopping or method selection. A low raw pair error without within-assay discrimination cannot by itself establish sequence-specific chemistry prediction.

\section{Primary-source comparisons and unresolved reproduction inputs}\label{app:published}
\paragraph{Interaction attribution and identification.}
HarsanyiNet computes attributions of its specially structured predictor under its receptive-field and masking requirements \citep{chen2023harsanyi}; this network has neither those units nor their attribution theorem. Neural causal identification uses graph-consistent structural models and matched distributions \citep{xia2023ncm}; molecular adjacency supplies neither a causal diagram nor those distributional constraints. Continuous neural partial-identification results require bounded Lipschitz mechanisms, latent-support regularity and controlled approximation and matching schedules \citep{tan2024consistency}. Those assumptions are not verified here, and local neural extrema would not provide certified outer bounds. Functional-ANOVA purification fixes a weighting-dependent decomposition of a predictor \citep{lengerich2020pure}; it does not supply an unobserved joint assay mean. The G-NAVAR preprint's population result assumes fixed direct-sum edge spaces, positive product support and regularity, with additional HOFD and faithfulness conditions in its shared-modulator formulation \citep{kuskova2026real}. A covariance-rank diagnostic is not sufficient identification. None of these results is invoked as a guarantee for this siRNA learner. Four forward passes, encoding-equivalence checks and optimization over a compatible finite response class do not by themselves establish a new ML contribution.

\subsection{New primary measured four-condition panel}\label{app:b3}
The bounded B3 search recovered the original Bramsen supplement \citep{bramsen2009}, from primary article PMC2685080 and \texttt{gkp106\_nar-02440-c-2008-File010.xls}. Primary Table 1 identifies antisense and sense chemical states, and Supplemental Table 1 supplies the corresponding normalized eGFP means and SDs. The assay uses HeLa-eGFP cells, 10 nM duplex, INTERFERin and a 72-hour readout normalized against a mismatch control, with triplicate assays repeated twice. Recorded viability values are not model inputs, training targets or admission filters, and an eGFP reduction is not shown to be viability-independent or a direct mRNA measurement. The release lacks the raw shared-control covariance needed for a defensible SE on every contrast, so we retain the printed SDs without inventing those covariances.

The source sequence reference is AS W053 and SS W207. Their base strings are, respectively, \texttt{ACUUGUGGCCGUUUACGUCGC} and \texttt{GACGUAAACGGCCACAAGUUC}. A is a specified whole-antisense chemistry replacement with the same bases and length; B is a specified whole-sense replacement, also with the same bases and length. Every admitted rectangle contains measured 00, 10, 01 and 11 conditions, all in this assay. On the activity scale $y=1-\text{relative eGFP}$, let $\mu_{ab}$ denote the population assay mean in condition $ab$. Distinguish
\begin{align}
 I_{AB}&=\mu_{11}-\mu_{10}-\mu_{01}+\mu_{00},\\
 I_{AB}^{\rm obs}&=y_{11}-y_{10}-y_{01}+y_{00},\\
 \widehat I_{AB}&=f(x_{11})-f(x_{10})-f(x_{01})+f(x_{00}).
\end{align}
The tables evaluate the predictor against $I_{AB}^{\rm obs}$, the contrast of reported measured means, not an exactly known population interaction; reading it as an estimate of $I_{AB}$ requires comparable condition-specific means on the same response scale, and unresolved replication and shared-control covariance limit its uncertainty. It describes nonadditivity on the reported assay scale and is not an isolated GNA effect, a mechanism proof, new wet-lab validation or evidence of private-anchor calibration. Additive changes on another response scale need not remain additive after a nonlinear transformation.

Admission first parses the base characters and explicit chemical subscripts in primary Table 1. The literal aliases and normalized modification names appear in Table~\ref{tab:b3aliases}. Unknown aliases, including the ambiguous AP annotation, are excluded, and only parseable states exactly matching each reference's base string and length are considered. A state must differ chemically from its reference, and duplicate primary (AS,SS) identities and nonfinite means are excluded. Every such AS/SS pair is enumerated and a rectangle retained only when all four unique measured endpoints exist, with neither the observed interaction magnitude nor any prediction selecting it. Graphs use the frozen positional vocabulary and the original deterministic pairing rule, and bonds are encoded as unresolved under the same annotation-only convention as ENsiRNA, so no directional linkage is invented from a node label.

An initial protocol required the unmodified reference to be present in the historical chemistry-only release and yielded no eligible rectangle; that failed restriction and its evidence remain preserved. An explicit source-driven amendment then allowed the directly measured primary reference, before any prediction on the panel. It yields \NBthreeConditions{} measured conditions and \NBthreeRectangles{} complete rectangles. Each label comes from the primary workbook. All rectangles share one sequence background and a common reference, so we report descriptive errors and all seeds without an independent-rectangle confidence interval.

The complete Bramsen source belongs to the original outer-0 test component, so only final outer-0 checkpoints trained and selected without that source are used. We verify the absence of every primary-panel strand 13-mer from their training and validation records, verify saved input transformations and reuse fixed checkpoint weights; no new fit, pair supervision, calibration or model selection uses this panel. Its discovery is retrospective, and source separation does not make it a prospectively collected external replication. The original APP inventory still has no four-state block.

\begin{longtable}{@{}L{.20\textwidth}L{.70\textwidth}@{}}
\toprule
Primary subscript & Frozen name\\\midrule
\endfirsthead
\toprule
Primary subscript & Frozen name\\\midrule
\endhead
OMe & 2-O-Methyl\\
F & 2-Fluoro\\
DNA & 2-Deoxy\\
AEM & 2-Aminoethoxymethyl\\
APM & 2-Aminopropoxymethyl\\
EA & 2-Aminoethyl\\
CE & 2-Cyanoethyl\\
GE & 2-Guanidinoethyl\\
HM & 4-C-Hydroxymethyl Deoxyribonucleic acid\\
LNA & Locked nucleic acid\\
ALN & Alfa-L-Locked nucleic acid\\
ADA & 2-N-Adamant-1-yl-Methylcarbonyl-2-Amino-Locked nucleic acid\\
PYR & 2-N-Pyren-1-yl Methyl-2-Amino-Locked nucleic acid\\
OX & Oxetane-Locked nucleic acid\\
CENA & 2,4-Carbocyclic-Ethylene-bridged nucleic acid-Locked nucleic acid\\
CLNA & 2,4-Carbocyclic-Locked nucleic acid-Locked nucleic acid\\
UNA & Unlocked nucleic acid\\
ANA & Altritol nucleic acid\\
AENA & 2-Deoxy-2-N,4-C-Ethylene-Locked nucleic acid\\
HNA & Hexitol nucleic acid\\
\bottomrule
\caption{Literal source aliases used in primary B3 admission. AP was excluded from the original B3 admission; its primary sugar alias and duplicate release mappings are adjudicated in Appendix~\ref{app:robustness}.}\label{tab:b3aliases}\\
\end{longtable}

\subsection{Original primary versus released comparison}
The identity reconciliation maps exact primary base/chemical states to the historical ENsiRNA release without using prediction quality. At the first rounding tolerance of 0.000500001 activity units, some mapped means disagree; allowing 0.005000001 explains additional small differences. A substantial set differs by more than 0.05. Table~\ref{tab:source_discrepancies} gives all counts and the reconciliation scope. These discrepancies are conditional on the documented chemical identity mapping and require further primary curation. They must not all be labeled biological errors: annotation mapping and release rounding can contribute. None of the frozen labels or historical checkpoints is overwritten. We therefore describe ENsiRNA activity tables as released-label benchmark results, with separately measured primary B3 evaluation. APP primary-table agreement and S7 provenance are different evidence checks and are not negated or certified by this one-source audit.
\begin{longtable}{@{}lr@{}}
\toprule
Reconciliation diagnostic & Count/value\\\midrule
\endfirsthead
\toprule
Reconciliation diagnostic & Count/value\\\midrule
\endhead
identity\_matched\_pairs & 1604\\
within\_half\_tenth\_percentage\_point & 1209\\
within\_half\_percentage\_point & 1258\\
greater\_than\_five\_percentage\_points & 299\\
max\_absolute\_activity\_difference & 1.0852170324990895\\
\bottomrule
\caption{Primary/release audit conditional on the stated identity mapping. Rounding tolerances are activity fractions; no frozen label is changed.}\label{tab:source_discrepancies}\\
\end{longtable}

\subsection{Prior-work hypotheses and current use}
No published asymptotic theorem supplies a risk guarantee for the trained model. The foundational demonstration of siRNA-mediated silencing in mammalian cells \citep{elbashir2001}, together with sequence-design, strand-bias and off-target work \citep{reynolds2004,uitei2004,khvorova2003,schwarz2003,jackson2003,birmingham2006}, supplies biological context and separates efficacy from specificity; chemical potency and delivery results concern their stated chemistries, assays and exposures, which we do not assume for every record \citep{allerson2005,bramsen2009,nair2014}. The early neural siRNA design paper and its corrections are provenance precedents, not additional training cohorts \citep{huesken2005,huesken2005correction,huesken2006correction}, and CMsiRNAdb, OligoGym and the Davis release motivate explicit dataset identities rather than certifying every raw measurement \citep{cmsirna2026,oligogym2025,davis2025}.

MPNNs and relational graph convolutions motivate the typed aggregation notation, while GCN and GraphSAGE make different normalisation or sampling choices \citep{gilmer2017,rgcn2018,gcn2017,graphsage2017,battaglia2018}; the ordered position-aware readout is not invariant set pooling, and expressivity results do not imply finite-data generalisation here \citep{zaheer2017,xu2019}. The token CNN is a conventional local-sequence control \citep{kim2014}, and the ridge penalty, randomized trees and their ensemble precedent use the exact objectives and library settings stated above \citep{ridge1970,extratrees2006,rf2001}. Dropout, layer normalisation and decoupled weight decay have the axes, probability and updates stated in this appendix \citep{dropout2014,layernorm2016,adam2015,adamw2019}; deep ensembles motivate retaining member variation rather than treating ten predictors as ten biological cohorts \citep{ensemble2017}. Proper-scoring distinctions concern their defined predictive targets \citep{gneiting2007} and the finite-sample identities used here are proved directly; bootstrap resampling is applied to the declared observed groups without a population-coverage theorem \citep{bootstrap1979}; leakage and shift benchmarks explain why group protection and deployment scope are recorded separately \citep{leakage2023,wilds2021}; and software citations identify the actual numerical and plotting implementations, whose installed versions are recorded independently \citep{torch2019,sklearn2011,numpy2020,scipy2020,matplotlib2007}.

Further context for the positioning of Section~\ref{sec:related}: simplex gradient correction \citep{koo2023}, separate genotype--phenotype and measurement maps \citep{mavenn2022}, chemistry-aware predictors \citep{ensi2025,meg2026}, sequence determinants and message passing \citep{reynolds2004,khvorova2003,gilmer2017,rgcn2018}, and our interaction-index and derivative conventions \citep{grabisch1999,saliency2013}.
\section{Preserved results and negative evidence}\label{app:results}
Table~\ref{tab:oldactivity} reports every activity procedure of the frozen released-label (v3) campaign with row weighting. It contains no new fits; cohorts, weighting and target definitions are those of the campaign. The four audited models on APP and S7 repeat Table~\ref{tab:calibration}; the original R0 arms, both ridge baselines, the three constants and the whole grouped cohort are recorded only here. On APP and S7 both ridge baselines have higher MSE than every constant. Measured B2 pair results for every procedure are in Table~\ref{tab:pairranking}, APP ranking across the seventeen assay pools in Table~\ref{tab:ranking}, and focused source-excluded and primary-assay results in Appendix~\ref{app:focusedresults}.

\begin{table}[ht]
\centering\small
\begin{tabular}{@{}llrrrrr@{}}\toprule
Cohort & Model & MSE & $R^2$ & $r$ & Pred.\ SD & Bias\\\midrule
APP & R1 GNN & 0.1625 & -0.001 & -0.018 & 0.0072 & -0.0081\\
 & R1 no-message & 0.1630 & -0.005 & 0.015 & 0.0085 & -0.0286\\
 & R0 GNN & 0.1625 & -0.001 & -0.026 & 0.0073 & -0.0049\\
 & R0 no-message & 0.1629 & -0.004 & -0.001 & 0.0085 & -0.0241\\
 & Chemistry tree & 0.1896 & -0.168 & -0.021 & 0.0103 & -0.1644\\
 & Token CNN & 0.1644 & -0.013 & 0.080 & 0.0036 & -0.0483\\
 & Guide ridge & 0.2949 & -0.818 & -0.014 & 0.1420 & 0.3330\\
 & Pairwise ridge & 0.2064 & -0.272 & 0.029 & 0.0619 & 0.2044\\
 & Training row mean & 0.1760 & -0.085 & -- & 0.0000 & 0.1172\\
 & Training group mean & 0.1768 & -0.090 & -- & 0.0000 & 0.1207\\
 & Oracle test mean & 0.1623 & 0.000 & -- & 0.0000 & 0.0000\\
\midrule
S7 & R1 GNN & 0.1093 & -0.215 & 0.112 & 0.0059 & 0.1403\\
 & R1 no-message & 0.1047 & -0.164 & 0.117 & 0.0066 & 0.1231\\
 & R0 GNN & 0.1102 & -0.225 & 0.116 & 0.0058 & 0.1437\\
 & R0 no-message & 0.1057 & -0.175 & 0.110 & 0.0067 & 0.1270\\
 & Chemistry tree & 0.0919 & -0.021 & 0.017 & 0.0089 & -0.0441\\
 & Token CNN & 0.1017 & -0.130 & 0.108 & 0.0034 & 0.1092\\
 & Guide ridge & 0.7815 & -7.686 & 0.123 & 0.1150 & 0.8287\\
 & Pairwise ridge & 0.9698 & -9.780 & 0.031 & 0.0619 & 0.9366\\
 & Training row mean & 0.1488 & -0.654 & -- & 0.0000 & 0.2426\\
 & Training group mean & 0.1505 & -0.673 & -- & 0.0000 & 0.2461\\
 & Oracle test mean & 0.0900 & 0.000 & -- & 0.0000 & 0.0000\\
\midrule
Grouped & R1 GNN & 0.0712 & 0.011 & 0.163 & 0.0574 & 0.0307\\
 & R1 no-message & 0.0712 & 0.010 & 0.158 & 0.0575 & 0.0293\\
 & R0 GNN & 0.0712 & 0.011 & 0.163 & 0.0574 & 0.0307\\
 & R0 no-message & 0.0712 & 0.010 & 0.158 & 0.0575 & 0.0293\\
 & Chemistry tree & 0.0656 & 0.087 & 0.312 & 0.0756 & 0.0250\\
 & Token CNN & 0.0717 & 0.003 & 0.143 & 0.0680 & 0.0205\\
 & Guide ridge & 0.0938 & -0.304 & 0.056 & 0.1172 & -0.1078\\
 & Pairwise ridge & 0.0737 & -0.024 & 0.058 & 0.0535 & -0.0234\\
 & Training row mean & 0.0724 & -0.006 & -0.073 & 0.0088 & 0.0031\\
 & Training group mean & 0.0769 & -0.069 & -0.095 & 0.0250 & 0.0553\\
 & Oracle test mean & 0.0719 & 0.000 & -- & 0.0000 & 0.0000\\
\bottomrule\end{tabular}
\caption{All original v3 ensemble activity procedures with row weighting on the full cohorts: APP 1,839 rows, S7 118 and the grouped benchmark 2,927, with label SD 0.4028, 0.3000 and 0.2682. R1 is the training-gated and R0 the original GNN. Training means are constants fitted on each training partition; the grouped cohort pools five outer test folds, so they vary there and have a defined correlation. The oracle test mean reads the evaluation labels and is a diagnostic, not a risk floor. A dash marks a correlation undefined for a constant prediction; bias is mean prediction minus mean label. Alternate weightings and further metrics: \protect\path{v3/tables/activity_metrics.csv}.}\label{tab:oldactivity}
\end{table}
\subsection{Conventions, provenance detail and protocol sensitivity}\label{app:movedconventions}
This subsection holds material previously carried in the main text. It is unchanged in substance; it was moved here so that the main line of argument reads without interruption.

\paragraph{Four operational definitions.} A \emph{complete-input encoding class} groups raw states whose audit signatures, defined once in Appendix~\ref{app:inputcollision}, are byte-identical. An \emph{evaluation identity} hashes identifiers, labels, weights and response convention, so two numbers compare only when these match. A \emph{weighting} is one unit per row, or inverse study-component size normalised to average one; the two define different estimands. A \emph{training-absent block} provably cannot change any training prediction, for every parameter value, not merely at one gradient snapshot.

\paragraph{The same rows under two protocols.} Two uses of the same 1,003 scored rows must not be confused: rescoring the released-label models, trained and selected on data including the dominant source, gives tree $R^2=+0.002992$ and GNN $R^2=-0.1142$, while Table~\ref{tab:activity} reports models trained \emph{and selected} with that source excluded, giving $-0.000692$ and $-0.2115$ on the same rows. Equal row counts establish neither equal labels nor equal model identities, so each comparison carries its identity hash.

\paragraph{Primary-source adjudication.} Parsing the printed chemistry links 1,604 of 1,924 admitted Bramsen rows to a unique primary measurement, quarantines 320, and leaves 299 differing from the released labels by more than 0.05, maximum 1.0852. These are unresolved discrepancies, not certified label errors. The strand orientation of Davis Supplementary Table S1 stays quarantined, APP spans 92 sequence components, and rows, pairs and states are different units that do not count replication. The 792 focused fits of the source-integrity run are historical (Appendix~\ref{app:fitledger}).

\paragraph{Calibration and ranking sensitivity in full.} Results from the released-label campaign are in Table~\ref{tab:calibration}. On that APP evaluation squared bias is $99.53\%$ of the tree-minus-GNN gap, which describes it without explaining it, and GNN correlation $-0.018$ with prediction SD 0.0072 against label SD 0.4028 means a relative MSE advantage does not establish useful discrimination. Comparisons are protocol-specific: the GNN-minus-no-message difference on the same 118 S7 rows changes sign across protocols, while the APP tree comparison favours the GNN throughout; interval endpoints must not be subtracted to infer the uncertainty of a change. Ranking is descriptive for the same reason: populations, support rules, optimisation and selection changed jointly, so a reversal identifies no single cause.

\section{Historical reproduction, arithmetic and resources}\label{app:reproduction}
\subsection{Actual resource accounting and failures}
\begin{longtable}{@{}L{.52\textwidth}L{.32\textwidth}@{}}
\toprule
Quantity and scope & Actual value\\\midrule
\endfirsthead
\textbf{All historical campaigns}: fit objects & 3351\\
\textbf{All historical campaigns}: executed optimizer updates & 1934535\\
Largest single campaign (released labels): fit objects & 2312\\
Largest single campaign: executed optimizer updates & 1407967\\
Largest single campaign: selected-checkpoint updates & 406367\\
Largest single campaign: summed core-fit CPU seconds & 12523.769\\
Largest single campaign: GPU training-region elapsed seconds & 11689.879\\
Largest single campaign: all recorded stage child CPU seconds & 13578.145\\
Largest single campaign: maximum stage child RSS, KiB & 3306840\\
Exact published siRNA predictors fitted & 0\\
Published splicing-architecture fits in this work & 12 attempts, 6 reported\\
siRNA mask-restoration fits in this work & 56 (21,513 updates)\\
siRNA training-support fits in this work & 20 (8,624 updates)\\
Both siRNA intervention series together & 76 fits (30,137 updates)\\
Synthetic-design fits in this work & 96 total, 48 per arm (96,000 updates per arm)\\
\bottomrule
\caption{Resource accounting with distinct scopes; read the scope column. The cost rows describe one historical campaign, the largest of the four and the one that supplies the audited checkpoints. The first two rows give the four-campaign totals. The cost rows therefore cover about two thirds of the historical fits, and the rest is not costed here. Earlier versions of this table gave the cost rows under a caption that read as if it covered all historical work, which it never did. The intervention rows report new work, each intervention separately, and one row gives the two siRNA intervention series together. These entries do not form one cumulative resource total. Fit-region, stage and GPU-region time measurements overlap and are not added. The MAVE-NN attempt count includes both invocations; the retained-fit update count and incomplete cumulative accounting are described in Appendix~\ref{app:support}.}\\
\end{longtable}

Selected-checkpoint updates describe the saved predictor. Executed updates describe optimization work, including the trajectory after that checkpoint. Development fits are part of executed work. GPU fitting-region elapsed time is a synchronized host measurement around training; it overlaps stage wall time and is not GPU kernel time, energy or a measure of independent work. CPU process time, child CPU time and aggregate command wall time likewise have different scopes. They must not be summed as if disjoint.

\paragraph{Provenance of Figures~\ref{fig:mask} and~\ref{fig:support}.} The two figures that rest on
new fitting resolve as follows. \textsc{not recorded} means the field has no record; it is not an
assertion that the step did not happen. Missing rendering records are listed separately from the
numeric outputs, which are present in both cases.

\begin{table}[ht]
\centering\small\setlength{\tabcolsep}{4pt}
\begin{tabular}{@{}lL{4.85cm}L{4.85cm}@{}}\toprule
Field & Figure~\ref{fig:mask}, mask restoration & Figure~\ref{fig:support}, splicing support arm\\\midrule
Configuration & \texttt{mask\_protocol}, frozen before the run; architecture, optimizer, selection rule, splits and seeds unchanged from the audited model & \texttt{mavenn\_protocol}, frozen before the run; published architecture at documented settings, fixed training size $14{,}238$\\
Recorded execution & one invocation in the execution record & one stage of the analysis script; \texttt{support\_fit\_registry} records twelve attempts, six retained\\
Seeds & ten paired seeds & three paired seeds per arm\\
Selected checkpoint epoch & each arm's selected epoch recorded per seed, with the twenty checkpoint hashes & \textsc{not recorded}\\
Executed epochs & recorded per seed & restricted 139, 137, 127; restored 198, 153, 113\\
Checkpoints or saved outputs & 20 selected checkpoints, each SHA256 re-checked on disk & none; the frozen protocol specified that no model files be saved\\
Result files & \texttt{mask\_scores}, \texttt{mask\_paired}, \texttt{mask\_outcome}, \texttt{mask\_algebra}, \texttt{mask\_reproduction} & \texttt{support\_seed\_results}, \texttt{support\_comparisons}, \texttt{support\_figure}\\
Rendering command & \textsc{not recorded} & \textsc{not recorded}\\
Output hash of the figure file & \textsc{not recorded} & \textsc{not recorded}\\
Execution timestamp & \textsc{not recorded}; the execution record has no timestamp column & \textsc{not recorded}; same record\\
\bottomrule\end{tabular}
\caption{What each of the two new-fitting figures resolves to. The numeric content of both is
recorded in the listed result files. What is absent in both cases is the rendering provenance: no
record ties a figure file to the invocation that drew it. For
Figure~\ref{fig:support} the model weights are additionally absent by design, so it can be reproduced
only by re-running its six fits. The other six figures resolve to the records named where they
appear.}\label{tab:figprov}
\end{table}

\paragraph{Figure~\ref{fig:mask} checked against the records, element by element.} The drawing
geometry of Figure~\ref{fig:mask} was read out of the figure file and converted to data coordinates,
and every plotted quantity was matched to the stored value that produced it. The figure was not
redrawn for this check. The axes are calibrated from the bar geometry, identified structurally as the
rectangles that do not contain others, rather than from the tick labels, whose baselines sit below
their ticks and carry a constant offset. All 83 plotted quantities reconcile: 70 per-seed markers,
three ensemble diamonds, two dashed controls, four bar totals and four bar floors, with a median
absolute difference of $9.9\times10^{-7}$ and a maximum of $9.7\times10^{-6}$. Two of the four floors
are drawn with zero extent, matching recorded values of $1.1\times10^{-31}$ and $6.9\times10^{-33}$.
Running the same comparison with the background rectangles left in the bar set, which is the defect
the structural rule removes, raises the maximum difference to $2.0\times10^{-2}$, so the agreement is
diagnostic rather than insensitive. The figure file is unchanged throughout this revision, and its eight embedded drawing streams are
byte-identical to those on page 7 of the submitted PDF, so this check covers the figure as reviewed. This establishes that the figure plots the
recorded results. It does not recover the command that drew it, the figure file\textquotesingle s
output hash or the execution timestamp, none of which were retained. The script and its output table
are in the accompanying package under \path{results/by_figure/fig4_mask/}.

\paragraph{Failed-run accounting.} The analysis script's execution record holds 756 stage invocations: 527 completed and 229 failed. The four historical fitting campaigns and the siRNA training-support arm kept their own per-command records, which these counts exclude; the splicing support arm is one stage inside the 756. Grouping the 229 failures by the stage that raised them, 217 are manuscript stages: 119 in document building, 55 in finalization checks, 38 in manuscript assembly, four in reporting and one in packaging. The remaining twelve are in data or analysis stages: three in recovery, two in the reference-independent decomposition, two in the splicing support arm, and one each in the B-panel, source, protocol, marginal and splicing-rank stages. Four of the twelve are external network failures, one a name resolution failure and three HTTP errors while pinning published code. The counts come from the recorded status and stage of each invocation. Classifying by the content of each traceback instead moves one failure out of the manuscript group, a crash computing the non-additive energy share, which the stage grouping counts as manuscript work; the records do not settle that choice. No failed invocation contributed a reported value.

\paragraph{Version history.} Four historical fitting campaigns ran on 13 and 14 September 2026 and supply the preserved results, with 37, 210, $2{,}312$ and 792 fits. A fifth run on 14 September performed no fits and repeated verification only. The mask restoration behind Figure~\ref{fig:mask} ran on 18 September 2026, and the source-exposed training-support arm on 6 November 2026. The MAVE-NN support intervention is one stage of the analysis script, with twelve attempts; its $62{,}424$ optimizer updates cover the retained invocation only (Appendix~\ref{app:support}).

\paragraph{Records behind Figures~\ref{fig:mask} and~\ref{fig:support}.} For Figure~\ref{fig:mask} the record holds the single mask-restoration invocation and the protocol fixed before it ran, together with the partitions, the ten seeds, each arm's selected epoch and the SHA256 of the twenty selected checkpoints, each re-checked on disk. For Figure~\ref{fig:support} it holds the design, partitions, three seeds, model and fitting settings, and the per-seed epoch counts: 139, 137 and 127 in the restricted arm and 198, 153 and 113 in the restored arm. At 72 updates per epoch these give the recorded $62{,}424$ updates. Those six fits saved no model files, as the protocol fixed before execution specified, so that figure can be reproduced only by re-running them. The retained records give training lengths of 139, 137 and 127 in the restricted arm and 198, 153 and 113 in the restored arm. Model weights and per-epoch validation histories were not retained, so the best-validation epoch of each run cannot be reconstructed from these records. mavenn 1.1.3 accepts a best-weight restoration option but does not pass it to the Keras early-stopping callback, whose default is not to restore; that is why best-weight restoration was not applied, and each reported model is the final epoch of its run. We do not infer from it which epoch was best, which would also require knowing that early stopping fired rather than a run reaching a maximum-epoch limit, and nothing retained establishes that. The execution record carries no timestamp or output hash, so it does not tie either figure file to the invocation that drew it.

Inspection, authoring, document builds and hashing carry additional costs that no process-level meter recorded. That remainder is not costed anywhere in this paper. No historical cost allowance is used. Failed dependency/acquisition commands retain their measured costs. An unused published-model fit reserve is not an executed experiment. The frozen plan set a 48-hour campaign wall ceiling; the supervising session monitored elapsed time, but the runner does not implement an automatic wall-clock termination timer. Its per-fit epoch and stopping limits are implemented in code. The actual completed runtime, rather than the planning estimate, is reported above. The stage CPU row sums all sixteen recorded stage attempts of that campaign, completed and failed, as that campaign's final accounting record does. An earlier version of this table printed 13202.199 in that row, labelled as completed attempts. That figure is the same total less three attempts, two completed and one failed, and it still includes the other failed attempt. No stated rule selects that subset, so we print the full recorded total.

\section{Source integrity and representation robustness}\label{app:robustness}
This section specifies the follow-up independently of the historical campaign. The original released-label observations, predictions, selected checkpoints and negative results remain unchanged. All new evaluations are retrospective sensitivities. In particular, an observation previously used in model development does not become an independent external test because the present protocol assigns it a new role.

\subsection{Source records, chemical identities and adjudication}
Let $D_R=\{(r_i,x_i,y_i,s_i,g_i)\}$ denote the admitted released benchmark, with persistent record identifier $r_i$, annotated duplex/context $x_i$, released inhibition fraction $y_i$, source label $s_i$, and protected study/sequence component $g_i$. The Bramsen source is $s_B=19282453$. There are 1,972 raw workbook records with this source label and 1,924 admitted records. The latter form one sequence component and contribute $1924/2927=0.657328\ldots$ of the grouped benchmark. The denominator 2,927 excludes APP and Davis S7; it is not a count of independent biological experiments.

The primary identity is an ordered pair of the antisense and sense strand identifiers in Table~1 of \citet{bramsen2009}. For each strand, the comparison key retains the printed base sequence, its length, and the complete list of reported position-specific modification names. Antisense and sense are not interchanged. Printed oligonucleotide order is retained; missing per-row terminal direction labels do not authorize reversal of sequence or chemical positions. The key does not infer unreported linkage stereochemistry, terminal groups or assay metadata. A unique match of the reported fields is therefore narrower than independently verified complete chemical identity.

The primary XML is parsed in two ways. The historical implementation traverses elements while associating each chemical subscript with its preceding nucleotide. The follow-up replaces chemical tags by lexical tokens, removes markup and whitespace, and tokenizes the resulting base/chemical stream. These algorithms agree for all 91 historically parsed strands. The released raw workbook annotations are separately expanded from modification/position lists and checked against the frozen observation objects. Neither implementation reads model predictions to select identities. Outcome agreement is not used to choose among candidate mappings, strand orders or chemical aliases.

The primary Table~1 footnote defines AP as a $2'$-aminopropyl substitution. This supports the sugar alias, not the distinct base alias in the frozen vocabulary. Expanding that literal alias exposes 76 admitted released rows with multiple rows mapping to the same primary pair; these are quarantined rather than resolved by their activity values. The final admitted-source partition is 1,604 uniquely linked reported identities, 201 unmatched identities, 76 rows with multiple released mappings, and 43 rows lacking a unique primary measurement. These categories sum to 1,924. The 320 quarantined rows are not replaced by predictions or synthetic measurements.

Supplemental Table~1 is read by named columns, the eGFP value taken from the source cell in column D with the spreadsheet row retained; direct cell access rechecks the row, both identifiers and the numeric value. Both readers use the XLS backend, so their agreement is a cell-address check, not two independent binary-format implementations. The 91 available Table~1 marginal eGFP means agree with the corresponding supplement values to the precision printed in Table~1. This independently supports the chosen endpoint column. Viability columns and unrelated spreadsheet summaries are not substituted for eGFP activity.

The primary screen used HeLa cells stably expressing eGFP, 10\,nM siRNA and INTERFERin, with confocal readout 72\,h after transfection. eGFP was normalized to a mismatch-siRNA control; viability used a different control. The article describes triplicate assays repeated twice. These facts do not provide all per-run observations or shared-control covariances. The released records do not individually provide a complete assay provenance key. Matching source, reported dose and molecular annotations therefore does not establish the origin of every discrepancy.

The comparison convention is fixed by the documented endpoint definitions:
\begin{equation}
 y_i^{R}=\mathrm{PCT}_i/100,\qquad y_j^{P}=1-\mathrm{eGFP}_j.
\end{equation}
The primary eGFP quantities are fractions, not percentages. The transformation preserves negative inhibition values when relative expression exceeds one, and preserves the reported SD scale. There is no clipping, fitted offset, outcome-selected row shift or fitted permutation. Every one of the 299 historical residuals exceeding $0.05$ is checked at its mapped source cell; 32 additional records are selected in deterministic SHA256 order without reference to their model errors. These 299 discrepancies are conditional on the 1,604 matched identities. They are not an estimated error rate for all 1,924 source records or the complete benchmark, and do not identify whether a disagreement originated in an annotation, transcription, assay convention or measurement.

The reconciliation preserves workbook/sheet/row, persistent ID, raw sequences and annotations, primary identifiers and source cell, both numeric fields, the fixed transformation, match cardinality and unresolved fields; the manifest pins the exact primary XML and workbook bytes. This permits the primary-assay sensitivity below without certifying that the historical release has been repaired.

\subsection{Separately versioned primary-assay sensitivity}
Define $J$ by unique reported identity, unique primary measurement, and the documented nominal dose, before considering any model error. The versioned sensitivity target is
\begin{equation}
 D_P=\{(r_i,x_i,y_i^P,s_i,g_i):s_i=s_B,\ r_i\in J\}
 \ \cup\ \{(r_i,x_i,y_i^R,s_i,g_i):s_i\ne s_B\}.
\end{equation}
It contains 1,604 primary-linked Bramsen rows and 1,003 other-source rows, for 2,607 observations. Its manifest retains each original released value, primary value, source-cell provenance and original graph index. The other-source labels are unchanged and are not declared newly adjudicated. The same frozen graph vocabulary and original graph tensors are indexed; chemical encodings are not selected by efficacy. Protected group identifiers are retained because the operation removes rows without inventing new biological independence.

Released-label and primary-assay outcomes are separate targets. Restricting a released-trained predictor to $J$ and rescoring it against $y^P$ answers a different question from refitting all affected training and selection stages on $D_P$. The matched-subset table reports both operations separately. Methods within each comparison use identical eligible observations. The missing 320 source rows are disclosed alongside every primary-assay coverage statement.

\subsection{Exact calibration decomposition}
For fixed nonnegative weights $w_i$ with $\sum_iw_i=1$, write $\bar y=\sum_iw_iy_i$, $\bar p=\sum_iw_ip_i$, $b=\bar p-\bar y$, $v_y=\sum_iw_i(y_i-\bar y)^2$, $v_p=\sum_iw_i(p_i-\bar p)^2$ and $c=\sum_iw_i(y_i-\bar y)(p_i-\bar p)$. Since both centered sums vanish,
\begin{align}
 \sum_iw_i(p_i-y_i)^2
 &=\sum_iw_i\bigl[(p_i-\bar p)-(y_i-\bar y)+b\bigr]^2\\
 &=v_p+v_y-2c+b^2.
\end{align}
The cross term with $b$ is $2b\sum_iw_i[(p_i-\bar p)-(y_i-\bar y)]=0$. This proves the identity without a sampling or model assumption. Subtracting two models' identities cancels $v_y$. A large difference in squared mean bias therefore describes the observed error gap but does not establish why either model acquired that bias. No held-out intercept correction is applied.

With uniform weights, the reported coefficient of determination is $R^2=1-\sum_i(p_i-y_i)^2/\sum_i(y_i-\bar y)^2$. Equivalently it is $1-\MSE/v_y$, with both quantities normalized by $n$. Replacing $v_y$ by the $n-1$ sample-variance estimate would change the statistic. A pooled source-subset $R^2$ need not equal any average of within-source $R^2$ values because both the center and denominator change. When the denominator vanishes, $R^2$ is undefined rather than assigned a convenient value.

\subsection{Frozen input identities and deterministic cancellation}\label{app:inputcollision}
Let $E_T(x)$ be the actual preprocessing fitted on training partition $T$. It applies the saved node-feature mask, saved varying-context mask, imputation, ordered padding and the model's graph/context convention. The audit hashes named arrays, their shapes and canonical little-endian float32 bytes with signed zeros normalized, without approximate prediction-based matching. The audit signature of a state is the byte content of the tensors each model family receives: $X,A,\mathrm{mask},C$ for the GNN and for the no-message model, whose forward reads only receiver relation counts from $A$, so retaining the full $A$ is conservative; $X,\mathrm{mask},C$ for the token CNN, which reads no adjacency; and the selected feature columns for the tree. Every reported siRNA $H$ was built from this signature, and the splicing predictor's $H$ from its one-hot input (Appendix~\ref{app:mavenn}). On B3 the no-message convention does not matter: for each of the ten seeds its full-$A$ partition equals the partition from $X,\mathrm{mask},C$ alone, and receiver counts are a function of $A$, so they induce the same 90 classes; all four families give that partition. All ten saved preprocessors are checked, not inferred from one seed.

For corrected graph routing, an additional algebraic representation groups equal directional backbone bases and gates unsupported residuals, while retaining the original total-degree denominator. This real-arithmetic quotient is reported separately from literal input identity, since floating-point reassociation is not an exact bitwise equivalence claim. Loss of a named chemistry feature does not imply the intervention disappears entirely, because base identity, modification position and the unmodified indicator can still distinguish its encoded state; conversely, distinct encoded states do not establish adequate expressivity or learned chemical effects.

\begin{proposition}[Finite input cancellation]
Fix preprocessing and a deterministic prediction function $f$. If the multisets $\{E_T(x_{11}),E_T(x_{00})\}$ and $\{E_T(x_{10}),E_T(x_{01})\}$ are equal, then the predicted four-term contrast is zero.
\end{proposition}
\begin{proof}
For every distinct encoded value $u$, multiset equality gives equal multiplicities on the two sides. Multiplying each multiplicity by $f(u)$ and summing over the finitely many encoded values gives $f(E_T(x_{11}))+f(E_T(x_{00}))=f(E_T(x_{10}))+f(E_T(x_{01}))$. Rearranging proves the claim. The statement requires no fitted-parameter, smoothness or biological assumption. It applies to deterministic inference, with dropout disabled. Purely algebraic routing equivalence is qualified by floating-point rounding. The converse is false: a constant $f$ gives zero contrast even when all four inputs differ.
\end{proof}

\begin{proposition}[Observed collision error floor]
Partition the observed inputs into exact encoding classes $C$. For squared error, every deterministic function of the encoded input has weighted empirical risk at least $\sum_C\sum_{i\in C}w_i(y_i-\bar y_C)^2$, where $\bar y_C=\sum_{i\in C}w_iy_i/\sum_{i\in C}w_i$ for positive-mass classes.
\end{proposition}
\begin{proof}
The function assigns one prediction $p_C$ to each class. Expanding $p_C-y_i=(p_C-\bar y_C)+(\bar y_C-y_i)$ gives
\[
 \sum_{i\in C}w_i(p_C-y_i)^2
 =\Bigl(\sum_{i\in C}w_i\Bigr)(p_C-\bar y_C)^2
  +\sum_{i\in C}w_i(y_i-\bar y_C)^2,
\]
because $\sum_{i\in C}w_i(\bar y_C-y_i)=0$. The first term is nonnegative. Summing over classes proves the lower bound. An unrestricted function assigning $p_C=\bar y_C$ attains it on these observed inputs; the particular neural architecture need not do so. Zero-mass classes contribute nothing. This is a finite-data identity, not a population minimax or generalization guarantee.
\end{proof}

\subsection{Shared-reference and leave-margin B3 sensitivities}\label{app:b3sensitivity}
The 140 distinct contrasts cross 14 antisense and 10 sense replacements. They use 165 measured states on one sequence background and share reference/marginal conditions. Writing all endpoint means as a vector $Y$ and the signed contrast matrix as $C$, the contrast covariance would be $C\operatorname{Cov}(Y)C^\top$. Endpoint SDs and the aggregate replication description do not supply this covariance. No independent-rectangle bootstrap or inferred mean-difference standard error is used.

For fixed model contrasts $p_j$ and measured contrasts $I_j$, a common reference perturbation $\delta$ sends every $I_j$ to $I_j+\delta$. With uniform weights, direct expansion gives
\begin{align}
 M(\delta)&=M(0)-2\delta\,\overline{(p-I)}+\delta^2,\\
 Z(\delta)&=Z(0)+2\delta\,\bar I+\delta^2,\\
 M(\delta)-Z(\delta)&=M(0)-Z(0)-2\delta\,\bar p,
\end{align}
where $M$ is model MSE and $Z$ is zero-contrast MSE. The last line follows by subtraction, since $\overline{(p-I)}+\bar I=\bar p$. The audit reports shifts of minus one, zero and plus one reported reference SD solely as descriptive sensitivity values; they are not confidence limits for the reference mean.

Leaving one antisense replacement removes its ten contrasts; leaving one sense replacement removes its fourteen contrasts. Predictions stay fixed and remaining errors are recomputed against the same zero control. These 24 omissions expose margin dependence but are neither new fitted folds nor independent biological replications. Observed assay nonadditivity, a population biological interaction, representability of an input change and the trained model's contrast are different objects.

The reported reference SD is 0.021415097 on the fractional response scale. Shifts retain 165 endpoints/140 contrasts; an antisense omission retains 154/130 and a sense omission 150/126. Predictions remain fixed. Previously reported MSEs and differences are reproduced to $2\times10^{-15}$; added spreads use the same identities. Table~\ref{tab:b3sensrisk} retains the GNN reversal after omitting JC5 and no-message reversals after omitting JC10 or JC5. Tree and CNN remain worse than zero throughout; no method reverses under reference shifts.

For every reduced panel, rebuild $C_J$ over retained endpoints, restrict the input-class indicator to $H_J$, and recompute an orthonormal basis for $\operatorname{col}(C_JH_J)$ using SVD with relative cutoff $10^{-10}$; rational elimination verifies the rank. The full-panel projector is never sliced and reused. The residual $(I-P_J)t_J$ has energy share 13.39--15.89\% under reference shifts, 8.00--15.83\% under antisense omissions and 13.69--16.65\% under sense omissions. These are finite-panel influence values, not sampling uncertainty or causal error allocation. Full reduced maps, bases and outcomes remain in the store.
\begin{table}[ht]
\centering\small
\begin{tabular}{@{}lrrrrrr@{}}\toprule
Sensitivity & Method & $m/q$ & Model MSE & Zero MSE & $10^3\Delta$ & Flips\\\midrule
Omit AS & Tree & 154/130 & 0.059628--0.077251 & 0.056062--0.073257 & 2.456--4.031 & 0\\
Omit AS & GNN & 154/130 & 0.056195--0.073399 & 0.056062--0.073257 & 0.015--0.143 & 0\\
Omit AS & No-msg & 154/130 & 0.056160--0.073339 & 0.056062--0.073257 & -0.045--0.098 & 1\\
Omit AS & CNN & 154/130 & 0.056290--0.073475 & 0.056062--0.073257 & 0.007--0.229 & 0\\
Omit SS & Tree & 150/126 & 0.055818--0.079524 & 0.053743--0.075358 & 1.176--4.188 & 0\\
Omit SS & GNN & 150/126 & 0.053882--0.075507 & 0.053743--0.075358 & -0.009--0.151 & 1\\
Omit SS & No-msg & 150/126 & 0.053848--0.075443 & 0.053743--0.075358 & -0.032--0.104 & 1\\
Omit SS & CNN & 150/126 & 0.053966--0.075589 & 0.053743--0.075358 & 0.000--0.231 & 0\\
Ref. shift & Tree & 165/140 & 0.066229--0.078566 & 0.062803--0.074524 & 3.426--4.042 & 0\\
Ref. shift & GNN & 165/140 & 0.062923--0.074664 & 0.062803--0.074524 & 0.120--0.140 & 0\\
Ref. shift & No-msg & 165/140 & 0.062871--0.074608 & 0.062803--0.074524 & 0.069--0.084 & 0\\
Ref. shift & CNN & 165/140 & 0.062985--0.074749 & 0.062803--0.074524 & 0.183--0.225 & 0\\
\bottomrule\end{tabular}
\caption{Ranges over three reference shifts, fourteen antisense omissions or ten sense omissions. $m/q$: retained endpoints/contrasts. Paired differences $\Delta$ are scaled by $10^3$ and computed before taking ranges, not by subtracting range endpoints. Flips count reversals from the unperturbed ordering. These are dependent descriptive diagnostics, not confidence limits.}\label{tab:b3sensrisk}
\end{table}
\begin{table}[ht]
\centering\small
\begin{tabular}{@{}lrrrrr@{}}\toprule
Sensitivity & Method & Predicted SD & Recorded SD & Largest influence & Gap change\\\midrule
Omit AS & Tree & 0.008806--0.018376 & 0.204264--0.229500 & JC10 & -0.001278\\
Omit AS & GNN & 0.000252--0.001648 & 0.204264--0.229500 & JC10 & -0.000115\\
Omit AS & No-msg & 0.000299--0.001714 & 0.204264--0.229500 & JC10 & -0.000122\\
Omit AS & CNN & 0.001113--0.002984 & 0.204264--0.229500 & JC10 & -0.000197\\
Omit SS & Tree & 0.007219--0.018653 & 0.203193--0.230994 & JC5 & -0.002558\\
Omit SS & GNN & 0.000327--0.001674 & 0.203193--0.230994 & JC5 & -0.000139\\
Omit SS & No-msg & 0.000348--0.001740 & 0.203193--0.230994 & JC5 & -0.000108\\
Omit SS & CNN & 0.000642--0.003028 & 0.203193--0.230994 & JC5 & -0.000204\\
Ref. shift & Tree & 0.017767 & 0.222442 & 1 & -0.000308\\
Ref. shift & GNN & 0.001589 & 0.222442 & 1 & -0.000010\\
Ref. shift & No-msg & 0.001652 & 0.222442 & -1 & 0.000008\\
Ref. shift & CNN & 0.002879 & 0.222442 & 1 & -0.000021\\
\bottomrule\end{tabular}
\caption{Spread ranges and the omitted margin (or reference multiplier) with largest absolute change of paired error difference. A positive reference multiplier adds one reported SD to the shared activity reference.}\label{tab:b3sensspread}
\end{table}
\subsection{Ranking reconciliation and uncertainty scope}
Historical and current predictions are joined by exact persistent observation and assay-pool identifiers. All 1,839 APP records and all 17 pool memberships agree. Larger measured inhibition is the preferred direction. Every campaign is rescored with the same outcome average-rank percentile and fractional selection mass among exact prediction ties at the top-five boundary. A saved ensemble row is an aggregate, not an extra seed. Ensemble predictions are reconstructed from actual member rows and the historical reported scores are retained alongside the common-definition scores.

The original, frozen-diagnostic, factorial and all-ENsiRNA deployment protocols remain separate. Training populations, support/context handling, selection effort, optimization budget and ensemble sizes changed jointly. The ranking reversal limits protocol-stability claims; it does not make a well-defined descriptive ranking score meaningless or identify one cause of the reversal. Pools share one related patent family and overlap in sequence components. Their number does not justify an independent-pool significance test.

Component resampling conditions on fixed fitted predictions and the observed cohorts. The independent review calculation uses RNG seed 914221; campaign comparisons retain the previously defined seed and 10,000 draws. Differences between valid resampling RNG realizations are distinguished from errors in the point estimate. Ten-seed variability is reported separately. An interval containing zero denotes an unresolved comparison, not proof of equivalence; no post-hoc equivalence margin is introduced.

\section{Complete focused numerical summaries}\label{app:focusedresults}
All numbers here derive from completed v4 fits or explicitly reused fixed predictions; no seed or failed attempt is removed. The focused v4 cohort results appear in main Table~\ref{tab:activity}, with matched training-constant comparisons in Appendix~\ref{app:constantchecks}. Historical v3 results are retained separately in Appendix~\ref{app:calibrationresults}; the metric identities are given in Appendix~\ref{app:metrics}. The complete per-cohort, per-weighting and per-seed rows are columns of \texttt{activity\_metrics} and \texttt{variance\_decomposition} in the canonical store rather than reprinted pages.
Row-level records at full precision, each with its column schema and SHA256, are held in the
separately indexed scientific archive rather than reprinted here.

\section{Structural attribution audit: complete specification}

\subsection{Exact contrast certificate and its limits}\label{app:structuralcertificate}
Let $s_1,\ldots,s_m$ be the distinct raw endpoint states used by a finite collection of contrasts. An encoding instance includes its fitted training-only preprocessing. Write $e(s_i)$ for the complete inputs consumed by its deterministic decoder, including all ordered slots and ancillary branches. Partition these states into $k$ classes of exactly equal encodings and define $H\in\{0,1\}^{m\times k}$ by $H_{ij}=1$ precisely when state $i$ belongs to class $j$. Let $C\in\mathbb R^{q\times m}$ contain the signed contrast coefficients. Repeated endpoints contribute their coefficients with multiplicity. An unrestricted deterministic decoder on these classes is an arbitrary vector $g\in\mathbb R^k$, and its contrast vector is $CHg$.
\begin{proposition}
$CH=0$ if and only if $CHg=0$ for every unrestricted deterministic decoder $g$ on the observed encoding classes.
\end{proposition}
\begin{proof}
If $CH=0$, multiplication by any $g$ gives $CHg=0$. Conversely, assume $CHg=0$ for every $g\in\mathbb R^k$. For each $j=1,\ldots,k$, choose the coordinate vector $g=e_j$. Then $CH e_j$ is the $j$th column of $CH$ and is zero. Since every column is zero, $CH=0$. Equivalently, if any entry $(CH)_{ij}$ is nonzero, the decoder taking value one on class $j$ and zero on the other classes witnesses a nonzero $i$th contrast. This proves both implications without continuity, probability or learning assumptions.
\end{proof}
For coefficients $(1,-1,-1,1)$, the condition is exactly equality of the positive and negative encoding multisets. It concerns this contrast, not every explanation functional. Restricted decoders may have additional limitations when $CH\ne0$: the basis-vector decoder used in the converse need not belong to a fixed neural family. The identity is elementary and is not an algorithmic novelty claim. For an ensemble, a sufficient certificate is that each constituent's own complete-input encoding passes the test. A hidden-state coincidence found after fitting is a checkpoint-specific observation, not a pre-fit certificate.

The implementation compares named tensor arrays with their shapes, dtype and exact C-order bytes. It never rounds, reorders slots, substitutes graph isomorphism, or turns an invalid record into a zero input. SHA256 fingerprints locate identities; actual byte payloads decide equality. The tensors compared are those of the single audit signature defined in Appendix~\ref{app:inputcollision}. Original node and edge conventions and their unresolved chemical assumptions remain as specified in the data appendix. The additional real-arithmetic routing quotient remains distinct from literal equality.

For a set $F$ of verified forced-zero contrasts and nonnegative weights $w_q$ with positive total, every decoder covered by this certificate incurs at least $\sum_{q\in F}w_q\widetilde I_q^2/\sum_qw_q$ recorded-label squared error. Each flagged prediction is zero, so its squared residual equals $\widetilde I_q^2$; all other residual squares are nonnegative, proving the inequality. This is a finite recorded-label bound with the full denominator. It neither removes measurement noise nor supplies a biological population-risk lower bound. Here $F$ is empty, so its contribution is zero.

\subsection{Why four coincident corners do not settle other explainers}\label{app:counterexample}
The reference scientific model target is $f(e(s_{11}))-f(e(s_{10}))-f(e(s_{01}))+f(e(s_{00}))$. Its players are the declared antisense and sense replacements, its baseline is the measured $00$ condition, and masking means selecting the corresponding full, re-encoded measured state. Other annotations and assay context are fixed. No continuous interpolation is needed.

For the analytic counterexample let $e(a,b,z)=ab(1-z)$ and $g_\theta(t)=\theta t$. At $z=1$, all four $a,b$ encodings equal zero, and the four-term contrast equals zero for every $\theta$. Define a three-player coalition game with baseline $(0,0,0)$, target $(1,1,1)$ and masking that sets precisely the players in a coalition to one. Re-encoding gives $v(S)=\theta\,1_{\{a,b\}\subseteq S}(1-1_{z\in S})$. For a pair, the Shapley interaction index weights background $S\subseteq\{z\}$ by $|S|!(3-|S|-2)!/(3-1)!$ \citep{grabisch1999}. Both weights are $1/2$. The second difference at $S=\varnothing$ is $\theta$, while that at $S=\{z\}$ is zero. Their weighted sum is $\theta/2$. Thus equality at four fixed-background corners does not fix this interaction index. All eight binary inputs are admissible for this explicitly defined analytic game; they are not synthetic biological outcomes. The implemented checks use $\theta=-2,0,3$ and retain the null case.

For continuous-path explanation, equality at endpoints alone cannot determine intermediate function derivatives or values. A term vanishing at the selected corners can be nonzero between them or in other backgrounds. No integrated-gradient or Hessian biological interpretation is inferred from the four-corner certificate. Our numerical gradient diagnostic instead evaluates $\nabla_X\hat y$ at the supplied encoded query, holding $A$, mask and context fixed \citep{saliency2013}. It is a signed local Euclidean derivative, with no baseline, masking or integration path. It is not an attribution to a feasible molecular edit, and its directions need not preserve one-hotness or joint chemical validity.

\subsection{Training-independent blocks: proof before perturbation}\label{app:trainingindependence}
Fix a finite training collection, its preprocessing, and all dropout masks. For R0 suppose $A_{rij}=0$ for every training record and all nodes $i,j$ for relation $r$. Consider changing only independently parameterized matrices $W_r^\ell$ for these absent relations. Embedding states are unchanged because they do not depend on these matrices. If the states entering layer $\ell$ are unchanged, every absent-relation term $A_{rij}h_j^\ell W_r^\ell$ is zero for every value of $W_r^\ell$. Terms on the other relations are unchanged. The total degree depends only on $A$, and the self transform, bias, residual, fixed dropout mask, LayerNorm and padding mask therefore return unchanged next states. Induction proves identical states through every layer and identical readout predictions on every training example for every value of the selected matrices. This holds for every dropout realization, batch and parameter value of the other blocks, not merely at one zero-gradient checkpoint.

Consequently any data loss that is a function only of these training predictions and fixed labels is independent of the selected blocks, including sums of endpoint losses and pair differences when both endpoints belong to the same covered training collection. Its parameter derivative is identically zero wherever defined. Our diagnostic checkpoints use activity training and do not introduce pair endpoints. The proof does not infer global prediction invariance: a held-out graph with that relation can use the selected block. Nor does it assert invariance of a selection rule that examines additional validation inputs; training-data independence and every possible source of training-procedure information are different statements.

For R1 the effective forward/backward backbone transforms use the respective shared bases $W_2,W_5$ plus supported residuals. Only residuals with a fixed false gate are independently absent; shared bases cannot be included merely because their own relation count is zero. False-gated residuals cancel from every forward, including held-out graphs, so the same perturbation has a stronger, architecture-imposed invariance. The verified original partition has four absent relations $0,1,3,4$, two layers and $32\times32$ matrices: $4\cdot2\cdot32^2=8192$ scalars. The corrected models happen to have 8192 false-gated residual scalars in the inspected partitions; these are not 8192 independently unconstrained effective backbone routes.

A zero data derivative does not mean an unchanged parameter. With zero moments and zero data gradient, AdamW still applies the decoupled factor $1-\eta_t\lambda$ to a parameter represented in the optimizer; regularization can also add an objective derivative. Weight changes do not prove label learning. Conversely an input gradient is not a parameter-loss gradient. The present architecture has no pretrained message matrices or unrecorded shared parameter aliases in these blocks. Structural independence concerns the data term; optimizer/regularization choices remain disclosed.

\subsection{Bounded diagnostic protocol and arithmetic}\label{app:perturbation}
The protocol used three saved seeds per arm, fixed Rademacher directions and scales $0$ and $\pm0.25$ of each block's root-mean-square magnitude. All 2,518 training predictions per checkpoint stayed bitwise equal. In the original model (R0) the largest of the 1,839 APP prediction changes is 0.003300 and 1,601 to 1,721 global ranks move, yet top-five selection within the seventeen assay pools changes in only 15/102 cases (minimum overlap 0.800), against 0/102 for R1; overlapping pools are not independent trials. The largest raw input-gradient change, 0.005870 on sixteen queries, is a local sensitivity, not a feasible chemical intervention. R1's gated residuals give the expected null, zero-scale controls pass and B3 changes are exactly zero.

The protocol used saved v3 deployment checkpoints for R0 and R1 GNN, seeds 1103, 2207 and 3301; all 2518 training inputs, all 1839 APP inputs and all 165 primary-panel inputs; first sixteen lexicographic APP identifiers for raw input derivatives. For each eligible layer/relation block draw fixed Rademacher signs with RNG seed 20260914 and multiply by that block's original root mean square parameter magnitude. Apply scales $0,-0.25,+0.25$ to this fixed direction. Thus the perturbation Frobenius norm in each block is at most one quarter of its original Frobenius norm. Neither labels, diagnostic outcomes nor rank changes select a direction, scale or query. No optimization occurs.

Predictions use deterministic CPU float32 inference in batches of 128, two Torch threads, disabled dropout, original scale restoration, and no clipping or recalibration. Every training prediction must be bitwise equal to its unperturbed value. All endpoint predictions are reused across rectangles. APP ranks use average ranks among all 1839 rows; these are global diagnostic ranks, not the seventeen assay-pool efficacy scores. Raw-gradient changes cover sixteen queries per checkpoint, not all APP records. B3 diagnostic deployment checkpoints include the source in historical training; they are separate from the source-held-out checkpoints used for the B3 measured performance table. Their zero changes are not a new external-evaluation result. Models are perturbed only in memory, restored exactly, and original checkpoint SHA256 hashes are checked. No modified checkpoint is saved. No training optimizer is constructed or advanced.

\subsection{Experimental support versus the one-hot simplex}
For a four-letter one-hot input $x$, $\sum_c x_c=1$. The normal vector is proportional to $(1,1,1,1)$, so the gradient projection is $G_c-\frac14\sum_dG_d$ \citep{koo2023}. A fixed-G experimental position occupies a single vertex within this simplex; a Y position permits only C and U. Their missing directions include tangent directions of the full simplex, so subtracting its normal component does not resolve experimental-support absence. For example $f(x)=\theta x_A$ at a fixed-G position is zero on all observed inputs, but its projected gradient is $\theta(3/4,-1/4,-1/4,-1/4)$ in A,C,G,U order and remains unconstrained by those observations.

The actual MAVE-NN input release has 30483 RNA sequences of length nine, with G fixed at position four and C/U at position five. There are three absent nucleotide columns at position 4 and two at position 5, totaling five. Using a DNA T channel for this RNA-U release would incorrectly count thirteen absent columns; the audit corrects that adapter error. An active constant G channel is not absent: for preactivation $u=w_Gx_G+b$, $\partial L/\partial w_G=(\partial L/\partial u)x_G$ can be nonzero when $x_G=1$. Its effect can be confounded with the intercept under this design. In an additive position-specific map five absent columns correspond to five $\theta_{lc}$ coefficients; with a dense width-$h$ first layer they can affect $5h$ weights; a shared convolution requires a separate analysis across all receptive-field positions.

MAVE-NN explicitly models genotype-to-phenotype and measurement maps \citep{mavenn2022}. Its BRCA2 exon-17 splice-site assay measures exon inclusion, not siRNA efficacy. The pinned release and layer definitions establish this support calculation. No supplied BRCA2 workshop checkpoint, implementation identifier or claimed attribution percentages could be identified. The categorical input-gradient workshop paper of \citet{majdandzic2021workshop} uses synthetic motif data, so it is not silently substituted as that missing experiment. No workshop percentages or additional fitted-model result are invented.

\subsection{Exact rank of the encoding-class contrast map}\label{app:contrastrank}
We continue the notation above: $C\in\R^{q\times m}$ holds the signed contrast coefficients over the $m$ distinct raw endpoint states of a finite recorded panel, and $H\in\{0,1\}^{m\times k}$ assigns each state to its complete-input equivalence class for one encoding instance. An unrestricted deterministic decoder on those classes is a vector $g\in\R^k$ and realises the contrast vector $CHg$.
\begin{proposition}\label{prop:rank}
The set of contrast vectors attainable by unrestricted deterministic decoders on these classes is exactly the column space of $CH$. A fixed linear functional $u^{\top}$ of the $q$ contrasts is identically zero for every such decoder if and only if $u^{\top}CH=0$, so the space of forced-zero functionals has dimension $q-\operatorname{rank}(CH)$. If in addition $C\1_m=0$, then $\operatorname{rank}(CH)\le k-1$.
\end{proposition}
\begin{proof}
Attainability is the definition of a column space. If $u^{\top}CH=0$ then $u^{\top}(CHg)=0$ for every $g$; conversely if $u^{\top}CHg=0$ for every $g$, taking $g=e_j$ for each $j$ shows every entry of $u^{\top}CH$ vanishes. The forced-zero functionals are therefore the left null space of $CH$, of dimension $q-\operatorname{rank}(CH)$ by rank--nullity. Finally $H\1_k=\1_m$, so $CH\1_k=C\1_m=0$: the columns of $CH$ satisfy one linear relation and at most $k-1$ of them are independent.
\end{proof}
Proposition~\ref{prop:rank} bounds a dimension. It does not say that any particular row cancels: row $i$ is forced to zero exactly when the $i$th row of $CH$ is the zero vector, which is a strictly stronger requirement than a rank deficit. Both statements were computed, not assumed.

For the primary panel, $q=140$, $m=165$ and every audited encoding instance has $k=90$. The forty instances, four methods by ten seeds, induce the same partition of the 165 states, verified by exact label comparison, so the concatenated constituent blocks of an ensemble are identical and the combined column space equals the single constituent space; every ensemble weight is $1/10$ and none is zero. No rank bound is transferred between differing partitions. With integer $C$ and $H$ and rational forward elimination, $\operatorname{rank}(C)=140$ and $\operatorname{rank}(CH)=72$, against the dimension bound $89$. Hence the audited encodings impose $140-72=68$ independent homogeneous linear constraints on the joint \emph{predicted} contrast vector: every $u$ in the $68$-dimensional left null space of $CH$ satisfies $u^{\top}CHg=0$ for every decoder $g$. The recorded contrast vector need not satisfy those constraints, and they are not $68$ vanishing individual interactions nor $68$ independent biological observations. Explicit exact left-null witnesses were verified to annihilate $CH$; the machine-readable witnesses are rows of the canonical store rather than printed coefficient dumps. No row of $CH$ is zero, which is consistent with the complete-input test finding no forced cancellation in any of the 140 rectangles.

Let $P$ be the orthogonal projector onto the column space of $CH$ and let $\widetilde I\in\R^{q}$ be the recorded contrast vector. For equal weights $1/q$ and any decoder $g$,
\[
 \tfrac1q\bigl\|CHg-\widetilde I\bigr\|^2\ \ge\ \tfrac1q\bigl\|(\mathrm{Id}-P)\widetilde I\bigr\|^2 ,
\]
because $CHg$ ranges over the column space and orthogonal projection minimises the distance to it. Moreover, for any fitted contrast vector $p\in S=\operatorname{col}(CH)$ the decomposition $\widetilde I-p=(\mathrm{Id}-P)\widetilde I+(P\widetilde I-p)$ has orthogonal summands, so
\[
 \tfrac1q\bigl\|\widetilde I-p\bigr\|^2=\tfrac1q\bigl\|(\mathrm{Id}-P)\widetilde I\bigr\|^2+\tfrac1q\bigl\|P\widetilde I-p\bigr\|^2 .
\]
For a fitted masked contrast vector $p\in S=\operatorname{col}(CH)$,
write $F=q^{-1}\|(\mathrm{Id}-P)\tilde I\|^2$.
For a restored predictor $p'$ that need not lie in $S$, the improvement is
\[
 \frac{\|\tilde I-p\|^2}{q}
 -\frac{\|\tilde I-p'\|^2}{q}
 =F+\frac{\|P\tilde I-p\|^2}{q}
 -\frac{\|\tilde I-p'\|^2}{q}.
\]
Consequently, $F$ does not upper-bound improvement after refitting.
Only the unrestricted oracle minimum decreases by exactly $F$ when the
restored contrast space contains the recorded target vector.
Numerically $\tfrac1q\|\widetilde I\|^2=0.06820465$ and $\tfrac1q\|(\mathrm{Id}-P)\widetilde I\|^2=0.00997952$, so the encoding constraints exclude 14.63\% of the recorded contrast energy. For the GNN ensemble the split reads $0.06833507\approx0.00997952+0.05835556$: the same residual is 14.6038\% of the fitted squared error, while 85.3962\% of it lies \emph{within} the permitted contrast subspace. These two percentages have different denominators, recorded contrast energy and fitted squared error, and must not be conflated. Most fitted-model squared error therefore lies inside the permitted subspace, and the mechanism producing the nearly constant interaction predictions remains unresolved. The residual is a finite-panel representational restriction for these recorded measurements and weights: not a noise-corrected biological error, not a population risk bound, and not a causal allocation of attenuation to representation, architecture, optimisation or data support. The unrestricted decoder used here need not be realizable by the fitted GNN, tree or CNN family. As an implementation check, each fitted ensemble contrast vector lies in $S$ with a maximum deviation of $1.514\times10^{-16}$, below the double-precision machine epsilon $2.22\times10^{-16}$, as Proposition~\ref{prop:rank} requires. This applies the general contrast-space characterization to one recorded panel on one sequence background. The resulting floor is specific to the recorded measurements, encoding and weights. It provides no population-risk guarantee, and no shared-control interval is constructed without a covariance model.

\subsection{Source decomposition identities and what they do not imply}\label{app:decomposition}
Let $w_i>0$ be normalised evaluation weights, $y_i$ recorded targets and $p_i$ fixed predictions. For a partition into groups $s$ put $a_s=\sum_{i\in s}w_i$, $\mu=\sum_iw_iy_i$, $\nu=\sum_iw_ip_i$, and let $\mu_s,\nu_s,\Var_s(y),\Var_s(p),\operatorname{Cov}_s(y,p)$ be the group means, variances and covariance under the renormalised weights $w_i/a_s$.
\begin{proposition}\label{prop:decomp}
With $V=\sum_iw_i(y_i-\mu)^2$ and $\MSE=\sum_iw_i(p_i-y_i)^2$,
\[
 V=\underbrace{\sum_sa_s\Var_s(y)}_{V_{\mathrm{within}}}+\underbrace{\sum_sa_s(\mu_s-\mu)^2}_{V_{\mathrm{between}}},
\]
\[
 \MSE=\sum_sa_s\bigl[\Var_s(y)+\Var_s(p)-2\operatorname{Cov}_s(y,p)+(\nu_s-\mu_s)^2\bigr],
\]
and $\operatorname{Cov}(y,p)=\sum_sa_s\operatorname{Cov}_s(y,p)+\sum_sa_s(\mu_s-\mu)(\nu_s-\nu)$.
\end{proposition}
\begin{proof}
Write $y_i-\mu=(y_i-\mu_{s(i)})+(\mu_{s(i)}-\mu)$ and expand. The cross term is $2\sum_sa_s(\mu_s-\mu)\sum_{i\in s}(w_i/a_s)(y_i-\mu_s)=0$ because each inner sum vanishes. The same splitting applied to $p$ gives the covariance identity. For the second display, within group $s$ the weighted mean of $(p_i-y_i)^2$ equals $\Var_s(p-y)+(\nu_s-\mu_s)^2$ and $\Var_s(p-y)=\Var_s(y)+\Var_s(p)-2\operatorname{Cov}_s(y,p)$; summing with weights $a_s$ gives the claim.
\end{proof}
Pooled $R^2=1-\MSE/V$ uses the pooled denominator $V$. It is not an average of the per-group $R^2_s=1-\MSE_s/\Var_s(y)$, which use different denominators, and $R^2_s$ is undefined when $\Var_s(y)=0$. Three distinct mechanisms make most per-group values negative while the pooled value is positive: a small $\Var_s(y)$ in the denominator, a group-specific calibration offset $(\nu_s-\mu_s)^2$, and unequal group sizes, since a count of groups is not a row weight. One recorded instance of the first mechanism is a 20-row source with label variance 0.000711, where the R1 GNN has $R^2\approx-93.78$ although its absolute error is not large; the per-source rows for every method are in \texttt{variance\_sources}.

These identities do not license the inference that between-group variance dominates or that a model has learned group levels. On the released grouped benchmark under equal row weights, $V_{\mathrm{between}}$ is $3.98\%$ of $V$ across the seventeen source categories and $3.33\%$ across the ten study-linked components, so the between-group part is the small part. Moreover, for this GNN evaluation the covariance splits as $+0.00306249$ within and $-0.00055903$ between, so the positive pooled score is carried by the within-source component while the between-source component is negative. That is an observed evaluation covariance for this arm only: the chemistry tree shows a positive between-source covariance on the same rows, and a negative covariance does not establish the absence of source-dependent learning. Source identifiers, study-linked components and assay identifiers are three different partitions of the same rows and are not renamed as one another; the recorded panel does not resolve assay identity, so no assay-level partition is asserted.

Two further qualifications are kept explicit. A group-mean oracle, and any score computed after centring predictions on evaluation-group means, read the evaluation labels themselves; they are retrospective diagnostics, not deployable corrected models and not model-risk lower bounds. And a negative within-group $R^2$ does not prove zero discrimination or the absence of chemistry information: calibration, a small label variance and measurement noise each produce it on their own. All scenarios here reuse the same recorded rows and the same saved model families, so the reported comparisons are dependent re-analyses rather than independent replications.

\subsection{Retained calibration results}\label{app:calibrationresults}
The retained calibration table supplements the main results. Cohort and adjudication counts appear in the data and robustness sections; Appendix~\ref{app:fitledger} gives deduplicated historical fit counts.

\begin{table}[ht]
\centering\small
\begin{tabular}{@{}llrrrrr@{}}\toprule
Cohort & Original v3 model & MSE & $R^2$ & $r$ & Pred. SD & Bias\\\midrule
APP & R1 GNN & 0.1625 & -0.001 & -0.018 & 0.0072 & -0.0081\\
 & R1 no-message & 0.1630 & -0.005 & 0.015 & 0.0085 & -0.0286\\
 & Chemistry tree & 0.1896 & -0.168 & -0.021 & 0.0103 & -0.1644\\
 & Token CNN & 0.1644 & -0.013 & 0.080 & 0.0036 & -0.0483\\
S7 & R1 GNN & 0.1093 & -0.215 & 0.112 & 0.0059 & 0.1403\\
 & R1 no-message & 0.1047 & -0.164 & 0.117 & 0.0066 & 0.1231\\
 & Chemistry tree & 0.0919 & -0.021 & 0.017 & 0.0089 & -0.0441\\
 & Token CNN & 0.1017 & -0.130 & 0.108 & 0.0034 & 0.1092\\
\bottomrule\end{tabular}
\caption{Unchanged v3 released-trained activity results on original response scales. Label SD is 0.4028 on APP and 0.3000 on S7. Prediction spread and mean calibration are distinct from discrimination; no test-fitted adjustment is used.}\label{tab:calibration}
\end{table}

\subsection{Current constants and model comparisons}
For each evaluation identity, weights are either one per row or inverse study-component size normalised to average one. Means, variances and losses divide by total weight. We report pooled $R^2=1-\MSE/V$; it is undefined when $V=0$, and it is never averaged across folds or groups. Training constants are fold-specific training-mean predictors; their pooled prediction SD need not be zero. The test-mean constant is retrospective, not deployable and not a universal risk floor. APP and S7 labels and weights are identical across scenarios and each contains a single study component, so their row and equal-component weightings coincide and are reported once. The full identity hash includes identifiers, labels, weights and the declared response convention, and every per-evaluation row remains in the canonical store table \texttt{activity\_metrics}.
Grouped coverage, targets and retrospective oracles for all four cohort/weighting combinations (2,607 and 1,003 rows over 150 and 149 sequence components; oracle MSE 0.062573, 0.075091, 0.059018 and 0.052895) are rows of \texttt{activity\_metrics}, with APP and S7 in main Table~\ref{tab:activity}; each has one study component, so their two weightings coincide. The oracle is the retrospective evaluation mean, not a deployable control; essential controls appear in Appendix~\ref{app:constantchecks}.

\paragraph{Indexed source decompositions.}
The indexed decomposition rows reuse saved predictions; nothing was refitted. Each identity of Proposition~\ref{prop:decomp} was verified to machine precision, and each variance identity was verified a second time by centring rows on their own group means. The summaries concern multi-source grouped cohorts, where the decomposition is informative; APP and S7 each fall in one source family and one study component, so $V_b=0$ there by construction. All 304 decomposition records for every method, weighting and partition, the retrospective group-mean oracle (which equals $V_w$ exactly), the group-centred loss and the within/between covariance split quoted in the main text are rows of \texttt{variance\_decomposition}; neither retrospective quantity is a deployable corrected model.

\paragraph{Per-source denominators.}
Per-source $R^2_s=1-\MSE_s/V_s$ uses the group label variance as its denominator and is undefined when that variance is zero. These are category-level summaries, not independent replications. All 1,144 per-source records for every campaign, cohort, weighting and method, each with its own retained denominator $V_s$ and row count, are rows of \texttt{variance\_sources}; counts of positive and negative ratios do not weight observations.

\subsection{Endpoint and marginal-effect reconstruction}
The source-held-out export is the authoritative input; the source-included deployment checkpoints of the perturbation experiment are not used. The 560 rectangle-corner appearances reduce to 165 measured endpoint states (one shared reference, fourteen antisense, ten sense, 140 combinations) whose repeated labels agree exactly. Ensembles are the mean of the ten actual final seeds, excluding the saved ensemble row and the deterministic seed-0 rows; they reproduce the stored ensemble to $1.3\times10^{-16}$ and the interaction errors of Table~\ref{tab:floor} exactly.
\begin{table}[ht]
\centering\small
\begin{tabular}{@{}lrrrrrrr@{}}\toprule
Method & MSE & $R^2$ & $r$ & Pred. SD & Label SD & Retro. mean & Train const\\\midrule
GNN & 0.077350 & 0.031844 & 0.184788 & 0.040265 & 0.282656 & 0.079894 & 0.080173\\
No-msg & 0.077353 & 0.031810 & 0.185135 & 0.038465 & 0.282656 & 0.079894 & 0.080173\\
Tree & 0.078546 & 0.016873 & 0.167266 & 0.076857 & 0.282656 & 0.079894 & 0.080173\\
CNN & 0.077276 & 0.032767 & 0.190711 & 0.045136 & 0.282656 & 0.079894 & 0.080173\\
\bottomrule\end{tabular}
\caption{Endpoint prediction over the 165 deduplicated states on the original response scale. The retrospective column is the MSE of the evaluation-mean predictor; the final column is the MSE of a genuine training-fitted constant recovered from the ten fit records, never from these evaluation labels.}\label{tab:app_endpoint}
\end{table}
\begin{table}[ht]
\centering\small
\begin{tabular}{@{}lrrrrrrr@{}}\toprule
Method & Family & $n$ & MSE & Zero MSE & $r$ & Pred. SD & Label SD\\\midrule
GNN & AS effect & 14 & 0.049856 & 0.051986 & 0.028779 & 0.024352 & 0.170616\\
GNN & SS effect & 10 & 0.003335 & 0.000855 & 0.286065 & 0.034469 & 0.013953\\
GNN & Interaction & 140 & 0.068335 & 0.068205 & -0.090263 & 0.001589 & 0.222442\\
No-msg & AS effect & 14 & 0.050546 & 0.051986 & 0.019753 & 0.021060 & 0.170616\\
No-msg & SS effect & 10 & 0.003236 & 0.000855 & 0.269464 & 0.034474 & 0.013953\\
No-msg & Interaction & 140 & 0.068281 & 0.068205 & -0.034429 & 0.001652 & 0.222442\\
Tree & AS effect & 14 & 0.042903 & 0.051986 & 0.116779 & 0.061459 & 0.170616\\
Tree & SS effect & 10 & 0.008423 & 0.000855 & 0.286572 & 0.064857 & 0.013953\\
Tree & Interaction & 140 & 0.071939 & 0.068205 & -0.177024 & 0.017767 & 0.222442\\
CNN & AS effect & 14 & 0.049311 & 0.051986 & 0.055427 & 0.023715 & 0.170616\\
CNN & SS effect & 10 & 0.004040 & 0.000855 & 0.244159 & 0.041626 & 0.013953\\
CNN & Interaction & 140 & 0.068409 & 0.068205 & -0.046847 & 0.002879 & 0.222442\\
\bottomrule\end{tabular}
\caption{The three effect families with directions and response scale fixed: antisense $f(A,0)-f(0,0)$, sense $f(0,B)-f(0,0)$ and the original reference-based interaction. Seed dispersion is separate and no independent-marginal interval is claimed.}\label{tab:app_marginal}
\end{table}

\subsection{Encoding and rank ablation}

\begin{figure}[ht]\centering\includegraphics[width=\textwidth]{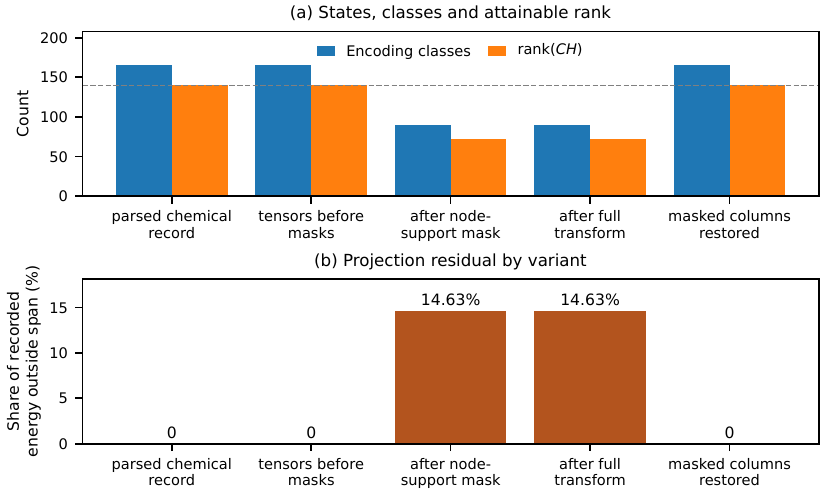}
\caption{Label-independent encoding ablation on the fixed panel, with the raw states, endpoint identities, response convention and $C$ held fixed. (a)~Encoding classes and the exact attainable rank per variant; the dashed line is $\operatorname{rank}(C)=140$. (b)~Share of recorded contrast energy outside the attainable span, linear scale. Full-row-rank variants have exact zero residual in real arithmetic; their computed residuals are roundoff. Parsed chemistry, the pre-mask tensors and the mask-restored diagnostic all keep 165 classes and full rank; applying the training-only node-support mask alone collapses them to 90 classes and rank 72. Higher rank means more representable contrasts, not better prediction; exact class counts, ranks, null dimensions and residual shares per variant are rows of \texttt{encoding\_ablation}.}\label{fig:encoding}\end{figure}

\begin{table}[ht]
\centering\small
\begin{tabular}{@{}lllr@{}}\toprule
Strand & Indistinguishable variants & Positions & Size\\\midrule
AS & JC-A1, JC-F1, JC-S1 & 3, 18 & 3\\
AS & JC-A2, JC-F2, JC-S2 & 4, 18 & 3\\
AS & JC-A3, JC-F3, JC-S3 & 3, 4, 18 & 3\\
SS & DO003, DO004 & 17 & 2\\
\bottomrule\end{tabular}
\caption{All non-singleton strand classes, determined by exact complete inputs across every partner. Other strand states are singletons. Full positional names, raw records and equality witnesses remain in the store.}\label{tab:app_collisions}
\end{table}

\subsubsection{Chemical factorization and complete rank argument}\label{app:chemicalfactor}
The raw panel is the full product of fifteen antisense and eleven sense states, including reference W053/W207. Two antisense states are equivalent only when their complete post-mask inputs coincide for every sense partner; sense equivalence is defined symmetrically. This compares $X,A,\mathrm{mask},C$ with shapes, dtypes and exact bytes. All 27,225 ordered state pairs satisfy the stronger biconditional: complete inputs coincide if and only if both strand classes coincide. There are nine antisense and ten sense classes; both reference classes are singletons.

Within each antisense position pattern in Table~\ref{tab:app_collisions}, JC-A denotes the recorded 2-deoxy-2-N,4-C-ethylene-locked chemistry, JC-F the recorded 2,4-carbocyclic-ethylene-bridged locked chemistry, and JC-S the recorded 2,4-carbocyclic locked chemistry. Their raw names differ but the fitted mask removes their distinguishing columns. Sense DO003 (2-aminoethyl) and DO004 (2-guanidinoethyl), both at position 17, also merge. Modification positions and modified versus unmodified indicators can remain distinguishable; an encoded equality is not chemical equivalence. Restoring masked columns to the encoding recovers rank, but not additively. A minimum of three restored chemistry columns recovers full rank; exhaustive search over all smaller sets establishes that lower bound, and exactly six distinct three-column sets attain it. One of these six minimal triples, the lexicographically first, is the two carbocyclic locked-nucleic-acid columns, indices 39 and 40, and the 2-aminoethyl column, index 42. In isolation those three restore 27, 27 and 8 of the lost ranks, 62 in total, against 68 jointly, and no single column restores more than half, so these marginal rank gains do not provide an additive attribution of the joint rank recovery. Rows are in \texttt{column\_localization}, \texttt{column\_localization\_minimal} and \texttt{column\_localization\_summary}.

Index product classes by $(a,b)$, reference $(0,0)$, and arbitrary decoder values $g_{ab}$. The unique interaction coordinates are $t_{ab}=g_{ab}-g_{a0}-g_{0b}+g_{00}$ for nonreference $a,b$. These $(9-1)(10-1)=72$ coordinates are independent: for any vector $v$, set $g_{ab}=v_{ab}$ on nonreference pairs and every reference-row/reference-column value to zero, yielding $t=v$. Every original rectangle maps to one coordinate, and every coordinate occurs because the panel is a full product. Thus $CH=RK$, with $K$ this surjective coordinate map and $R$ its repetition matrix. The columns of $R$ have nonempty disjoint supports, hence are independent, proving rank 72.

For each identical-row group choose representative $q_0$. Each $e_q-e_{q_0}$ for a nonrepresentative row annihilates $CH$. These vectors are independent because each nonrepresentative coordinate appears in only its own vector. There are $140-72=68$ of them, equal to the left-null dimension, so they span it: all restrictions are duplicate-contrast equalities, with no additional general linear dependencies or individually zero row. Rational elimination independently checks the rank and every witness. The forty audited instances share this partition, not forty unrelated representation designs. Restoring masked columns here is an encoding analysis; the refit that feeds restored features to the predictor is the intervention described below in Appendix~\ref{app:maskrestore}.

\subsubsection{Duplicate-row grouping is not equivalent to the rank computation}\label{app:rowgap}
The spanning result above is a property of these panels, not of Proposition~\ref{prop:attain}. Take twelve raw states and three disjoint four-corner contrasts, on states 1--4, 5--8 and 9--12, each with signed coefficients $+1,-1,-1,+1$ in that order, so $\operatorname{rank}(C)=3$. Let an encoding send states $1,\dots,12$ to classes $1,2,3,4,2,4,5,6,1,3,5,6$, so $k=6$. Then
\[CH=\begin{pmatrix}1&-1&-1&1&0&0\\0&1&0&-1&-1&1\\1&0&-1&0&-1&1\end{pmatrix}.\]
All three rows are distinct, nonzero and pairwise nonproportional. Row three is the sum of rows one and two, so $\operatorname{rank}(CH)=2$, and the left null space of $CH$ is spanned by $u=(1,1,-1)^{\top}$: that functional of the three contrasts is zero for every decoder on these six classes. A check for zero, equal or proportional rows can only return functionals supported on one or two rows, at any tolerance, so no such check finds $u$. The instance is constructed and carries no measurements, so it shows a rank deficit and no floor. The measured instance in this work is a radius-2 count fingerprint on a public lipid panel, not a siRNA panel (Appendix~\ref{app:lossy}). On the siRNA panels with a positive floor the duplicate-row result stands as stated: all 68 restrictions on B3 and all 730 on the expanded Bramsen panel of Appendix~\ref{app:breadth} are duplicate-contrast equalities.

\subsubsection{Constructed factorial sweep: what the floor binds}\label{app:floorsynthetic}
The orthogonal split $\lVert y-p\rVert^2=\lVert(\mathrm{Id}-P)y\rVert^2+\lVert Py-p\rVert^2$ holds for any
$p\in S$, and a predictor that consumes only the merged classes has $p=CHg\in S$ by construction, so
its out-of-subspace error equals the analytic floor as algebra. To test whether that equality is a
property of the input restriction or merely of the measurement, the sweep runs two arms per cell. A
protocol whose SHA256 (\texttt{744c6053\ldots}) the fitting stage re-asserts, and which the append-only
execution ledger records before the fits, fixed eight surjective level-merging encoders on two
constructed antisense-by-sense designs ($6\times5$ and $8\times6$), spanning rank deficit 0 to 29 (at
most $19/20$ of $q$), ground truth $y_{ij}=\mu+\alpha_i+\beta_j+\gamma_{ij}$ with $\gamma$ drawn once
per setting and zeroed on the reference row and column, marginal-only training cells so every
interaction is extrapolation, two architectures (a 16-unit MLP and a two-node message-passing model)
and three retained seeds each.

All 48 merged-class fits (96,000 optimizer updates, 0 failures) passed a gate stricter than the
stated $10^{-10}$ floor tolerance: it also requires zero within-class prediction deviation and a
split residual below $10^{-10}$. The largest absolute floor gap is $1.1\times10^{-16}$, the largest
split residual $5.6\times10^{-16}$, within-class prediction deviation is exactly zero in all 48, and
an independent exact-rational floor from the left null space of $CH$ matches the projector floor to
$5.6\times10^{-17}$ in all eight settings. The zero-deficit control has floor exactly zero in the
rational arithmetic, with realized perpendicular residual $q^{-1}\lVert(\mathrm{Id}-P)(y-p)\rVert^2$ below $2\times10^{-31}$; this is distinct from the stored
\texttt{zero\_control} column, which is the no-interaction baseline. Because $\gamma$ is redrawn per
setting, floor magnitudes are not comparable across deficits.

A second pre-registered arm (SHA256 \texttt{7dce9922\ldots}, frozen before any of its fits) changes
only the input: the same 48 fits receive a one-hot over the $a+b$ raw levels instead of the
$a_\ast+b_\ast$ merged classes, so their contrast vectors are not confined to $S$. The registered
prediction was that this arm\textquotesingle s out-of-subspace error would fall below the floor wherever the
deficit is nonzero. \emph{It did not.} The arm is genuinely unconstrained: its within-class
prediction deviation runs $0.55$ to $3.76$, against exactly $0.0$ in every merged-class fit. Its
out-of-subspace error does depart from the floor, but of the 42 nonzero-deficit fits only 8 fall below
it and 34 above, with median ratio $1.33$, maximum $15.97$ and minimum $0.81$; no setting has every
seed below the floor and four have every seed above it. The raw-index arm is not confined to the original contrast subspace. Paired total contrast error on the same settings, models and seeds, meaning the mean square over that design\textquotesingle s $q$ rectangles rather than over its cells, is higher for the raw-index arm in 28 of the 42 nonzero-deficit fits and lower in 14, with median ratio 1.179 over the range 0.577 to 3.903. Lifting the input restriction therefore usually increased total error here and did not systematically improve it: under marginal-only training the added freedom is unconstrained by data. The six zero-deficit pairs are identical, as expected where the two inputs coincide. These are synthetic-only fits on constructed
data; a smaller out-of-subspace error is not better prediction, and they say nothing about siRNA
efficacy. Per-setting rows are in \texttt{floor\_synthetic},
\texttt{floor\_synthetic\_summary}, \texttt{floor\_synthetic\_rawindex} and
\texttt{floor\_synthetic\_rawindex\_outcome}.

\subsubsection{Mask-restoration protocol}\label{app:maskrestore}
The recorded protocol uses the removed oracle floor $F=0.009980$ as a
comparison scale. It labels the ensemble improvement $\Delta$ on all
140 rectangles as near floor ($F/2\leq\Delta\leq F$), near zero
($|\Delta|<F/2$), above the floor ($\Delta>F$), or worse
($\Delta\leq-F/2$). These are descriptive outcome categories, not
bounds on the improvement attainable after refitting. The unmasked arm re-ran the 36-fit
development grid and chose the masked configuration, and ten masked refits reproduced the
historical signatures, epochs and predictions. Training inputs are identical across arms, while 173 of 198 validation rows and 157 of 165 B3 states carry restored columns. Comparing each seed\textquotesingle s selected parameters between the unmasked and masked refits, 5 of the 10 seeds are bit-identical: exactly those whose validation selection chose the same epoch in both arms. The 100 affected rectangles improved by
$+0.000028$ and the 40 unaffected by $-0.000011$, against a zero-interaction control of
0.068205 in both arms; antisense MSE went from 0.049856 to 0.050204 with its 0.018972 floor
removed and sense from 0.003335 to 0.003073. Per-seed rows are in \texttt{mask\_scores} and
\texttt{mask\_paired}.

\subsubsection{Splitting the mask-restoration change}\label{app:maskdecomp}
The refit changes two things at once: the evaluation encoding and the validation-selected checkpoint. This subsection separates them by inference alone. No fit was run and no saved value changed.

Interaction MSE fell from 0.068335 to 0.068318, an improvement of 0.000017 (improvement is masked minus restored, as in Figure~\ref{fig:mask}). At the masked-arm checkpoints, restoring the evaluation encoding gives an improvement of $-0.0000018$, that is, slightly higher error; switching to the restored-arm checkpoints then gives $+0.0000188$. In the other order the checkpoint switch gives $+0.0000212$ and the encoding change $-0.0000042$. So in the ten-seed ensemble, restoring the encoding at fixed parameters slightly worsened error, and the overall improvement comes from the difference in selected checkpoints. Both components vary in sign across seeds, and a bootstrap over seeds puts zero inside the 95\% interval of each.

The four cells cross the masked and restored checkpoints with the masked and restored evaluation encodings. In the order masked/masked, masked/restored, restored/masked and restored/restored they are 0.06833507332376673, 0.06833684869524138, 0.06831383469039375 and 0.06831804925810184. The ensemble rule is the one used throughout: average the 165 endpoint predictions across the ten seeds, then compute one MSE over the 140 rectangles. The per-seed values of all four cells, and the mapping from each seed to its checkpoint file, are in the accompanying package under \path{results/analyses/mask_decomposition/results.json}; they are not reprinted here. The two orders differ by $2.4\times10^{-6}$, because the change is not additive. Every ensemble component is at most $0.213\%$ of the removed floor. Dropping a single seed reverses the sign of the overall difference in two of the ten cases, so the attribution is not stable at this sample size. Five seeds chose the same epoch in both arms; for them the checkpoint component is exactly zero and the whole per-seed change is the encoding change. For the other five seeds no per-epoch parameters were saved at their common epochs, so parameter identity at matching epochs is inferred for those seeds from training histories and validation fingerprints, not shown by hashing. At endpoint level the encoding change at fixed parameters is not negligible: predictions move by up to 0.020 to 0.037 per seed and by 0.0119 in the ensemble, so the near-zero reading applies to interaction MSE and not to the predictions themselves.

\subsection{Model versus training constants under both weightings}\label{app:constantchecks}
\begin{table}[ht]
\centering\small
\begin{tabular}{@{}lrrrrrrr@{}}\toprule
Scen. & Weights & Model & Constant & Difference & Lower & Upper & Grp\\\midrule
primary & rows & GNN & row const & -0.001217 & -0.004579 & 0.020462 & 10\\
primary & rows & GNN & grp const & -0.006226 & -0.010077 & 0.021086 & 10\\
primary & rows & Tree & row const & -0.006784 & -0.013080 & 0.003109 & 10\\
primary & rows & Tree & grp const & -0.011792 & -0.016886 & 0.000805 & 10\\
primary & equal study & GNN & row const & 0.015532 & -0.007321 & 0.038432 & 10\\
primary & equal study & GNN & grp const & 0.015232 & -0.007620 & 0.038511 & 10\\
primary & equal study & Tree & row const & -0.007124 & -0.025018 & 0.008232 & 10\\
primary & equal study & Tree & grp const & -0.007423 & -0.024537 & 0.007632 & 10\\
excluded & rows & GNN & row const & 0.007004 & -0.005002 & 0.040376 & 9\\
excluded & rows & GNN & grp const & 0.009191 & -0.001189 & 0.038475 & 9\\
excluded & rows & Tree & row const & -0.004148 & -0.013327 & 0.009386 & 9\\
excluded & rows & Tree & grp const & -0.001961 & -0.008954 & 0.006106 & 9\\
excluded & equal study & GNN & row const & 0.019636 & -0.001169 & 0.042041 & 9\\
excluded & equal study & GNN & grp const & 0.021490 & 0.003128 & 0.042072 & 9\\
excluded & equal study & Tree & row const & -0.001319 & -0.019699 & 0.014955 & 9\\
excluded & equal study & Tree & grp const & 0.000535 & -0.012901 & 0.012998 & 9\\
\bottomrule\end{tabular}
\caption{Fixed-ensemble error minus each training-fitted constant, with paired conditional 95\% percentile intervals from resampling whole study-linked components under both weightings at the existing fixed analysis seed over 10,000 draws, predictions and selection held fixed. Constants are training-fitted and never an evaluation mean. The tree has lower observed MSE than both constants under row weighting, but not against the group-mean predictor under equal-component weighting on the excluded cohort; all four source-excluded tree-versus-constant intervals include zero. Ten and nine components respectively: conditional, few-component intervals whose endpoints are not subtracted.}\label{tab:app_intervals}
\end{table}
Leave-one-study-component-out rescoring omits one evaluation component and renormalises the stated weights, leaving the fitted models and their training-partition constants unchanged. Across 152 omissions the descriptive ordering changes in 8 cases and no retrospective denominator was near zero. These are influence diagnostics, not new independent studies.

\subsection{Practical size of the bounded perturbation}
\begin{table}[ht]
\centering\small
\begin{tabular}{@{}lrrrrrrr@{}}\toprule
Method & Seed & Scale & Changed & Mean ovl. & Min ovl. & Med. disp. & Max disp.\\\midrule
GNN & 1103 & -0.250000 & 0 & 1.000000 & 1.000000 & 0.000000 & 0.000000\\
GNN & 1103 & 0.250000 & 0 & 1.000000 & 1.000000 & 0.000000 & 0.000000\\
GNN & 2207 & -0.250000 & 0 & 1.000000 & 1.000000 & 0.000000 & 0.000000\\
GNN & 2207 & 0.250000 & 0 & 1.000000 & 1.000000 & 0.000000 & 0.000000\\
GNN & 3301 & -0.250000 & 0 & 1.000000 & 1.000000 & 0.000000 & 0.000000\\
GNN & 3301 & 0.250000 & 0 & 1.000000 & 1.000000 & 0.000000 & 0.000000\\
original\_gnn & 1103 & -0.250000 & 4 & 0.952941 & 0.800000 & 0.009615 & 0.093750\\
original\_gnn & 1103 & 0.250000 & 3 & 0.964706 & 0.800000 & 0.009615 & 0.089888\\
original\_gnn & 2207 & -0.250000 & 0 & 1.000000 & 1.000000 & 0.004808 & 0.077778\\
original\_gnn & 2207 & 0.250000 & 4 & 0.952941 & 0.800000 & 0.004808 & 0.093750\\
original\_gnn & 3301 & -0.250000 & 4 & 0.952941 & 0.800000 & 0.009615 & 0.067416\\
original\_gnn & 3301 & 0.250000 & 0 & 1.000000 & 1.000000 & 0.009615 & 0.062500\\
\bottomrule\end{tabular}
\caption{Within each of the seventeen exact APP assay pools, top-five selection mass under the existing fractional boundary-tie rule, and average-rank displacement normalised by $n-1$. The pools overlap and come from one patent family, so these are not independent trials.}\label{tab:app_selection}
\end{table}
\begin{table}[ht]
\centering\small
\begin{tabular}{@{}lrrrrr@{}}\toprule
Arm & Seed & Scale & Max coordinate & Max $L_2$ & Max relative $L_2$\\\midrule
R1 & 1103 & -0.250000 & 0.000000 & 0.000000 & 0.000000\\
R1 & 1103 & 0.250000 & 0.000000 & 0.000000 & 0.000000\\
R1 & 2207 & -0.250000 & 0.000000 & 0.000000 & 0.000000\\
R1 & 2207 & 0.250000 & 0.000000 & 0.000000 & 0.000000\\
R1 & 3301 & -0.250000 & 0.000000 & 0.000000 & 0.000000\\
R1 & 3301 & 0.250000 & 0.000000 & 0.000000 & 0.000000\\
R0 & 1103 & -0.250000 & 0.005795 & 0.071476 & 0.240187\\
R0 & 1103 & 0.250000 & 0.005870 & 0.088127 & 0.336519\\
R0 & 2207 & -0.250000 & 0.004085 & 0.055905 & 0.212644\\
R0 & 2207 & 0.250000 & 0.005719 & 0.075019 & 0.294461\\
R0 & 3301 & -0.250000 & 0.003599 & 0.054309 & 0.288555\\
R0 & 3301 & 0.250000 & 0.002808 & 0.019713 & 0.108521\\
\bottomrule\end{tabular}
\caption{Raw input-gradient changes on the same sixteen fixed APP queries. Each statistic maximizes across those queries; the coordinate statistic also maximizes across coordinates. Relative change divides each query gradient-change norm by its baseline norm, with threshold $10^{-8}$; no near-zero references occur. Other graph inputs are fixed. These are local encoded sensitivities, not feasible chemical interventions, and R1 is the expected structural null control.}\label{tab:app_gradient}
\end{table}

\subsection{Why the published ENsiRNA-mod route could not be evaluated}
The official repository documents two routes: a prebuilt-image route and an \path{ENsiRNA-mod/easy_run.py} entrypoint. The image route was not available to us and was not attempted. The documented Linux route additionally requires a licensed Rosetta installation and pre-folded PDB geometry, and the released inference code shells out to Rosetta binaries and consumes atom identities and RNA sequence embeddings. The training partition of the released checkpoint is not declared, so overlap with Bramsen is unknown and no held-out predictive claim is made. The earlier pinned partial-branch chemistry audit is unchanged.

\begin{table}[ht]
\centering\small
\begin{tabular}{@{}lrrrr@{}}\toprule
Protocol & GNN & No-message & Tree & CNN\\\midrule
v1 original & 0.521 & 0.497 & 0.565 & 0.453\\
v2 factorial & 0.538 & 0.498 & 0.574 & 0.530\\
v2 deployment & 0.397 & 0.367 & 0.631 & 0.510\\
v3 group selected & 0.534 & 0.577 & 0.445 & 0.583\\
Source-excluded & 0.400 & 0.374 & 0.696 & 0.421\\
Primary-assay & 0.565 & 0.588 & 0.514 & 0.563\\
\bottomrule\end{tabular}
\caption{Common tie-aware APP mean top-five percentile across the identical seventeen assay pools; larger is preferred. The pools overlap and come from one patent family, so they are not seventeen independent experiments.}\label{tab:ranking}
\end{table}
\begin{table}[ht]
\centering\small
\begin{tabular}{@{}lrrrrr@{}}\toprule
Method & Row MSE & Comp. MSE & $r$ & Pred. SD & Sign\\\midrule
Activity GNN & 0.2351 & 0.3560 & 0.083 & 0.036 & 111/137\\
Pair GNN & 0.2221 & 0.3211 & 0.248 & 0.038 & 111/137\\
Pair no-message & 0.2220 & 0.3196 & 0.228 & 0.043 & 111/137\\
Pair tree & 0.2438 & 0.3101 & 0.120 & 0.167 & 109/137\\
Pair ridge & 0.2193 & 0.2888 & 0.226 & 0.148 & 105/137\\
Zero & 0.2808 & 0.4342 & -- & 0.000 & 0/137\\
Training mean & 0.2638 & 0.3336 & -0.076 & 0.064 & 111/137\\
\bottomrule\end{tabular}
\caption{Unchanged held-out measured B2 errors. Comp.\ gives each of twelve sequence components equal mass. Sign counts use 137 non-ties under the operational 0.02 band; the majority control also reaches 111/137.}\label{tab:pairranking}
\end{table}

\paragraph{Exact algebra records.} The exact matrix rows, all left-null witnesses and full-precision projection components remain in \texttt{chemical\_exact\_matrices}, \texttt{chemical\_left\_null\_witnesses} and \texttt{contrast\_projection}. Section~\ref{app:contrastrank} gives the complete proof; the main text gives the numerical error decomposition.

\paragraph{Complete perturbation records.}
Per-method, per-scale perturbation outcomes -- maximum training change, APP change, rank range, gradient change and B3 contrast change -- are rows of \texttt{perturbation\_summary} in the canonical store.

\subsection{Published preprocessing: exact scope}
ENsiRNA-mod commit \path{028824341635903f3c661f5d1cc737de106493d5} uses radius-two 512-bit Morgan chemistry fingerprints. Its partial branch admits 140/140 B3 rectangles with no cancellation; complete inference also requires atom IDs, geometry and RNA-FM. MEG-mod commit \path{c335a4c69d56ef73754677da8bfe627c8112b350} supplies the UniMol dictionary: 117/140 rectangles pass its modification branch without cancellation, 23 have unsupported annotations. Undeclared strand-token construction/dimensions block its complete forward pass. Unknown inputs are never zero-filled.

Eight conformance pairs per method include an unchanged control, sugar/phosphate edits and unsupported names. ENsiRNA has three partial collisions, three distinct pairs, one control and one unsupported pair; MEG has two, two, one and three. ENsiRNA fingerprints coincide for locked/altritol nucleic acid, 2-fluoro/fluoroarabino and uracil/pseudouridine; MEG embeddings coincide for the first and third. Other branches can distinguish them: these are neither chemical-equivalence claims nor whole-predictor certificates. ModMapper exposes a PHP form but no local parser/ontology release, blocking all 140 rectangles and eight conformance inputs; no job is submitted. Raw names, downloaded bytes, hashes and witnesses remain in the store.
\subsection{Historical fit accounting}\label{app:fitledger}
Preserved campaigns contain 37, 210, $2{,}312$ and 792 unique siRNA fit records ($3{,}351$ total and $1{,}934{,}535$ optimizer updates); the separate twenty-update profile in the first campaign is not an efficacy fit. This paper adds two siRNA intervention series. The mask restoration of Appendix~\ref{app:maskrestore} contributes 56 fits ($21{,}513$ optimizer updates), and the source-exposed training-support arm of Appendix~\ref{app:supportdecouple} contributes 20 fits ($8{,}624$ optimizer updates). Together the two series are 76 fits and $30{,}137$ optimizer updates, with no failed fit recorded. These totals cover the two siRNA intervention series only; the splicing and synthetic work is counted next. The MAVE-NN support intervention took twelve attempts, six of which retain per-seed results (Appendix~\ref{app:support}), and its reported $62{,}424$ updates cover the retained invocation, not a complete total across both invocations. The synthetic work on constructed factorial designs is 96 fits in total, 48 in each of the merged-class and raw-index arms, each arm with $96{,}000$ optimizer updates and no failures; no siRNA checkpoint was read or modified. The raw-index arm was not covered by the pipeline's own execution record, and its counts come from that run's own record. Historical fits, new attempts, retained outcomes and failed invocations are distinct accounting units, and the failure classification is in Appendix~\ref{app:reproduction}. Source, training/selection, response and common-row changes remain identified in \texttt{protocol\_comparisons} and \texttt{protocol\_matched\_changes}; prior ranking corrections and source discrepancies are retained above.
\subsection{Canonical execution and evidence mapping}
Set \texttt{PYTHONDONTWRITEBYTECODE=\allowbreak 1} and run \texttt{python} \path{analysis/structural_attribution_audit/audit.py} \texttt{all}, then replace \texttt{all} with \texttt{verify-resume} to check unchanged reuse. The source-exposed training-support arm of Appendix~\ref{app:supportdecouple} is not a stage of this script: it ran from its own frozen runner and keeps its own protocol, per-seed and outcome records. The same directory holds the manifest, SQLite results and single report with exact historical dependency paths. No scientific input is copied or modified. Data that cannot be redistributed is replaced by its extraction script and a checksum manifest. Styles and builds are shipped alongside the manuscript sources rather than inside them, and the source archive compiles on its own. Recorded diagnostic and build costs exclude additional nonzero administrative work.

\subsection{Independent-predictor fidelity and coverage}\label{app:independentattempt}
Two candidates were allowed: ENsiRNA-mod at commit \texttt{\seqsplit{028824341635903f3c661f5d1cc737de106493d5}} and MEG-mod at \texttt{\seqsplit{c335a4c69d56ef73754677da8bfe627c8112b350}}, both still at those commits. Each recorded failure was reclassified against a compatible environment supplying Torch 2.11 with CUDA, PyTorch Geometric, ViennaRNA 2.6.4 and \texttt{rna-fm} with cached weights. The Python dependencies listed above were available in the compatible environment. The remaining obstacles included missing native structure-generation software and unresolved release semantics.

MEG-mod's released \texttt{BAN\_graph.py} contains no import statements at all, which is why its documented entry point stops at line 10 with \texttt{NameError: torch}: a released-source defect, not a missing package. Restoring the omitted block is additive, verified by deleting it again and recovering the released bytes exactly. The repaired module then executes, and its forward's only remaining unbound globals are \texttt{s\_tokens}, \texttt{a\_tokens}, \texttt{max\_L\_mod\_sense} and \texttt{max\_L\_mod\_anti}, which build the padded modification-token tensors consumed by \texttt{bcn\_mod}, one of the two halves entering the output block. The released checkpoint holds sixteen \texttt{bcn\_mod} tensors, so that branch is required under the strict load its own prediction script performs, and their shapes fix \texttt{embed\_dim}${=}1536$, an admissible binding. No shape states what a modification token contains, in what order, or how many there may be; that semantics appears in no released file, so the route stopped rather than guessed. The release also omits the \texttt{requirements.txt} its README directs users to install.

The earlier record that ENsiRNA-mod has no local checkpoint is corrected: five released checkpoints ship inside the repository at 9,523,522 bytes each, and \texttt{checkpoint\_1.ckpt} deserialises unrepaired into \texttt{model.mask\_model.RNAmaskModel} with 2,365,808 parameters exposing the documented \texttt{test} interface. Exactly one input is missing. \texttt{X}, the per-residue coordinates, is that interface's first argument and is consumed twice, once to build the nine-nearest-neighbour graph and once as coordinates for the equivariant network, so without real structures there is no graph and substituted coordinates would be invented geometry. The release obtains them only from Rosetta \texttt{rna\_denovo}, absent here and distributed behind a licence-and-registration page; the linked pre-folded archive covers the authors' duplexes, not these cohorts. The image route was not retried: no image runtime was available to us, and the published image is in any case larger than the free space on the working filesystem.

Zero of 156 B2 pairs (312 endpoints, 12 sequence components, 26 distinct base duplexes), zero of 165 B3 endpoints and zero of 140 B3 contrasts therefore received a complete external prediction, and 477 endpoint records would have needed native inputs. Support was fixed from native input requirements, never from observed error, and nothing was zero-filled: these are unassessed cases, not model errors. Supervised efficacy-training and selection overlap stays undocumented for both checkpoints, so neither could have supported a held-out claim even had it run. The frozen cap was two candidates, 1,200 seconds of setup, two CPU threads, no GPU and at most six fitting attempts only after fidelity and a runtime estimate; with no executable forward and no geometry the fitting budget stayed zero and all six attempts are unspent. No optimizer was constructed, no image pulled and no repository file added. Pins, diffs, provenance, checkpoint identities, executed commands, denominators and a frozen but never executed external evaluation rule are in the \texttt{recovery} tables of the canonical store. Independent predictive validation remains incomplete.

\paragraph{Rank and floor of the two released chemistry branches.} The rank and floor computation of Appendix~\ref{app:contrastrank} was applied to each release\textquotesingle s own chemistry branch at the pinned commit, with zero fits, zero optimizer updates and zero predictor forward passes. Classes are formed by byte equality of that branch\textquotesingle s encoder output over the endpoint states, the rule that also defines our 90-class partition. ENsiRNA-mod\textquotesingle s radius-two 512-bit Morgan branch keeps all 165 states distinct, so $\operatorname{rank}(CH)=\operatorname{rank}(C)=140$ on the 140 rectangles; the recorded contrast vector lies in the attainable span and the floor is exactly zero. MEG-mod\textquotesingle s UniMol lookup keeps its 140 encodable states distinct and gives $\operatorname{rank}(CH)=117$ on the 117 rectangles whose four corners it encodes, with floor exactly zero. The 25 states carrying 2-Fluoro were rejected rather than zero-filled, which excludes 23 of the 140 rectangles. These are properties of chemistry branches and not certificates about the complete predictors, whose later stages could still merge states. The comparison sets a masked pipeline against unmasked chemistry encoders rather than one predictor against another. The conformance probes of Appendix~\ref{app:published} show both branches giving equal outputs to chemically distinct modifications, so neither branch is injective beyond these panel states. Per-pipeline classes, ranks, null dimensions, floors and exclusion rules are in \texttt{published\_rank}, \texttt{published\_rank\_partitions} and \texttt{published\_rank\_scope}.

\subsection{Does the audit identify unreliable predictions?}\label{sec:utility}
Two fixed scores test usefulness: training-support absence among chemical entries relevant to an input or edit (A), and maximum ensemble-effect change over the two existing bounded probes (B); lower scores retain supposedly more reliable cases. Three fixed seeds define each predictor, its disagreement score and its probes, and R0 deployment diagnostics are never mixed with source-held-out R1 errors. Novelty is nearest-training positional-chemistry Jaccard distance, coverage is 20, 40, 60, 80 and 100\% with fractional ties and equal sequence-component base weights, and a constant score fractionally retains every case, supplying no ranking information (Appendix~\ref{app:utility}).
\begin{table}[ht]\centering\small
\begin{tabular}{@{}lrrr@{}}\toprule
Retention score & Model MSE & Zero MSE & Excess MSE\\\midrule
Support absence & 0.434741 & 0.434166 & +0.000575\\
Probe sensitivity & 0.356843 & 0.356330 & +0.000513\\
Ensemble SD & 0.363586 & 0.362827 & +0.000759\\
Chemical novelty & 0.434741 & 0.434166 & +0.000575\\
Random & 0.451729 & 0.451135 & +0.000594\\
\bottomrule\end{tabular}
\caption{R0 B2 at 60\% observation mass: 93.6 fractional pairs out of 156, twelve sequence components. Each zero control uses identical retained membership and weights. Random averages 200 frozen selections. These three-seed activity-only predictions are distinct from historical pair-supervised estimates.}\label{tab:utility}
\end{table}
The primary comparison, B minus ensemble disagreement in excess-over-zero MSE at 60\% coverage, is -0.000246 with paired conditional 95\% interval [-0.001059, +0.000561]. It does not establish a utility advantage. Lower raw pair error under B mostly follows lower zero-effect error on the retained pairs, which does not demonstrate learned sequence-specific chemistry. A and novelty are constant across B2, B is constant for R1, and on source-held-out B3 it remains zero despite poor interaction recovery. Secondary APP and within-assay comparisons establish no general diagnostic superiority.
\begin{figure}[ht]\centering\includegraphics[width=\textwidth]{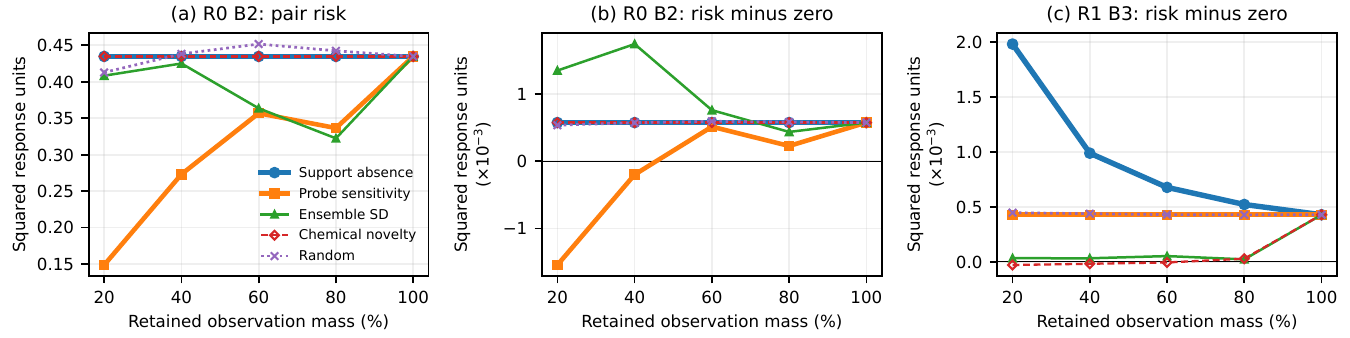}
\caption{Retrospective utility at all frozen coverage levels. Raw and excess-over-zero risk differ because selection changes the measured-effect distribution. Panels (a,b) use the identical R0 deployment ensemble on B2; (c) uses the separate source-held-out R1 ensemble on B3. B3 has one background; these curves are descriptive, and the frozen primary comparison is at 60\%. Constant scores coincide: support absence with chemical novelty on B2 (a,b), probe sensitivity with random retention for R1 (c).}\label{fig:utility}
\end{figure}

\subsection{Diagnostic utility: complete protocol and estimands}\label{app:utility}
All cohorts were previously inspected: this is a frozen retrospective analysis, not preregistration. R0 and R1 deployment ensembles use seeds 1103, 2207 and 3301 with identical members for predictions, disagreement and probes; APP's 1,839 endpoints and B2's 156 bundles were excluded from their supervised training and selection, and one patent family with existing sequence components does not certify external generalization. Deployment B3 is source-exposed and descriptive, while a separate R1 outer0 ensemble excludes Bramsen and scores one previously inspected B3 background. Checkpoints, memberships, masks, response conventions and evaluation identities are recorded jointly; no deployment score is paired with an outer0 error.

Let $p_{ij}$ denote seed $j$'s endpoint prediction, $J=3$, $q_{cj}=\sum_i a_{ci}p_{ij}$ and $\bar q_c=J^{-1}\sum_jq_{cj}$. Endpoints use one coefficient one, B2 changed minus reference, and B3 its marginal and four-corner coefficients; the target is the same signed combination of measured activities, response values are unchanged and no calibration reads evaluation labels.

\paragraph{Support absence (A).} Use the actual schema's half-open chemical columns $[39,84)$ and original adjacency. For an endpoint the relevant set is its active chemistry occurrences and directed edges. For a contrast it is the union of active entries that differ across corners. If $r_{ce}$ indicates this set and $u_e$ is one when the corresponding node-feature or relation type never occurs in the exact training partition, define $A_c=\sum_e r_{ce}u_e/\sum_e r_{ce}$, or zero for an empty set. Entries count once; this is an occurrence score, not a learned effect size, and all evaluated names are in the fixed schema. A preliminary offset $[36,81)$ was corrected against the manifest and encoders; superseded scores and protocol remain preserved, and the correction was not chosen for better outcomes.

\paragraph{Probe sensitivity (B).} The existing RNG 20260914 generates fixed Rademacher directions at each block's original RMS. With $s\in\{-0.25,+0.25\}$ define $B_c=\max_s|J^{-1}\sum_jq_{cj}^{(s)}-\bar q_c|$. This is a maximum over two tested ensemble probes, not a certified worst case. R0 perturbs only absent independent transforms and R1 only false-gated residuals, never shared bases (Appendix~\ref{app:trainingindependence}). Existing deployment zero/probe checks were reused, and three outer0 checkpoints received full-training checks at zero and both scales with bitwise-invariant predictions, unchanged B3 outputs, exactly restored parameters and retained checkpoint hashes. No perturbed checkpoint is saved.

\paragraph{Comparators.} Disagreement is $[J^{-1}\sum_j(q_{cj}-\bar q_c)^2]^{1/2}$ using actual per-seed contrasts, preserving endpoint covariance. Novelty is nearest-training Jaccard distance between sets of (strand-slot, pre-mask chemistry-column) pairs with empty/empty distance zero, a contrast taking its largest endpoint distance, and the same training partition defines novelty and support. Higher always means less reliable; no reversal, combination or threshold is optimized. Random retention averages 200 fixed uniform-score vectors per model/cohort block at RNG 20260916.

\paragraph{Retention and risk.} For $n$ objects and $f\in\{.2,.4,.6,.8,1\}$ let $t$ be the smallest score threshold reaching $fn$ objects. Set $r_i=1$ below $t$, zero above it, and $(fn-n_{<t})/n_{=t}$ at equality. Thus $\sum_i r_i=fn$ without ID or outcome tie-breaking and a constant score gives $r_i=f$ everywhere. Base weights $w_i$ are inverse sequence-component counts (equal object weights for B3); renormalize $v_i=w_ir_i/\sum_kw_kr_k$ and compute $R=\sum_iv_i(\bar q_i-y_i)^2$, $R_0=\sum_iv_i(b_i-y_i)^2$ and $E=R-R_0$. Endpoint $b_i$ averages the actual checkpoint training means; contrast $b_i=0$. Report fractional count, positive-membership count, retained base-weight mass and effective sample size $(\sum_iw_ir_i)^2/\sum_i(w_ir_i)^2$. A smaller $R$ may just retain smaller measured effects, hence the matched $E$ comparison.

The primary estimand is $E_B(.6)-E_{\rm SD}(.6)$ for R0 B2, negative favoring B; all other coverages, scores and comparators are secondary. Within-assay checks require at least ten objects and the same fractional rule, and overlapping pools are descriptive, not independent replication. No candidate-selection improvement is inferred from global rank movement.

\paragraph{Uncertainty.} Draw 2,000 multinomial resamples of observed sequence components at RNG 20260916, applying identical multiplicities to the retained numerator and denominator for each score. Fits and original selection memberships stay fixed; recompute paired risk and excess-risk differences and take the 2.5th/97.5th percentiles. No-retained-mass resamples are undefined and counted, not zero-filled (up to four in a secondary comparison); all primary resamples are defined, and random averages its defined frozen selections within each resample. Twelve B2 components from one family limit interpretation; no B3 interval or rectangle bootstrap is used.

\paragraph{Secondary checks on endpoints and B3.}
The conclusion also depends on its measured target. On source-held-out B3 novelty selects contrasts with much smaller measured squared effects, its lower raw MSE tracking its lower zero-effect risk (rows of \texttt{utility\_curves}), and constant B retains every interaction at fractional mass, so it cannot identify poor recovery. These are 140 dependent contrasts on one background, not 140 independent validations.

On APP endpoints the probe-minus-disagreement excess-risk difference is +0.006818, conditional 95\% interval [-0.008608, +0.023869]. Within the 17 overlapping APP assay pools, the probe score has lower excess risk in 7 and higher in 10; on the 6 B2 assay contexts it is lower in 5 and higher in 1. These descriptive counts preserve context without creating independent biological replications; they do not overturn the primary null, and no candidate-selection improvement follows.

\paragraph{Results and limits.} The primary interval includes zero. B2 support absence and novelty are constant and R1 probe sensitivity is zero even on the poorly predicted source-held-out B3 contrasts; the historical ten-seed B3 result is preserved separately from this three-seed ensemble. All predictions, controls, seed contrasts, risks, denominators, paired intervals and within-assay outcomes are in \texttt{utility\_observations}, \texttt{utility\_curves}, \texttt{utility\_comparisons} and \texttt{utility\_within\_assay}. The primary null does not prove every diagnostic useless; it rejects a claimed advantage unsupported by this test.
\begin{table}[ht]
\centering\small
\begin{tabular}{@{}lrrrrrr@{}}\toprule
Coverage & Pair mass & Probe MSE & Its zero MSE & SD MSE & $\Delta E$ & 95\% interval\\\midrule
20\% & 31.2 & 0.148697 & 0.150237 & 0.408344 & -0.002887 & [-0.004963,-0.000467]\\
40\% & 62.4 & 0.273094 & 0.273292 & 0.425092 & -0.001938 & [-0.003812,-0.000147]\\
60\% & 93.6 & 0.356843 & 0.356330 & 0.363586 & -0.000246 & [-0.001059,+0.000561]\\
80\% & 124.8 & 0.336510 & 0.336282 & 0.322389 & -0.000208 & [-0.000668,+0.000103]\\
100\% & 156.0 & 0.434741 & 0.434166 & 0.434741 & +0.000000 & [+0.000000,+0.000000]\\
\bottomrule\end{tabular}
\caption{R0 B2 across all frozen coverage levels. Delta E is probe minus disagreement in excess-over-identical-subset-zero risk; 60 percent is primary. Equal-component weights are renormalized after fractional retention. Different methods retain different zero risks.}\label{tab:utilitycoverage}
\end{table}

\subsection{Complete published predictor: MAVE-NN measured effects}\label{app:mavenn}
\paragraph{Native reference comparison and provenance.}
The intended split holds out sequences within one assay and gene context;
it is not a new-gene-context test or independent biological validation.
The native evaluation gave $I_{\mathrm{var}}=0.305865$ and
$I_{\mathrm{pred}}=0.342436$, versus tutorial values $0.306635$ and
$0.357971$, respectively. The differences are $-0.000770$ and
$-0.015535$. The tutorial reports $0.357971\pm0.013754$ for
$I_{\mathrm{pred}}$,\footnote{\url{https://mavenn.readthedocs.io/en/v1.1.3/tutorials/3_splicing_mpra_multiple_gpmaps.html}}
so the present estimate is not inside that displayed
interval. This observation alone does not establish a statistically
significant discrepancy. No combined-uncertainty criterion or
prediction-level numerical tolerance is established here; this is a
reference comparison, not bitwise reproduction.
Identity enumeration preceded scoring: 2083 rectangles whose endpoints all belong to the released test split over 28 position pairs and 1853 backgrounds, and 12486 single substitutions, with position 4 fixed G and position 5 restricted to C and U, so no globally unmeasured substitution is scored.

Implementation, data and checkpoint are pinned: mavenn 1.1.3, TensorFlow 2.21.0, Keras 3.10.0 (newer Keras dereferences an argument specification that mavenn\textquotesingle s error-handling decorator hides) and Python 3.13.12, outside the repository. The released weights are \texttt{mpsa\_ge\_pairwise.weights.h5} (SHA256 \texttt{90e955f8\ldots}) with metadata \texttt{.pickle} (\texttt{d41a069a\ldots}); the data are \texttt{mpsa\_data.csv.gz} (\texttt{df125142\ldots}, 30,483 rows) and \texttt{mpsa\_replicate\_data.csv.gz} (\texttt{c18a8247\ldots}, 30,697 rows). The first hash equals the independently pinned GitHub copy used for the input-support calculation. Settings are the released ones: pairwise G-P map, GE regression, SkewedT noise, heteroskedasticity order two, 50 hidden nodes, $\theta$ and $\eta$ regularisation 0.1. The checkpoint metadata record $N=24{,}405$ and aggregate training
statistics, but not the full sequence memberships. The author training
notebook removes released test rows before setting data and saves the
corresponding model name.
This supports the intended split; count
equality alone does not establish the exact source--data--checkpoint
provenance or exclude test use from every model-selection step. The released response values use a consensus-sequence assay normalizer
before the random split.
This shared assay-scale normalization is
distinct from model-fitted preprocessing. Its common additive offset
on the log scale cancels from within-library contrasts whose
coefficients sum to zero, but remains relevant to endpoint-scale
interpretation.

Predictions use the complete native observable map \texttt{x\_to\_yhat}; the latent $\phi$ is never substituted for a measurement. An additive latent map can still produce nonadditive observed contrasts through the nonlinear measurement map, so observed nonadditivity is not evidence of a coding error. No calibration transform is fitted on test or replicate labels.

A rectangle is eligible only when all four sequences that differ at exactly two declared positions are released test rows; a single substitution needs both of its sequences there. Orientation is lexicographic in both position indices and both state pairs, so identities are label-independent, and duplicates are removed. Weights are three nested equal levels: each of the 28 eligible position pairs carries equal total weight, each of its backgrounds equal weight within the pair, and each eligible rectangle equal weight within its background; the weights sum to one. Single substitutions use the analogous rule over positions, backgrounds and substitution pairs. Enumeration precedes scoring, and 2083 of 2083 enumerated rectangles are supported, with zero unsupported and zero unresolved rectangle identities, because eligibility is decided by membership of the released test split alone.

\paragraph{Encoding rank and floor for this pipeline.} The diagnostic of Appendix~\ref{app:contrastrank} was applied to the splicing encoder on its own terms. The complete consumed input is \texttt{x\_ohe} as produced by MAVE-NN's \texttt{x\_to\_stats}, float32, 36 coordinates per sequence, and classes are formed by byte equality of that tensor, the same rule used for the siRNA encoder. All 3,696 endpoint states of the primary rectangle population remain distinct, so $H$ restricts to the identity on them and $\operatorname{rank}(CH)=\operatorname{rank}(C)=1{,}986$ against $q=2{,}083$ contrasts. The 97 forced-zero directions are therefore dependencies of the rectangle enumeration rather than of the encoding. Because $\operatorname{col}(CH)\subseteq\operatorname{col}(C)$ and the two ranks are equal, the column spaces coincide; the recorded contrast vector is $C$ applied to the recorded endpoint means and so lies in $\operatorname{col}(CH)$, making the floor exactly zero. The computed value is $6.2\times10^{-30}$, at float rounding. Library~2 gives 3,521 states, $\operatorname{rank}(CH)=\operatorname{rank}(C)=1{,}856$ against $q=1{,}947$, 91 design dependencies and floor zero on the same argument. Ranks were obtained by exact sparse rational elimination and cross-checked against GF($p$) elimination at two primes and against the SVD rank. This is the second instance of the diagnostic and the one in which it does not fire; per-pipeline rows, certificates and encoding descriptions are in \texttt{rank\_floor\_generalization}.

The additive comparator checkpoint was not available in the installed example-model directory used for this evaluation. Historical additive assets exist elsewhere in the author repository, but a compatible evaluation of those assets was not established here. The additive model was not refitted; the authors' tutorial values ($I_{\mathrm{var}}=0.216834$, $I_{\mathrm{pred}}=0.219266$) are external context, not measurements from our evaluation.

The replicate library is matched by exact sequence only, never by its own set column: 5884 of 6078 test sequences recur, 1947 of 2083 rectangles and 11937 of 12486 single substitutions survive, and the losses are reported rather than back-filled. It is measured assay replication of the same library design, not independent biological replication of each sequence or rectangle. No replicate-agreement threshold was chosen after seeing model errors and replicate agreement is not treated as an exact noise ceiling. Full identities, weights and per-contrast values are in \texttt{mavenn\_fidelity}, \texttt{mavenn\_coverage}, \texttt{mavenn\_endpoints}, \texttt{mavenn\_contrasts}, \texttt{mavenn\_replicate} and \texttt{mavenn\_comparator}.

\subsection{Encoding breadth audit: further measured panels}\label{app:breadth}
A separate audit applied the rank and floor calculation of Appendix~\ref{app:contrastrank} beyond the primary panel. Its candidate routes, eligibility rules and size caps were frozen before any new rank, floor or prediction error was computed (protocol SHA-256 \texttt{8a3c05d2\ldots}), and every selected case is reported, including zeros. It constructed no model or optimizer and ran no checkpoint inference. Classes use the byte-equality rule of Appendix~\ref{app:inputcollision}: for the GNN, all four forward arguments $X,A,\mathrm{mask},C$ after the historical training-support preprocessing, with no audit-chosen threshold; for MAVE-NN, the released one-hot input. The restriction added by an encoding is counted as $d_{\mathrm{enc}}=\operatorname{rank}(C)-\operatorname{rank}(CH)$, because $q-\operatorname{rank}(CH)$ also counts dependencies already present in the contrast design. Ranks are exact, by rational elimination or, where $C$ has full row rank, a signed identity minor. Floors use equal weights $1/q$ on each panel\textquotesingle s recorded response scale. Every value in Table~\ref{tab:breadth} was recomputed independently from the saved $C$, $H$ and measured contrasts.
\begin{table}[ht]
\centering\footnotesize\setlength{\tabcolsep}{3.5pt}
\begin{tabular}{@{}lrrrrrrrrr@{}}\toprule
Panel & $m$ & $k$ & $q$ & $\operatorname{rank}C$ & $\operatorname{rank}CH$ & $d_{\mathrm{enc}}$ & Floor & Energy & Share\\\midrule
B3 (reproduced) & 165 & 90 & 140 & 140 & 72 & 68 & 0.009980 & 0.068205 & 14.63\%\\
Bramsen, expanded & 1932 & 1183 & 1845 & 1845 & 1115 & 730 & 0.007565 & 0.055180 & 13.71\%\\
\quad new contrasts only & 1792 & 1111 & 1705 & 1705 & 1043 & 662 & 0.007367 & 0.054111 & 13.61\%\\
MAVE-NN splicing & 3696 & 3696 & 2083 & 1986 & 1986 & 0 & 0 & 0.412839 & 0\%\\
GB1 binding & 17689 & 17689 & 15856 & 15856 & 15856 & 0 & 0 & 0.549741 & 0\%\\
AGILE lipid, tail$\times$tail & 1100 & 920 & 800 & 800 & 740 & 60 & 0.315003 & 12.425324 & 2.54\%\\
\quad head$\times$tail pair & 1100 & 920 & 1026 & 1026 & 855 & 171 & 2.159133 & 21.666001 & 9.97\%\\
\bottomrule\end{tabular}
\caption{Encoding rank and empirical floor on further measured panels, equal weights $1/q$. $d_{\mathrm{enc}}$ counts restrictions added by the encoding; $q-\operatorname{rank}C$ counts design dependencies (97 for splicing, none elsewhere). Energy is $q^{-1}\lVert\widetilde I\rVert^2$ and Share is the floor as a percentage of it, a different denominator from the fitted-error shares of Table~\ref{tab:floor}. Zero floors are exact: $H$ is injective, so $\operatorname{rank}(CH)=\operatorname{rank}(C)$ and the recorded contrasts lie in $\operatorname{col}(CH)$. The new-contrasts row is a follow-up check run after the frozen audit. The last two rows are the lipid panel of Appendix~\ref{app:lossy} under its two contrast families, also computed after that audit and under a different encoder, a radius-two count fingerprint. In the first of those families no restriction is a duplicate-row relation; in the second every restriction is. Response scales differ between panels, so floors are compared within a panel, not across panels.}\label{tab:breadth}
\end{table}

\paragraph{Bramsen screen, expanded.} The expanded panel uses the primary spreadsheet of \citet{bramsen2009}, the same assay context (HeLa eGFP, 10\,nM, 72\,h), response ($1-$relative eGFP, unclipped) and W053/W207 anchor as B3. Of 2,208 parsed rows, 2,158 are unique finite measured pairs once empty rows and duplicated identities are removed. Excluding 181 unresolved or missing strand identities and 45 pairs that fail the original pairing-admission rule leaves 1,932 endpoints, and 1,845 anchored rectangles have all four corners measured (221 candidates lack one). All 165 original states are retained, and their raw tensors, labels and 90-class partition reproduce. The 1,767 added states change strand sequence or length, so only the original 165 satisfy the chemistry-only restriction of B3: the expanded contrasts measure anchored strand replacement, a broader estimand. No row of $CH$ is zero, so the 730 restrictions are joint, as in B3. After the frozen audit, a follow-up check removed the 140 original contrasts, each matched exactly once by its signed endpoint identities; the remaining 1,705 contrasts alone lose 662 directions with floor 0.007367, so the positive floor is not carried by the original panel. All these contrasts share the assay, controls and reference state of B3. They are within-study coverage, not an independent encoding or biological replication, and because the populations differ, the lower expanded floor is not an encoding improvement.

\paragraph{GB1 under the same encoder.} The released protein-G binding assay of \citet{olson2014}, distributed with MAVE-NN 1.1.3 on its log2 enrichment scale, is encoded by the same one-hot rule with a 20-letter alphabet over 55 residues. A frozen size cap selects the first 50 of 1,485 position-pair blocks in lexicographic order, residue 2 against residues 3--52, so coverage is deliberately limited rather than random. Within each block, rectangles are anchored at the first measured amino-acid pair and need four measured double mutants; no wild-type or single-mutant corner is imputed. Of 17,907 rows in the selected blocks, 17,689 form 15,856 complete rectangles (218 candidates lack a corner). $H$ is injective, so the floor is exactly zero. The selected rows mix the released training, validation and test splits and no saved predictions exist for them, so this is a descriptive encoding diagnostic, not held-out prediction. It adds an assay under an existing encoder, not an independent encoder. The splicing panel of Appendix~\ref{app:mavenn} reproduces with floor zero under both equal and nested weights.

\paragraph{Blocked routes and scope.} ENsiRNA-mod and MEG-mod were rechecked at unchanged commits and remain blocked (Appendix~\ref{app:independentattempt}): ENsiRNA-mod builds its graph from per-duplex Rosetta coordinates that are not available for these states, and the forward pass of MEG-mod uses per-example modification tokens that its release never defines. No geometry, token meaning or partial branch was substituted, so no independent complete modified-siRNA encoding passed the eligibility gate. The audit therefore supports a reproducible diagnostic and a same-study extension, not prevalence across independent encoders. The independent lossy-encoder example with a positive floor is the lipid-panel fingerprint of Appendix~\ref{app:lossy}, in the last two rows of Table~\ref{tab:breadth}; it was not a route of this audit, and one instance is not prevalence. It fits nothing, so it cannot say when restoring representability improves accuracy; a controlled design crossing encoding restriction with training coverage, at matched training sizes and with a correctly specified interaction model, would address that question.

\subsection{An independent lossy encoder on a public lipid panel}\label{app:lossy}
This subsection applies the diagnostic of Appendix~\ref{app:contrastrank} to an encoder outside the audited siRNA preprocessing.

\textbf{The panel, and what it is not.} The AGILE combinatorial ionizable-lipid library \citep{agile2024} is a complete factorial of $20\times12\times5=1{,}200$ cells over one head-group factor at 20 levels and two tail factors at 12 and 5 levels. We use the 1,100 of those cells whose entry is a single compound. The 100 excluded cells are one whole level of the 12-level tail factor, level \texttt{B4}. The redistribution records them not in the panel file but in \path{data/CisTrans.csv} at the same commit, whose 200 rows carry an isomer field: all 100 \texttt{B4} entries read \texttt{C+T} and give two SMILES strings joined by a plus sign, so no single molecule can be encoded for them, while the 100 companion \texttt{B5} entries are the pure isomer and are retained. No row of the 1,100-cell panel file carries such an entry. What remains is a complete $20\times5\times11=1{,}100$ factorial, which is the panel used below. We take it in the form redistributed by \citep{lantern2025}, \path{data/AGILE_with_codes.csv} at commit \texttt{11240f29}, and the response is the \texttt{Target} column of that file, which is AGILE\textquotesingle s own HeLa transfection readout: all 1,100 labels agree to within $5\times10^{-8}$ with the corresponding \texttt{expt\_Hela} values in the pinned AGILE release at commit \texttt{aee76246}. That is agreement between two pinned releases, not a revision history for either. AGILE also reports a separate RAW 264.7 readout, which is not used here. The response is AGILE\textquotesingle s transfection potency in HeLa: its equation~(11) defines it as the base-two logarithm of the ratio of mean luminescence in transfected to untreated cells, 24 hours after treatment. We use the released values without further transformation, so the recorded range is $-2.3456$ to $15.9597$ and the floor of $0.3150033$ is a mean square in squared units of that logarithm.

\textbf{The encoder, and whose it is.} The encoder is a count circular fingerprint with chirality, at radius 2 and width 2,048 bins. It is the setting taken by the count-fingerprint configuration in the code released with \citet{moayedpour2024}, under the shipped setting \texttt{fingerprints.type: cfp}. That call passes counts and chirality explicitly and passes no positional argument, so the radius and the width are DeepChem\textquotesingle s documented defaults while the other two are chosen by the release. The vector holds integer counts, not bits. That configuration is one of several compared in that release and is not the input of its published model, which embeds molecules with a fine-tuned large language model; we audit the fingerprint configuration, not that model. The name \texttt{LipoBART} is the repository, not a model named in the paper. That setting selects one of four mutually exclusive generators, each producing one vector per molecule with no concatenation. Counts matter: binary bit vectors at the same radius give a different and coarser partition, so the calculation below uses the count vector. Classes are formed by byte equality of that complete input, the rule of Appendix~\ref{app:inputcollision}; the vector hashed is the featurizer\textquotesingle s own output. The 1,100 states fall into 920 classes: 180 merged pairs, so 360 merged states and 740 singletons. The merged pairs are distinct molecules with the same formula and different connectivity, and no merged pair has an equal measured target, the smallest gap being 0.0065 and the largest 10.019. Those gaps are context, not the proof: a positive floor is established by the four-term certificate below, whose recorded violation is 1.025131 while every attainable contrast vector gives exactly zero on it. The partition is unchanged at fingerprint widths 1,024, 2,048 and 4,096, with and without chirality, and folded or unfolded, so it is not an artifact of hash folding. We computed the panel from both public spellings of the same 1,100 cells. They assign different stereoisomers to 100 of them, all at one level of the eleven-level tail factor and differing in double-bond configuration alone; the two releases give identical merges and ranks, with floors agreeing to eight significant figures. No model was loaded, trained or run, because the calculation needs only the featurizer\textquotesingle s input map.

\paragraph{Ranks and floors.} Two contrast families were computed. Family 1 takes, within each head group, the anchored rectangles crossing the two tail factors, $q=800$; family 2 crosses head group with the 55 tail-pair combinations, $q=1{,}026$. Family 1 gives $\operatorname{rank}(C)=800$, $\operatorname{rank}(CH)=740$, no design dependencies, 60 encoding restrictions, and a floor of 0.3150033, which is $2.54\%$ of that family\textquotesingle s recorded contrast energy. All 800 rows of $CH$ are nonzero and pairwise distinct under equality, under sign and under scalar multiple, so the row-coincidence witness space has dimension zero and no one of the 60 restrictions is a duplicate-row relation. Family 2 gives $\operatorname{rank}(C)=1{,}026$, $\operatorname{rank}(CH)=855$, 171 encoding restrictions and a floor of 2.1591334; its 855 distinct rows account for the full rank, so there every restriction is a duplicate-row relation. Ranks were obtained by exact elimination modulo two large primes and over the rationals, and cross-checked against an SVD rank and against an independent symbolic implementation. Both families are listed in Table~\ref{tab:breadth}.

\paragraph{A certificate that can be checked by hand.} Sixty integer left-null vectors were computed and verified by exact integer multiplication to annihilate $CH$: forty with support 4 and twenty with support 6. A duplicate-row relation has support 2, so no one of these is one. Writing $[A1\!:\!C2\times B6]$ for the anchored rectangle at head group A1, level 2 of the five-level tail factor and level 6 of the eleven-level tail factor, the smallest certificate is
\[[A1\!:\!C2\times B6]-[A1\!:\!C2\times B9]-[A1\!:\!C2\times B10]+[A1\!:\!C3\times B9]=0,\]
forced for every decoder that sees only this fingerprint. The four recorded contrasts are $0.733390$, $4.582692$, $-2.403318$ and $2.471115$. Combined with the coefficients of that relation they give $1.025131$ rather than zero, which is why the floor is positive. The relation reflects merges along a diagonal of the tail-by-tail grid: at radius 2 the atom environments do not distinguish which acyl arm carries the longer chain, so one merge constrains four anchored rectangles at once.

\paragraph{Reproduction from public sources.} Take the panel as redistributed by \citet{lantern2025}, \path{data/AGILE_with_codes.csv} at commit \texttt{\seqsplit{11240f29ef92323649ae60d177b21df77e2d428b}}; the fingerprint call is \path{src/generate_mol2fp.py} in the code released with \citet{moayedpour2024} at commit \texttt{\seqsplit{1d792752944c87e2380e7a25af15f68d792bb3b6}}; form $H$ by exact equality of RDKit count vectors from \texttt{GetMorganGenerator(\allowbreak radius=2, \allowbreak fpSize=2048, \allowbreak includeChirality=True)}; form $C$ as the $4\times10$ anchored rectangles within each head group. Then the rank of $CH$ is 740 against $\operatorname{rank}(C)=800$, and projecting the recorded contrasts onto $\operatorname{col}(CH)$ gives the floor above.

\paragraph{Limits.} First, we apply a featurizer configuration from one release to a panel from another, so this is a statement about a featurization parameterisation and not about any published pipeline run end to end. On the iPhos library shipped with \citet{moayedpour2024}, the same call merges 158 of 572 states in the redistributed release, but that endpoint is categorical and we did not validate those concatenated structures, which is why the continuous AGILE panel is used here. Second, the merge is specific to the parameterisation: it is identical at radius 3, weaker but still present at radii 4 and 6, where 80 and 40 pairs merge, and absent at the three radii at or above 8 that we tested, namely 8, 16 and 1,024, where the fingerprint separates all 1,100 states so no encoding restriction remains; 1,024 is the value \citet{lantern2025} passes positionally. Nothing is recorded between 17 and 1,023. The released pipeline of \citet{lantern2025} also raises on duplicate feature vectors, so a lossy configuration stops there; we report both as measured negatives. The reading that DeepChem\textquotesingle s default radius is 2 rests on its documented signature, which we did not execute here, and every number above was produced at an explicitly passed radius of 2. Third, nothing is fitted on this panel, so the floor states what the encoding permits and nothing is claimed about any predictor\textquotesingle s realised error. Fourth, the gap depends on the contrast family, as family 2 shows: duplicate-row checks miss 60 encoding restrictions in family 1 and none in family 2, where all 171 restrictions are duplicate-row relations, so the gap is not nonzero for this encoder unconditionally. The panel is a lipid library and not modified siRNA, so this broadens where the diagnostic applies without supporting the modified-siRNA application. The ranks, the certificates and the floor are each confirmed by more than one method, and the panel construction and the choice of contrast families were not independently reviewed.

\subsection{Reference-independent B3 decomposition}\label{app:b3decomp}
The verified 165-state panel is a complete 15 by 11 matrix with antisense rows and sense columns, including the shared reference state. For measured or predicted $Y$ set $\mu=\overline{Y}$, $a_i=\overline{Y_{i\cdot}}-\mu$, $b_j=\overline{Y_{\cdot j}}-\mu$ and $R_{ij}=Y_{ij}-\mu-a_i-b_j$, with equal weights over all 165 cells. Reconstruction is exact: $\sum_ia_i=\sum_jb_j=0$, every row and column sum of $R$ vanishes, the three centred components are mutually orthogonal and total centred energy splits exactly, the largest violation over all nine checks, a residual column sum, being 2.7e-15 against a $10^{-12}$ tolerance. The split is linear, so prediction error decomposes exactly as $\mathrm{MSE}=(\Delta\mu)^2+\overline{(\Delta a_i)^2}+\overline{(\Delta b_j)^2}+\overline{(\Delta R_{ij})^2}$, and the realised identity gap is at most $3\times10^{-17}$.

Only saved full-precision endpoint predictions are used; no B3 fit or inference was performed. The historical ten-seed ensemble and its seeds are analysed separately from the three-seed utility ensemble, and seed spread is optimisation variability, not biological replication: the GNN\textquotesingle s nonadditive error component has mean 0.018185 and standard deviation 0.000067 across ten seeds. Two columns of the historical export, the \texttt{original} arms, reproduce the \texttt{corrected} arms bitwise over all 165 states and ten seeds although their checkpoints differ; this unresolved anomaly is retained in \texttt{b3decomp\_export\_anomalies} and carries no reported result, since every published B3 number uses the four audited families. Absolute amplitudes are kept beside fractions, because a large fraction of a tiny spread is not accurate recovery: predicted nonadditive energies are 1.6e-06 to 1.5e-04 against a measured 0.018138. Reference-based contrasts, their weighting and the 140-rectangle results are unchanged, and no projection fitted to evaluation labels is reported as a predictor.

\subsection{Training support for the restricted siRNA chemistry}\label{app:supportdecouple}
This subsection reports the follow-up to the mask-restoration refit of Appendix~\ref{app:maskrestore}, in which the restored columns had no training support. It is a siRNA experiment and is independent of the splicing intervention below.

\textbf{Why the source-held-out form cannot be built here.} This part fits nothing. Five of the 44 restored columns have nonzero marginal gain on the B3 interaction rank. They are 2,4-carbocyclic ethylene-bridged LNA-LNA, 2,4-carbocyclic LNA-LNA and 2-deoxy-2-N,4-C-ethylene-LNA, with gain 27 each, and 2-aminoethyl and 2-guanidinoethyl, with gain 8 each. A minimal recovery triple is the first two together with 2-aminoethyl. They are carried by 118, 116, 107, 264 and 47 rows. Every carrier of every one of them lies in source family 19282453, the source of the B3 panel, so none is admissible under the frozen source-held-out rule. Of the 4,884 dataset rows, 2,960 are admissible. The 1,924 inadmissible rows are exactly the intersection of that source family with the 13-mer-overlap set, so relaxing the conservative 13-mer rule admits no carrier. The 19 restored columns that admissible rows do exercise have marginal gain exactly zero on the B3 rank, and columns 31 to 38 are zero in all 4,884 rows. On this panel, therefore, source-held-out evaluation and training support for the restricted chemistry are mutually exclusive. This is a property of this dataset. A cohort carrying the same chemistry on unrelated backgrounds would permit the test.

\textbf{The pre-registered source-exposed arm.} The protocol was frozen before any fit, under the content hash \texttt{814789c5\ldots} of its canonical form, and the fitting stage re-asserts it. The node-support mask is removed in both arms, so representability is identical and full in each: 165 complete-input classes, $\operatorname{rank}(CH)=140$ and an encoding floor of $1.1\times10^{-31}$. The arms share 605 training rows. One arm adds 200 off-background non-carriers of the five driver columns, the other 200 off-background carriers. Both pools come from the same study group and the same source family, and each occupies the same 26 duplex backgrounds, with per-background counts matched exactly. Selection is deterministic and reads no label. Both arms use learning rate $10^{-4}$ and weight decay $10^{-4}$. Relation support and context gating are identical, the B3 evaluation input, adjacency and contrast tensors are byte-identical, parameter counts match at 153,345, and each seed starts from identical parameters. The training rows are the only difference.

\textbf{Result.} Improvement is support absent minus support present, so positive means error fell. On all 140 rectangles it is $+0.000018$, which is $0.18\%$ of the floor removed in the mask refit. It is $+0.000072$ on the 100 affected rectangles, $-0.000116$ on the 40 unaffected, $-0.000168$ on the 25 rectangles none of whose corners carries a driver column, and $-0.001189$ on the 3 rectangles whose inputs are unchanged. The antisense family gives $+0.000111$ on $q=14$ and the sense family $+0.000280$ on $q=10$ with no multiplicity control, which we do not report as a finding. All ten seeds fall in the pre-declared near-zero band, five improve, and the per-seed range is $-0.000355$ to $+0.000439$. The 25-rectangle subset was pre-registered as the control for a target-distribution shift. It moved against the supported arm by more than the whole panel improved.

\textbf{Matched epochs.} Eight (seed, epoch) pairs have a saved checkpoint in both arms. The mean improvement over them is $-0.000361$, with the supported arm worse at seven of the eight. The one pair favouring it is seed 9973 at epoch 61, at $+0.001068$. The two arms train on different row sets, so their trajectories diverge and a shared epoch index is not a shared point. This is a sensitivity finding, not the fixed-parameter decomposition of Appendix~\ref{app:maskdecomp}.

\textbf{What this arm does not license.} It is source-exposed. The 200 pool rows come from the outer0 test group and from source family 19282453, endpoint-disjoint and background-disjoint from the B3 panel, so no held-out claim rests on them. The two pools' label distributions differ, which was recorded before fitting and is not controlled by the frozen design: pool means 0.641 and 0.573, 0.616 and 0.599 at the training level, and 0.659 and 0.651 under the equal-group weighting used in training. Support is partial. The supported arm covers the five driver columns and one further column of gain zero, while 26 of the 44 restored columns are unsupported in both arms, 8 of them unsupportable by any row in this dataset. One support amount was tested, 200 rows out of 805, on one architecture and one panel. The arm comprises 20 fits and 8,624 optimizer updates with no failures, and its per-seed rows, per-subset values and matched-epoch comparison are recorded with the frozen protocol in that arm\textquotesingle s own run directory, outside the results store.

\subsection{Splicing training-support intervention: design and accounting}\label{app:support}
\begin{figure}[ht]\centering\includegraphics[width=0.7\textwidth]{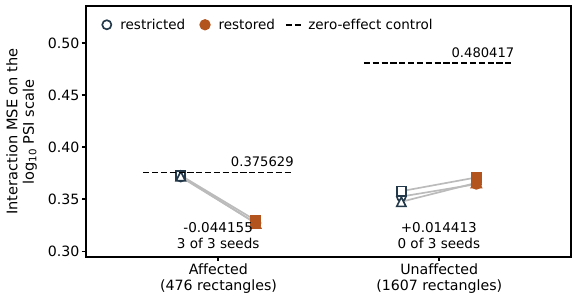}
\caption{Fixed-size training-support comparison, three paired seeds individually.
Restoration lowers interaction MSE on the affected rectangles and not the unaffected ones;
the zero-effect control is identical across arms within a subset and differs between them,
so it is drawn per subset. This is one controlled training-data comparison on one feature
(position 9, base G, smallest of 28 four-state training counts), only 2000 training rows are
replaced and the shared validation set carries none of it. It does not show that missing
support always causes error, nor that support alone makes an explanation
valid.}\label{fig:support}
\end{figure}

The feature is selected from training input counts: among positions admitting all four nucleotides, choose the smallest trainval count, breaking ties by position index and then nucleotide. That selects position 9, G, with 5286 trainval occurrences, which varies in the library and has measured restoration examples, so no globally unmeasured substitution is used. Both arms train on $14{,}238$ observations and share a core of $12{,}238$. The restricted arm adds $2{,}000$ feature-free observations to that core; the restored arm instead adds the first $2{,}000$ feature-carrying trainval sequences in lexicographic order. The removed feature-free observations are the last $2{,}000$ restricted-arm training rows in that order. Selection uses identities only, not response labels. Backgrounds are not individually matched, so other aspects of the training distribution can differ along with feature support. A shared validation set of 4881 feature-free rows keeps early stopping blind to the withheld feature, test identities are identical and excluded from every selection step, and representation and vocabulary are unchanged, so observations are removed, not the feature channel. Affected rectangles are those with the feature in at least one corner, fixed by identity rather than observed effect or error size, and the unaffected set is a specificity check. Three paired seeds per arm supplied six retained fits, with $62{,}424$ reported optimizer updates, no recorded training failure in that retained invocation and no GPU use. An earlier invocation completed six fits but failed while serializing an empirical-gauge coefficient undefined for a zero-support feature, before its per-seed results were persisted. The recovery therefore consumed twelve attempts against a six-attempt authorization. The same feature, design, identities, weights and seeds were reused; this was a repeated execution, not a second specification. The $62{,}424$-update count covers the retained invocation; a complete two-invocation update total is not established here. The attempt ledger and retained per-seed results are recorded separately in \texttt{support\_attempt\_ledger} and \texttt{support\_seed\_results}; unpersisted first-invocation outcomes are not represented as available per-seed evidence. The zero-effect control is identical across arms on identical identities, so an improvement in training loss alone would not count.\label{appendix:end}

%% file: references.bib
@article{huesken2005,
  author  = {Dieter Huesken and Joerg Lange and Craig Mickanin and Jan Weiler and Fred Asselbergs and others},
  title   = {Design of a genome-wide {siRNA} library using an artificial neural network},
  journal = {Nature Biotechnology},
  volume  = {23},
  number  = {8},
  pages   = {995--1001},
  year    = {2005},
  doi     = {10.1038/nbt1118},
  url     = {https://www.nature.com/articles/nbt1118}
}

@article{huesken2005correction,
  author  = {Dieter Huesken and Joerg Lange and Craig Mickanin and Jan Weiler and others},
  title   = {Corrigendum: Design of a genome-wide {siRNA} library using an artificial neural network},
  journal = {Nature Biotechnology},
  volume  = {23},
  number  = {10},
  pages   = {1315},
  year    = {2005},
  doi     = {10.1038/nbt1005-1315a},
  url     = {https://www.nature.com/articles/nbt1005-1315a}
}

@article{huesken2006correction,
  author  = {Dieter Huesken and Joerg Lange and Craig Mickanin and Jan Weiler and others},
  title   = {Corrigendum: Design of a genome-wide {siRNA} library using an artificial neural network},
  journal = {Nature Biotechnology},
  volume  = {24},
  number  = {8},
  pages   = {1033},
  year    = {2006},
  doi     = {10.1038/nbt0806-1033e},
  url     = {https://www.nature.com/articles/nbt0806-1033e}
}

@article{gneiting2007,
  author  = {Tilmann Gneiting and Adrian E. Raftery},
  title   = {Strictly Proper Scoring Rules, Prediction, and Estimation},
  journal = {Journal of the American Statistical Association},
  volume  = {102},
  number  = {477},
  pages   = {359--378},
  year    = {2007},
  doi     = {10.1198/016214506000001437},
  url     = {https://doi.org/10.1198/016214506000001437}
}

@article{elbashir2001,
  author  = {Sayda M. Elbashir and Jens Harborth and Winfried Lendeckel and Abdullah Yalcin and Klaus Weber and Thomas Tuschl},
  title   = {Duplexes of 21-nucleotide {RNAs} mediate {RNA} interference in cultured mammalian cells},
  journal = {Nature},
  volume  = {411},
  number  = {6836},
  pages   = {494--498},
  year    = {2001},
  doi     = {10.1038/35078107},
  url     = {https://doi.org/10.1038/35078107}
}

@article{reynolds2004,
  author  = {Angela Reynolds and Devin Leake and Queta Boese and Stephen Scaringe and William S. Marshall and Anastasia Khvorova},
  title   = {Rational {siRNA} design for {RNA} interference},
  journal = {Nature Biotechnology},
  volume  = {22},
  number  = {3},
  pages   = {326--330},
  year    = {2004},
  doi     = {10.1038/nbt936},
  url     = {https://doi.org/10.1038/nbt936}
}

@article{uitei2004,
  author  = {Kumiko Ui-Tei and Yuki Naito and Fumitaka Takahashi and others},
  title   = {Guidelines for the selection of highly effective {siRNA} sequences for mammalian and chick {RNA} interference},
  journal = {Nucleic Acids Research},
  volume  = {32},
  number  = {3},
  pages   = {936--948},
  year    = {2004},
  doi     = {10.1093/nar/gkh247},
  url     = {https://doi.org/10.1093/nar/gkh247}
}

@article{khvorova2003,
  author  = {Anastasia Khvorova and Angela Reynolds and Sumedha D. Jayasena},
  title   = {Functional {siRNAs} and {miRNAs} Exhibit Strand Bias},
  journal = {Cell},
  volume  = {115},
  number  = {2},
  pages   = {209--216},
  year    = {2003},
  doi     = {10.1016/s0092-8674(03)00801-8},
  url     = {https://doi.org/10.1016/s0092-8674(03)00801-8}
}

@article{schwarz2003,
  author  = {Dianne S. Schwarz and György Hutvágner and Tingting Du and Zuoshang Xu and Neil Aronin and Phillip D. Zamore},
  title   = {Asymmetry in the Assembly of the {RNAi} Enzyme Complex},
  journal = {Cell},
  volume  = {115},
  number  = {2},
  pages   = {199--208},
  year    = {2003},
  doi     = {10.1016/s0092-8674(03)00759-1},
  url     = {https://doi.org/10.1016/s0092-8674(03)00759-1}
}

@article{jackson2003,
  author  = {Aimee L. Jackson and Steven R. Bartz and Janell Schelter and Sumire V. Kobayashi and Julja Burchard and Mao Mao and Bin Li and Guy Cavet and Peter S. Linsley},
  title   = {Expression profiling reveals off-target gene regulation by {RNAi}},
  journal = {Nature Biotechnology},
  year    = {2003},
  volume  = {21},
  number  = {6},
  pages   = {635--637},
  doi     = {10.1038/nbt831},
  url     = {https://doi.org/10.1038/nbt831}
}

@article{birmingham2006,
  author  = {Birmingham, Amanda and Anderson, Emily M. and Reynolds, Angela
             and Ilsley-Tyree, Diane and others},
  title   = {3' {UTR} seed matches, but not overall identity, are associated
             with {RNAi} off-targets},
  journal = {Nature Methods},
  year    = {2006},
  volume  = {3},
  number  = {3},
  pages   = {199--204},
  doi     = {10.1038/nmeth854},
}

@article{allerson2005,
  author  = {Charles R. Allerson and Namir Sioufi and Russell Jarres and others},
  title   = {Fully 2'-Modified Oligonucleotide Duplexes with Improved in Vitro Potency and Stability Compared to Unmodified Small Interfering {RNA}},
  journal = {Journal of Medicinal Chemistry},
  year    = {2005},
  volume  = {48},
  number  = {4},
  pages   = {901--904},
  doi     = {10.1021/jm049167j},
  url     = {https://doi.org/10.1021/jm049167j}
}

@article{bramsen2009,
  author  = {Bramsen, Jesper B. and Laursen, Maria B. and Nielsen, Anne F.
             and others},
  title   = {A large-scale chemical modification screen identifies design rules
             to generate {siRNAs} with high activity, high stability and low toxicity},
  journal = {Nucleic Acids Research},
  year    = {2009},
  volume  = {37},
  number  = {9},
  pages   = {2867--2881},
  issn    = {0305-1048},
  doi     = {10.1093/nar/gkp106},
}

@article{nair2014,
  author  = {Nair, Jayaprakash K. and Willoughby, Jennifer L. S. and Chan, Amy
             and others},
  title   = {Multivalent {N}-Acetylgalactosamine-Conjugated {siRNA} Localizes in
             Hepatocytes and Elicits Robust {RNAi}-Mediated Gene Silencing},
  journal = {Journal of the American Chemical Society},
  year    = {2014},
  volume  = {136},
  number  = {49},
  pages   = {16958--16961},
  issn    = {0002-7863},
  doi     = {10.1021/ja505986a},
}

@inproceedings{gilmer2017,
  author    = {Justin Gilmer and Samuel S. Schoenholz and Patrick F. Riley and Oriol Vinyals and George E. Dahl},
  title     = {Neural Message Passing for Quantum Chemistry},
  year      = {2017},
  booktitle = {International Conference on Machine Learning},
  pages     = {1263--1272},
  series    = {Proceedings of Machine Learning Research},
  volume    = {70},
  url       = {https://proceedings.mlr.press/v70/gilmer17a.html}
}

@misc{battaglia2018,
  author = {Peter W. Battaglia and Jessica B. Hamrick and Victor Bapst and Alvaro Sanchez-Gonzalez and others},
  title  = {Relational inductive biases, deep learning, and graph networks},
  year   = {2018},
  note   = {arXiv:1806.01261},
  url    = {https://arxiv.org/abs/1806.01261}
}

@inproceedings{xu2019,
  author    = {Keyulu Xu and Weihua Hu and Jure Leskovec and Stefanie Jegelka},
  title     = {How Powerful are Graph Neural Networks?},
  year      = {2019},
  booktitle = {International Conference on Learning Representations},
  url       = {https://openreview.net/forum?id=ryGs6iA5Km}
}

@inproceedings{zaheer2017,
  author    = {Manzil Zaheer and Satwik Kottur and Siamak Ravanbakhsh and Barnabas Poczos and Ruslan Salakhutdinov and Alexander Smola},
  title     = {Deep Sets},
  year      = {2017},
  booktitle = {Advances in Neural Information Processing Systems},
  volume    = {30},
  pages     = {3391--3401},
  url       = {https://papers.nips.cc/paper/6931-deep-sets.pdf}
}

@article{cmsirna2026,
  author  = {Sicheng He and Cheng Chen and Xianrun Pan and others},
  title   = {{CMsiRNAdb}: a database of chemically modified {siRNA} silencing efficiency for nucleic acid drug design},
  year    = {2026},
  journal = {BMC Bioinformatics},
  volume  = {27},
  pages   = {33},
  doi     = {10.1186/s12859-025-06359-y},
  url     = {https://doi.org/10.1186/s12859-025-06359-y}
}

@article{davis2025,
  author  = {Davis, Sarah M. and Hildebrand, Samuel and MacMillan, Hannah J.
             and Monopoli, Kathryn R. and Buchwald, Julianna and Sousa, Jacquelyn
             and Cooper, David and Ly, Socheata and Echeverria, Dimas
             and McHugh, Nicholas and Ferguson, Chantal and Coles, Andrew
             and Hariharan, Vignesh N. and O'Reilly, Daniel and Tang, Qi
             and Furgal, Raymond and Yamada, Ken and Alterman, Julia F.
             and Gilbert, James W. and Knox, Emily and Pineda, Yamilett
             and Weston, Caitlyn N. and Baer, Christina E. and Pai, Athma A.
             and Khvorova, Anastasia},
  title   = {Systematic analysis of {siRNA} and {mRNA} features impacting fully
             chemically modified {siRNA} efficacy},
  journal = {Nucleic Acids Research},
  year    = {2025},
  volume  = {53},
  number  = {12},
  pages   = {gkaf479},
  issn    = {1362-4962},
  doi     = {10.1093/nar/gkaf479},
}

@article{meg2026,
  author  = {Yuanting Chen and Mengyu Tong and Long Chen and Weihua Li and Keyun Zhu and Jiaying Li and Yun Tang and Guixia Liu},
  title   = {{MEG-mod}: A Multiview Enhanced Graph Neural Network for Knockdown Efficiency Prediction of Chemically Modified {siRNA}},
  year    = {2026},
  journal = {Journal of Medicinal Chemistry},
  volume  = {69},
  pages   = {13434--13451},
  doi     = {10.1021/acs.jmedchem.6c00411},
  url     = {https://doi.org/10.1021/acs.jmedchem.6c00411}
}

@article{ridge1970,
  author  = {Arthur E. Hoerl and Robert W. Kennard},
  title   = {Ridge Regression: Biased Estimation for Nonorthogonal Problems},
  year    = {1970},
  journal = {Technometrics},
  volume  = {12},
  pages   = {55--67},
  doi     = {10.1080/00401706.1970.10488634},
  url     = {https://doi.org/10.1080/00401706.1970.10488634}
}

@article{extratrees2006,
  author  = {Pierre Geurts and Damien Ernst and Louis Wehenkel},
  title   = {Extremely randomized trees},
  year    = {2006},
  journal = {Machine Learning},
  volume  = {63},
  pages   = {3--42},
  doi     = {10.1007/s10994-006-6226-1},
  url     = {https://doi.org/10.1007/s10994-006-6226-1}
}

@article{numpy2020,
  author  = {Charles R. Harris and K. Jarrod Millman and Stéfan J. van der Walt and others},
  title   = {Array programming with {NumPy}},
  year    = {2020},
  journal = {Nature},
  volume  = {585},
  pages   = {357--362},
  doi     = {10.1038/s41586-020-2649-2},
  url     = {https://doi.org/10.1038/s41586-020-2649-2}
}

@article{scipy2020,
  author  = {Pauli Virtanen and Ralf Gommers and Travis E. Oliphant and others},
  title   = {{SciPy} 1.0: fundamental algorithms for scientific computing in {Python}},
  year    = {2020},
  journal = {Nature Methods},
  volume  = {17},
  pages   = {261--272},
  doi     = {10.1038/s41592-019-0686-2},
  url     = {https://doi.org/10.1038/s41592-019-0686-2}
}

@article{matplotlib2007,
  author  = {John D. Hunter},
  title   = {{Matplotlib}: A {2D} Graphics Environment},
  year    = {2007},
  journal = {Computing in Science \& Engineering},
  volume  = {9},
  pages   = {90--95},
  doi     = {10.1109/MCSE.2007.55},
  url     = {https://doi.org/10.1109/MCSE.2007.55}
}

@misc{adam2015,
  author = {Diederik P. Kingma and Jimmy Ba},
  title  = {{Adam}: A Method for Stochastic Optimization},
  year   = {2014},
  note   = {arXiv:1412.6980; author manuscript},
  url    = {https://arxiv.org/abs/1412.6980}
}

@misc{adamw2019,
  author = {Ilya Loshchilov and Frank Hutter},
  title  = {Decoupled Weight Decay Regularization},
  year   = {2017},
  note   = {arXiv:1711.05101; author manuscript},
  url    = {https://arxiv.org/abs/1711.05101}
}

@misc{layernorm2016,
  author = {Jimmy Lei Ba and Jamie Ryan Kiros and Geoffrey E. Hinton},
  title  = {Layer Normalization},
  year   = {2016},
  note   = {arXiv:1607.06450; author manuscript},
  url    = {https://arxiv.org/abs/1607.06450}
}

@misc{rgcn2018,
  author = {Michael Schlichtkrull and Thomas N. Kipf and Peter Bloem and Rianne van den Berg and Ivan Titov and Max Welling},
  title  = {Modeling Relational Data with Graph Convolutional Networks},
  year   = {2017},
  note   = {arXiv:1703.06103; author manuscript},
  url    = {https://arxiv.org/abs/1703.06103}
}

@misc{gcn2017,
  author = {Thomas N. Kipf and Max Welling},
  title  = {Semi-Supervised Classification with Graph Convolutional Networks},
  year   = {2016},
  note   = {arXiv:1609.02907; author manuscript},
  url    = {https://arxiv.org/abs/1609.02907}
}

@misc{graphsage2017,
  author = {William L. Hamilton and Rex Ying and Jure Leskovec},
  title  = {Inductive Representation Learning on Large Graphs},
  year   = {2017},
  note   = {arXiv:1706.02216; author manuscript},
  url    = {https://arxiv.org/abs/1706.02216}
}

@misc{kim2014,
  author = {Yoon Kim},
  title  = {Convolutional Neural Networks for Sentence Classification},
  year   = {2014},
  note   = {arXiv:1408.5882; author manuscript},
  url    = {https://arxiv.org/abs/1408.5882}
}

@misc{ensemble2017,
  author = {Balaji Lakshminarayanan and Alexander Pritzel and Charles Blundell},
  title  = {Simple and Scalable Predictive Uncertainty Estimation using Deep Ensembles},
  year   = {2016},
  note   = {arXiv:1612.01474; author manuscript},
  url    = {https://arxiv.org/abs/1612.01474}
}

@misc{torch2019,
  author = {Adam Paszke and Sam Gross and Francisco Massa and others},
  title  = {{PyTorch}: An Imperative Style, High-Performance Deep Learning Library},
  year   = {2019},
  note   = {arXiv:1912.01703; author manuscript},
  url    = {https://arxiv.org/abs/1912.01703}
}

@article{ensi2025,
  author  = {Wenchong Tan and Mingshu Dai and Shimin Ye and Xin Tang and Dawei Jiang and Dong Chen and Hongli Du},
  title   = {{ENsiRNA}: A Multimodality Method for {siRNA-mRNA} and Modified {siRNA} Efficacy Prediction Based on Geometric Graph Neural Network},
  journal = {Journal of Molecular Biology},
  volume  = {437},
  pages   = {169131},
  year    = {2025},
  doi     = {10.1016/j.jmb.2025.169131},
  url     = {https://pubmed.ncbi.nlm.nih.gov/40194620/}
}

@article{rf2001,
  author  = {Leo Breiman},
  title   = {Random Forests},
  journal = {Machine Learning},
  volume  = {45},
  pages   = {5--32},
  year    = {2001},
  doi     = {10.1023/A:1010933404324},
  url     = {https://doi.org/10.1023/A:1010933404324}
}

@article{leakage2023,
  author  = {Sayash Kapoor and Arvind Narayanan},
  title   = {Leakage and the reproducibility crisis in machine-learning-based science},
  journal = {Patterns},
  volume  = {4},
  pages   = {100804},
  year    = {2023},
  doi     = {10.1016/j.patter.2023.100804},
  url     = {https://par.nsf.gov/servlets/purl/10513990}
}

@article{bootstrap1979,
  author  = {Bradley Efron},
  title   = {Bootstrap Methods: Another Look at the Jackknife},
  journal = {The Annals of Statistics},
  volume  = {7},
  pages   = {1--26},
  year    = {1979},
  doi     = {10.1214/aos/1176344552},
  url     = {https://doi.org/10.1214/aos/1176344552}
}

@article{dropout2014,
  author  = {Nitish Srivastava and Geoffrey Hinton and Alex Krizhevsky and Ilya Sutskever and Ruslan Salakhutdinov},
  title   = {Dropout: A Simple Way to Prevent Neural Networks from Overfitting},
  journal = {Journal of Machine Learning Research},
  volume  = {15},
  pages   = {1929--1958},
  year    = {2014},
  url     = {https://www.jmlr.org/papers/v15/srivastava14a.html}
}

@article{sklearn2011,
  author  = {Fabian Pedregosa and Gaël Varoquaux and Alexandre Gramfort and others},
  title   = {{Scikit-learn}: Machine Learning in {Python}},
  journal = {Journal of Machine Learning Research},
  volume  = {12},
  pages   = {2825--2830},
  year    = {2011},
  url     = {https://jmlr.org/papers/v12/pedregosa11a.html}
}

@inproceedings{wilds2021,
  author    = {Pang Wei Koh and Shiori Sagawa and Henrik Marklund and others},
  title     = {{WILDS}: A Benchmark of in-the-Wild Distribution Shifts},
  booktitle = {International Conference on Machine Learning, PMLR},
  pages     = {5637--5664},
  year      = {2021},
  url       = {https://proceedings.mlr.press/v139/koh21a.html}
}

@inproceedings{oligogym2025,
  author    = {Rachapun Rotrattanadumrong and Carlo {De Donno}},
  title     = {{OligoGym}: Curated Datasets and Benchmarks for Oligonucleotide Drug Discovery},
  booktitle = {Advances in Neural Information Processing Systems},
  volume    = {38},
  year      = {2025},
  url       = {https://proceedings.neurips.cc/paper_files/paper/2025/hash/4927dd556ff6e49671e3a77ac1f12487-Abstract-Datasets_and_Benchmarks_Track.html}
}

@inproceedings{chen2023harsanyi,
  author    = {Lu Chen and Siyu Lou and Keyan Zhang and Jin Huang and Quanshi Zhang},
  title     = {{HarsanyiNet}: Computing Accurate Shapley Values in a Single Forward Propagation},
  booktitle = {International Conference on Machine Learning, PMLR},
  pages     = {4804--4825},
  year      = {2023},
  url       = {https://proceedings.mlr.press/v202/chen23s.html}
}

@inproceedings{xia2023ncm,
  author    = {Kevin Xia and Yushu Pan and Elias Bareinboim},
  title     = {Neural Causal Models for Counterfactual Identification and Estimation},
  booktitle = {International Conference on Learning Representations},
  year      = {2023},
  url       = {https://arxiv.org/abs/2210.00035}
}

@article{Wong2018SpliceSites,
  author  = {Wong, Mandy S. and Kinney, Justin B. and Krainer, Adrian R.},
  title   = {Quantitative Activity Profile and Context Dependence of All Human 5' Splice Sites},
  journal = {Molecular Cell},
  year    = {2018},
  volume  = {71},
  number  = {6},
  pages   = {1012--1026.e3},
  doi     = {10.1016/j.molcel.2018.07.033},
  url     = {https://pmc.ncbi.nlm.nih.gov/articles/PMC6179149/}
}

@inproceedings{tan2024consistency,
  author    = {Jiyuan Tan and Jose Blanchet and Vasilis Syrgkanis},
  title     = {Consistency of Neural Causal Partial Identification},
  booktitle = {Advances in Neural Information Processing Systems},
  year      = {2024},
  note      = {Updated arXiv version 3, June 2025},
  url       = {https://arxiv.org/abs/2405.15673}
}

@inproceedings{lengerich2020pure,
  author    = {Benjamin Lengerich and Sarah Tan and Chun-Hao Chang and Giles Hooker and Rich Caruana},
  title     = {Purifying Interaction Effects with the Functional {ANOVA}: An Efficient Algorithm for Recovering Identifiable Additive Models},
  booktitle = {International Conference on Artificial Intelligence and Statistics},
  pages     = {2402--2412},
  year      = {2020},
  url       = {https://proceedings.mlr.press/v108/lengerich20a.html}
}

@article{kuskova2026real,
  author  = {Valentina Kuskova and Dmitry Zaytsev and Michael Coppedge},
  title   = {When Are Neural Interaction Discoveries Real? Identifiability, Recoverability, and a Pre-Fit Diagnostic},
  journal = {arXiv preprint arXiv:2606.08390},
  year    = {2026},
  url     = {https://arxiv.org/abs/2606.08390}
}

@article{bilodeau2024,
  author  = {Blair Bilodeau and Natasha Jaques and Pang Wei Koh and Been Kim},
  title   = {Impossibility theorems for feature attribution},
  journal = {PNAS},
  volume  = {121},
  number  = {2},
  pages   = {e2304406120},
  year    = {2024},
  doi     = {10.1073/pnas.2304406120},
  url     = {https://doi.org/10.1073/pnas.2304406120}
}

@article{koo2023,
  author  = {Antonio Majdandzic and Chandana Rajesh and Peter K. Koo},
  title   = {Correcting gradient-based interpretations of deep neural networks for genomics},
  journal = {Genome Biology},
  volume  = {24},
  pages   = {109},
  year    = {2023},
  doi     = {10.1186/s13059-023-02956-3},
  url     = {https://doi.org/10.1186/s13059-023-02956-3}
}

@article{mavenn2022,
  author  = {Ammar Tareen and Mahdi Kooshkbaghi and Anna Posfai and William T. Ireland and David M. McCandlish and Justin B. Kinney},
  title   = {{MAVE-NN}: learning genotype-phenotype maps from multiplex assays of variant effect},
  journal = {Genome Biology},
  volume  = {23},
  pages   = {98},
  year    = {2022},
  doi     = {10.1186/s13059-022-02661-7},
  url     = {https://doi.org/10.1186/s13059-022-02661-7}
}

@article{azzolin2026,
  author  = {Steve Azzolin and Stefano Teso and Bruno Lepri and Andrea Passerini and Sagar Malhotra},
  title   = {{GNN} Explanations that do not Explain and How to find Them},
  journal = {arXiv preprint arXiv:2601.20815},
  year    = {2026},
  url     = {https://arxiv.org/abs/2601.20815}
}

@inproceedings{azzolin2025,
  author    = {Steve Azzolin and Antonio Longa and Stefano Teso and Andrea Passerini},
  title     = {Reconsidering Faithfulness in Regular, Self-Explainable and Domain Invariant {GNNs}},
  booktitle = {International Conference on Learning Representations (ICLR)},
  year      = {2025},
  url       = {https://arxiv.org/abs/2406.15156}
}

@article{chen2020,
  author  = {Hugh Chen and Joseph D. Janizek and Scott Lundberg and Su-In Lee},
  title   = {True to the Model or True to the Data?},
  journal = {arXiv preprint arXiv:2006.16234},
  year    = {2020},
  url     = {https://arxiv.org/abs/2006.16234}
}

@article{saliency2013,
  author  = {Karen Simonyan and Andrea Vedaldi and Andrew Zisserman},
  title   = {Deep Inside Convolutional Networks: Visualising Image Classification Models and Saliency Maps},
  journal = {arXiv preprint arXiv:1312.6034},
  year    = {2013},
  url     = {https://arxiv.org/abs/1312.6034}
}

@article{boxhunter1961,
  author  = {G. E. P. Box and J. S. Hunter},
  title   = {The {$2^{k-p}$} Fractional Factorial Designs Part {I}},
  journal = {Technometrics},
  volume  = {3},
  number  = {3},
  pages   = {311--351},
  year    = {1961},
  doi     = {10.1080/00401706.1961.10489951},
  url     = {https://doi.org/10.1080/00401706.1961.10489951}
}

@article{wuchen1992,
  author  = {Wu, C. F. J. and Chen, Youyi},
  title   = {A graph-aided method for planning two-level experiments when
             certain interactions are important},
  journal = {Technometrics},
  year    = {1992},
  volume  = {34},
  number  = {2},
  pages   = {162--175},
  doi     = {10.1080/00401706.1992.10484905},
  url     = {https://doi.org/10.1080/00401706.1992.10484905},
}

@article{grabisch1999,
  author  = {Michel Grabisch and Marc Roubens},
  title   = {An axiomatic approach to the concept of interaction among players in cooperative games},
  journal = {International Journal of Game Theory},
  volume  = {28},
  pages   = {547--565},
  year    = {1999},
  doi     = {10.1007/s001820050125},
  url     = {https://doi.org/10.1007/s001820050125}
}

@article{damour2022,
  author  = {Alexander D'Amour and Katherine Heller and Dan Moldovan and others},
  title   = {Underspecification Presents Challenges for Credibility in Modern Machine Learning},
  journal = {Journal of Machine Learning Research},
  year    = {2022},
  volume  = {23},
  number  = {226},
  pages   = {1--61},
  url     = {https://jmlr.org/papers/v23/20-1335.html}
}

@inproceedings{majdandzic2021workshop,
  author    = {Antonio Majdandzic and Peter K. Koo},
  title     = {Statistical correction of input gradients for black box models trained with categorical input features},
  booktitle = {ICML 2021 Workshop on Computational Biology},
  year      = {2021},
  url       = {https://icml-compbio.github.io/2021/papers/WCBICML2021_paper_56.pdf}
}

@article{olson2014,
  author  = {C. Anders Olson and Nicholas C. Wu and Ren Sun},
  title   = {A comprehensive biophysical description of pairwise epistasis throughout an entire protein domain},
  journal = {Current Biology},
  year    = {2014},
  volume  = {24},
  number  = {22},
  pages   = {2643--2651},
  doi     = {10.1016/j.cub.2014.09.072},
  url     = {https://doi.org/10.1016/j.cub.2014.09.072}
}

@article{moayedpour2024,
  title   = {Representations of lipid nanoparticles using large language
             models for transfection efficiency prediction},
  author  = {Moayedpour, S. and Broadbent, J. and others},
  journal = {Bioinformatics},
  volume  = {40},
  number  = {7},
  pages   = {btae342},
  year    = {2024},
  doi     = {10.1093/bioinformatics/btae342},
}

@article{agile2024,
  title   = {{AGILE} platform: a deep learning powered approach to
             accelerate {LNP} development for {mRNA} delivery},
  author  = {Xu, Y. and Ma, S. and Cui, H. and others},
  journal = {Nature Communications},
  volume  = {15},
  pages   = {6305},
  year    = {2024},
  doi     = {10.1038/s41467-024-50619-z},
}

@article{lantern2025,
  title   = {{LANTERN}: A Machine Learning Framework for Lipid Nanoparticle
             Transfection Efficiency Prediction},
  author  = {Mehradfar, A. and Sepehri, M. S. and others},
  journal = {arXiv preprint arXiv:2507.03209v1},
  year    = {2025},
  url     = {https://arxiv.org/abs/2507.03209v1},
}
